\documentclass{article}
\usepackage[nonatbib,preprint]{neurips_2021}
\usepackage[numbers]{natbib}
\usepackage[utf8]{inputenc} 
\usepackage[T1]{fontenc}    
\usepackage{hyperref}       
\usepackage{url}            
\usepackage{booktabs}       
\usepackage{amsfonts}       
\usepackage{nicefrac}       
\usepackage{microtype}      
\usepackage[export]{adjustbox}
 
\usepackage{amsmath,mathrsfs,amssymb,mathtools,bm}
\usepackage{scalerel}
\usepackage{amsthm}
\usepackage{txfonts}
\usepackage{fontawesome}
\usepackage{wrapfig}
\usepackage{todonotes}
\usepackage{xcolor}
\usepackage{setspace}
\usepackage{bm,bbm}
\usepackage{paralist}
\usepackage{enumitem}
\usepackage{subcaption}

\usepackage[ruled,linesnumbered,noresetcount,vlined]{algorithm2e}
\makeatletter

\newcommand{\nosemic}{\renewcommand{\@endalgocfline}{\relax}}
\newcommand{\dosemic}{\renewcommand{\@endalgocfline}{\algocf@endline}}
\let\oldnl\nl
\newcommand{\nonl}{\renewcommand{\nl}{\let\nl\oldnl}}

\SetKwFor{Repeat}{repeat}{:}{endw}
\makeatother

\makeatletter
\newcommand{\oset}[3][0ex]{%
  \mathrel{\mathop{#3}\limits^{
    \vbox to#1{\kern-2\ex@
    \hbox{$\scriptstyle#2$}\vss}}}}
\makeatother
\newcommand{\optimal}[1]{\oset{\scalebox{.5}{$\star$}}{#1}\!}

\usepackage{letltxmacro}
\LetLtxMacro\orgvdots\vdots
\LetLtxMacro\orgddots\ddots

\newtheorem{corollary}{Corollary}
\newtheorem{problem*}{Problem}

\newtheorem{theorem}{Theorem}
\newtheorem{lemma}{Lemma}

\newtheorem{definition}{Definition}
\newtheorem{proposition}{Proposition}

\newtheorem*{example*}{Example}

\DeclareMathOperator*{\argmax}{argmax}
\DeclareMathOperator*{\argmin}{argmin}

\newcommand*{\defeq}{\stackrel{\text{def}}{=}}

\newcommand{\btheta}{{\bm{\theta}}}

\newcommand{\cA}{\mathcal{A}}

 \newcommand{\cL}{\mathcal{L}}
\newcommand{\cM}{\mathcal{M}} \newcommand{\cN}{\mathcal{N}}

 \newcommand{\cY}{\mathcal{Y}}
 \newcommand{\cX}{\mathcal{X}}

\newcommand{\EE}{\mathbb{E}} \newcommand{\RR}{\mathbb{R}}

\usepackage{esvect}

\newcommand{\fpx}{p^{\leftrightarrow}_{\bm{x}}}
\newcommand{\fp}[1]{p^{\leftrightarrow}_{#1}}

\newcommand\hugP[1]{\left(#1\right)}

\title{A Fairness Analysis on Private Aggregation of Teacher Ensembles}

\author{Cuong Tran\\
    Syracuse University\\
    \texttt{ctran@syr.edu}\\
  \And
  My H.~Dinh\\
  Syracuse University\\
  \texttt{mydinh@syr.edu}\\
  \And
  Kyle Beiter\\
  Syracuse University\\
  \texttt{kbeiter@syr.edu}\\
  \And
  Ferdinando Fioretto\\
  Syracuse University\\
  \texttt{ffiorett@syr.edu}\\
}

\begin{document}

\maketitle\sloppy\allowdisplaybreaks

\maketitle\sloppy\allowdisplaybreaks
 
\begin{abstract} 
 The Private Aggregation of Teacher Ensembles
 (PATE) \cite{papernot2018scalable} is an important
 private machine learning framework. It combines multiple
 learning models used as teachers for a student model that
 learns to predict an output chosen by noisy voting among the
 teachers. The resulting model satisfies differential privacy and has
 been shown effective in learning high quality private models in
 semisupervised settings or when one wishes to protect the data
 labels.

This paper asks whether this privacy-preserving framework introduces
or exacerbates bias and unfairness and shows that PATE can introduce
accuracy disparity among individuals and groups of individuals. 
The paper analyzes
which algorithmic and data properties are responsible for the
disproportionate impacts, why these aspects are affecting different
groups disproportionately, and proposes guidelines to mitigate these
effects. The proposed approach is evaluated on several datasets and
settings.
\end{abstract}

\section{Introduction}
\label{sec:introduction}

The availability of large datasets and inexpensive computational
resources has rendered the use of machine learning (ML) systems
instrumental for many critical decisions involving individuals, 
including criminal assessment, landing, and hiring, all of which have 
a profound social impact.  
A key concern for the adoption of these system regards how they handle bias and discrimination 
and how much information they leak about the individuals whose data 
is used as input.

Differential Privacy (DP) \cite{dwork:06} is an algorithmic property 
that bounds the risks of disclosing sensitive information of individuals 
participating in a computation.
It has become the paradigm of choice in privacy-preserving machine 
learning systems and its deployments are growing at a fast rate. 
However, it was recently observed that DP systems may induce biased 
and unfair outcomes for different groups of individuals 
\cite{NEURIPS2019_eugene,pujol:20,xu2020removing}. 

The resulting outcomes can have significant societal and economic
impacts on the involved individuals: classification errors may
penalize some groups over others in important determinations
including criminal assessment, landing, and hiring 
\cite{NEURIPS2019_eugene} or can result in disparities regarding
the allocation of critical funds and benefits \cite{pujol:20}.
{\em While these surprising observations are becoming increasingly 
common, their causes are largely understudied and not fully understood.}

This paper makes a step toward this important quest, and studies the 
disparate impacts arising when training a model using \emph{Private Aggregation 
of Teacher Ensembles} (PATE) \cite{papernot2018scalable} an important
and popular privacy-preserving machine learning framework. It combines multiple
agnostic learning models used as teachers for a student model that
learns to predict an output chosen by noisy voting among the
teachers. The resulting model satisfies differential privacy and has
been shown effective in learning high quality private models in
semisupervised settings or when one wishes to protect the data labels.

The paper analyzes which properties of the algorithm and the data are 
responsible for the disproportionate impacts, why these aspects are 
affecting different individuals or groups of individuals disproportionately, and proposes a solution that may aid mitigating these effects. 

In summary, the paper makes the following contributions:
\begin{enumerate}[leftmargin=*, parsep=0pt, itemsep=2pt, topsep=0pt]
\item It uses a fairness notion that relies on the concept of excessive risk, and measures the direct impact of privacy to the model accuracy for individuals or groups.
\item It analyzes this fairness notion in PATE, a state-of-the-art privacy-preserving ML framework. 
\item It isolates key components of the model parameters and the data properties which are responsible for the observed disparate impacts. 
 \item It studies when and why these components affect different  individuals or groups disproportionately. 
 \item Finally, based on these findings, it proposes a method that may aid mitigating these unfairness effects while retaining high accuracy.
\end{enumerate}

To the best of the authors knowledge, this work represents a first effort toward understanding the reasons of the disparate impacts in 
privacy-preserving ensemble models.

\section{Related Work} 
\label{sec:related_work}
The study of the disparate impacts caused by privacy-preserving algorithms has recently seen several important developments.
\citet{ekstrand:18} raise questions about the tradeoffs involved between privacy and fairness. \citet{cummings:19} study the tradeoffs arising between differential privacy and equal opportunity, a fairness notion requiring a classifier to produce equal true positive rates across different groups. They show that there exists no classifier that simultaneously achieves $(\epsilon,0)$-DP, satisfies equal opportunity, and has accuracy better than a constant classifier.  
This development has risen the question of whether one can practically build fair models while retaining sensitive information private. 
To this end, \citet{jagielski:18} presents two algorithms that
satisfy $(\epsilon,\delta)$-differential privacy and equalized odds.
\citet{mozannar2020fair} develops methods to adapt a nondiscriminatory learner to work with privatized protected attributes and \citet{tran2020differentially} proposes a differentially private learning approach to enforce several group fairness notions using a Lagrangian dual method. 

\citet{pujol:20} were seemingly the first to show, empirically, that 
resource allocation decisions made using DP datasets may disproportionately affect some groups of individuals over others. These studies were complemented theoretically by \citet{Cuong:IJCAI21}.
Similar observations were also made in the context of model learning. \citet{NEURIPS2019_eugene} empirically observed that the accuracy of a DP model trained using DP-Stochastic Gradient Descent (DP-SGD) decreases disproportionately across groups causing larger negative impacts to the underrepresented groups. \citet{farrand2020neither,uniyal2021dpsgd} reaches similar conclusions and show that this disparate impact is not limited to highly imbalanced data. 

This paper builds on this body of work and their important empirical observations. It provides an analysis for the reasons of unfairness in the context of semi-supervised private learning ensembles, a commonly adopted scheme in privacy-preserving ML systems as well as introduces mitigating guidelines.

\section{Preliminaries: Differential Privacy}
\label{sec:preliminaries}

Differential privacy (DP) \cite{dwork:06} is a strong privacy notion used to quantify and bound the privacy loss of an individual's participation in a computation. 
 Informally, it  states that the probability of any output does not change much when a record is added or removed from a dataset, limiting the amount of information that the output reveals about any individual.  
The action of adding or removing a record from a dataset $D$, resulting in a new dataset $D'$, defines the notion of \emph{adjacency}, denoted $D \sim D'$.
\begin{definition}
  \label{dp-def}
  A mechanism $\cM \!:\! \mathcal{D} \!\to\! \mathcal{R}$ with domain $\mathcal{D}$ and range $\mathcal{R}$ is $(\epsilon, \delta)$-differentially private, if, for any two adjacent inputs $D \sim D' \!\in\! \mathcal{D}$, and any subset of output responses $R \subseteq \mathcal{R}$:
  \[
      \Pr[\cM(D) \in R ] \leq  e^{\epsilon} 
      \Pr[\cM(D') \in R ] + \delta.
  \]
\end{definition}
\noindent 
Parameter $\epsilon > 0$ describes the \emph{privacy loss} of the algorithm, with values close to $0$ denoting strong privacy, while parameter 
$\delta \in [0,1)$ captures the probability of failure of the algorithm to satisfy $\epsilon$-DP. 
The global sensitivity $\Delta_\ell$ of a real-valued 
function $\ell: \mathcal{D} \to \mathbb{R}$ is defined as the maximum amount 
by which $\ell$ changes  in two adjacent inputs:
\(
  \Delta_\ell = \max_{D \sim D'} \| \ell(D) - \ell(D') \|.
\)
In particular, the Gaussian mechanism, defined by
\(
    \mathcal{M}(D) = \ell(D) + \mathcal{N}(0, \Delta_\ell^2 \, \sigma^2), 
\)
\noindent where $\mathcal{N}(0, \Delta_\ell^2\, \sigma^2)$ is 
the Gaussian distribution with $0$ mean and standard deviation 
$\Delta_\ell^2\, \sigma^2$, satisfies $(\epsilon, \delta)$-DP for 
$\delta \!>\! \frac{4}{5} \exp(-(\sigma\epsilon)^2 / 2)$ 
and $\epsilon \!<\! 1$ \cite{dwork:14}.

\begin{figure*}[tb]
    \centering
    \includegraphics[width=0.75\linewidth]{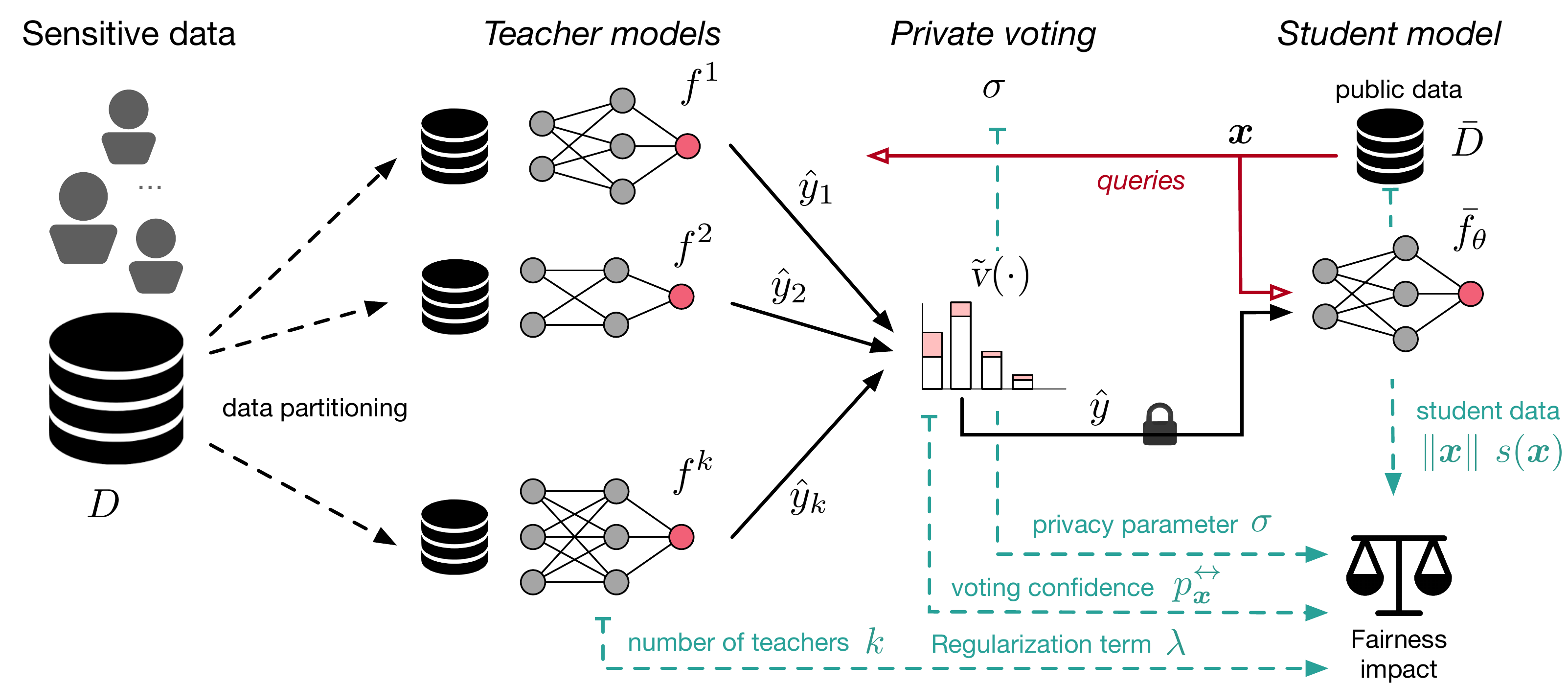}
    \caption{Illustration of PATE and aspects contributing to  fairness impact.}
    \label{fig:scheme}
\end{figure*}

\section{Problem Settings and Goals}
\label{sec:problem}

This paper considers a \emph{private} dataset $D$ consisting of $n$ individuals' data points $(\bm{x}_i, y_i)$, with $i \!\in\! [n]$, drawn i.i.d.~from an unknown distribution $\Pi$. Therein, $\bm{x}_i \!\in\! \cX$ is a feature vector that \emph{may} contain a protected group attribute $\bm{a}_i \!\in\! \cA \!\subset\! \cX$, and $y_i \!\in\! \cY = [C]$ is a $C$-class label. 
{For example, consider a classifier that needs to predict criminal defendant’s recidivism. The training example features $\bm{x}_i$ may describe the individual's demographics, education, occupation, and crime committed, the protected attribute $\bm{a}_i$, if available, may describe the individual's gender or ethnicity, and $y_i$ represents whether or not the individual has high risk to reoffend.} 

This paper studies the fairness implications arising when training privacy-preserving semi-supervised transfer learning models. 
The setting is depicted in Figure \ref{fig:scheme}. We are given an ensemble of \emph{teacher} models $\bm{T} \!=\! \{f^i\}_{i=1}^k$, with each $f^i \!:\! \cX \!\to\! \cY$ trained on a non-overlapping portion $D_i$ of $D$. This ensemble is used to transfer knowledge to a \emph{student} model $\bar{f}_\theta \!:\! \cX \!\to\! \cY$, where $\btheta$ denotes a vector of real-valued parameters associated with model $\bar{f}$. 

The student model $\bar{f}$ is trained using a \emph{public} dataset $\bar{D} \!=\! \{\bm{x}_i\}_{i=1}^m$ with samples drawn i.i.d.~from the same distribution $\Pi$ considered above but whose labels are unrevealed. 
The paper focuses on learning {classifier} $\bar{f}_\theta$ using  knowledge transfer from the teacher model ensemble $\bm{T}$ while guaranteeing the privacy of each individual's data $(\bm{x}_i, y_i) \!\in\! D$. 
The sought model is learned by minimizing the regularized empirical risk function
\begin{align}
\label{eq:ERM}
    \optimal{\btheta} &= \argmin_{\bm{\theta}} \cL\left(\btheta; \bar{D}, \bm{T}\right)
    = \sum_{\bm{x} \in \bar{D}} 
    \ell \left( \bar{f}_\theta(\bm{x}), \textsl{v}\left(\bm{T}(\bm{x}) \right)\right)  
    + \lambda \left\| \theta \right\|^2,
\end{align}
where $\ell \!:\! \cY \times \cY \!\to\! \mathbb{R}_+$ is a loss function and measures the performance of the model,
$\textsl{v} \!:\! \cY^k \!\to\! \cY$ is a \emph{voting scheme} used to
decide the prediction label from the ensemble $\bm{T}$, with 
$\bm{T}(\bm{x})$ used as a shorthand for $\{f^i(\bm{x})\}_{i=1}^k$, and $\lambda>0$ is a regularization parameter. 

The paper focuses on learning classifiers that protect the disclosure of the individual's data using the notion of differnetial privacy and it analyzes the fairness impact (as defined below) of privacy on different groups and individuals. 

\subsection*{Privacy}
\emph{Privacy} is achieved by using a differentially private version $\tilde{\textsl{v}}$ of the voting function $\textsl{v}$, defined 
as 
\begin{equation}
\label{eq:noisy_max}
    \tilde{\textsl{v}} \left( \bm{T}(\bm{x}) \right) \!=\! \argmax_j 
    \left\{ \#_j \left(\bm{T}(\bm{x}) \right) \!+\! 
               \cN \left(0, \sigma^2 \right) \right\},
\end{equation}
which perturbs the reported counts $\#_j(\bm{T}(\bm{x}))\!=\!|\{i\!:\!i \!\in\![k], f^i(\bm{x}) \!=\! j\}|$ associated to label $j \!\in\! \cY$, via additive Gaussian noise of zero mean and standard deviation $\sigma$.  
The overall approach, called \emph{PATE}, guarantees $(\epsilon, \delta)$-differential privacy, with privacy loss scaling with the magnitude of the standard deviation $\sigma$ and the size of the public dataset $\bar{D}$ \cite{papernot2018scalable}. 
A detailed discussion reviewing the privacy analysis of PATE is reported in {Appendix \ref{app:privacy_analysis}}.
Throughout the paper, the privacy-preserving parameters of the model $\bar{f}$ are denoted with $\tilde{\bm{\theta}}$.

\subsection*{Fairness}
The fairness analysis focuses on the notion of \emph{excessive risk} \cite{NIPS2017_f337d999,ijcai2017548}. It defines the difference between the private and non private risk functions: 
\begin{flalign}
\label{def:excessiver_risk}
  R(S, \bm{T}) \defeq \EE_{\tilde{\btheta}}  
  \left[ \cL(\tilde{\btheta}; S, \bm{T}) \right] 
    - \cL(\optimal{\btheta}; S, \bm{T}),
\end{flalign}
where the expectation is defined over the randomness of the private mechanism, $S$ is a subset of $\bar{D}$, and 
$\tilde{\btheta}$ denotes the private student's model parameters while $\optimal{\btheta}\, =\! \argmin_\btheta \cL(\btheta; \bar{D}, \bm{T})$. 
The above definition captures both individual $R(\{\bm{x}\}, \bm{T})$ excessive risk for a sample $\bm{x}$ and group $R(\bar{D}_{\leftarrow a}, \bm{T})$ excessive risk for a group $a$, where $\bar{D}_{\leftarrow a}$ denotes the subset of $\bar{D}$ containing exclusively samples whose group attribute is $a \in \cA$. 
This paper uses shorthands $R(\bm{x})$ and $R(\bar{D}_{\leftarrow a})$ to denote $R(\bm{x}, \bm{T})$ and $R(\bar{D}_{\leftarrow a}, \bm{T})$.

Finally, this paper assumes that the private mechanisms are non-trivial, i.e., they minimize the population-level excessive risk 
$R(\bar{D})$ and the fairness goal is to minimize excessive risk difference among all individuals and/or groups.

\section{PATE Fairness Analysis: Roadmap}
\label{sec:roadmap}

The next sections focus on two orthogonal aspects of PATE: the \emph{algorithm's parameters} and the \emph{public student data distribution characteristics} and analyze their fairness impact.

Within the algorithm's parameters, in addition to the privacy variable $\sigma$, the paper reveals two surprising aspects which have a direct impact on fairness: The size $k$ of the teacher ensemble and the regularization parameter $\lambda$ associated with the student risk function. 
Regarding the public student data's characteristics, the paper shows that the magnitude of the sample input norms $\|\bm{x}\|$ and the distance of a sample to the decision boundary (denoted $s(\bm{x})$) play decisive roles to exacerbate the excessive risk induced by the student model. 
These aspects are illustrated schematically with green dotted lines in Figure \ref{fig:scheme}. 

Several aspects of the analysis in this paper rely on the following definition. 
\begin{definition}[Flipping probability]
Given a data sample $(\bm{x}, y) \!\in\! D$, for an ensemble model $\bm{T}$ and voting scheme $\textsl{v}$, the \emph{flipping probability} of $\bm{T}$ is defined as: 
\begin{equation}
\label{eq:flipping_pr}
    \fp{\bm{x}} \defeq  \Pr\left[ 
    \tilde{\textsl{v}}(\bm{T}(\bm{x})) \neq \textsl{v}(\bm{T}(\bm{x})) 
    \right].
\end{equation}
\end{definition}
\noindent It connects the \emph{voting confidence} of the teacher ensemble with the perturbation induced by the privacy-preserving voting scheme, and will be instrumental in the fairness analysis introduced below.  

The following sections use several standard datasets including UCI Adults, Credit card, Bank, and Parkinsons \cite{UCIdatasets,article, Moro2014ADA} to support the theoretical claims. The results use feed-forward networks with two hidden layers and nonlinear ReLU activations for both the ensemble and student models. All reported metrics are average of 100 repetitions, used to compute the empirical expectations.
When not otherwise stated, the experiments refer to the \textsl{Credit card} dataset. 

The main paper reports a glimpse of the empirical results, which appears in an extended form in the Appendix (\ref{app:exp_ext}). 
Additional description of the dataset and proofs of all theorems are reported in the Appendix.

\section{Algorithm's Parameters}
\label{sec:alg_params}

This section focuses on analyzing the algorithm's parameters that affect the disparate impact of the student model outputs. In more details, it shows that, in addition to the privacy parameter $\sigma$, the 
regularization term $\lambda$ of the empirical risk function 
$\cL(\btheta, \bar{D}, \bm{T})$ (see Equation \eqref{eq:ERM}) and the 
size $k$ of the teacher ensemble $\bm{T}$ largely control the difference between model learned with noisy and clean labels. 
The fairness analysis reported in this section assumes that the student model loss $\ell(\cdot)$ is convex and \emph{decomposable}:

\begin{definition}[Decomposable function]
\label{def:1}
A function $\ell(\cdot) $ is \emph{decomposable} if there exists a parametric function $h_{\btheta} \!:\! \cX \!\to\! \RR$, a constant real number $c$, and a function
$z \!:\! \RR \!\to\! \RR$, such that, for $\bm{x} \!\in\! \cX$, 
and $y \!\in\! \cY$:
\begin{equation}
\label{eq:decomposable}
    \ell(f_{\btheta}(\bm{x}), y) = z(h_{\btheta}(\bm{x})) 
    + c \, y\, h_{\theta}(\bm{x}).
\end{equation}
\end{definition}

Note that a number of loss functions commonly adopted in machine 
learning, including the logistic loss and the least square loss 
function, are decomposable \cite{gao2016risk, patrini2014almost}. 
Additionally, while it is common to impose restrictions on the nature
of the loss function to render the analysis tractable, our findings 
are empirically validated on non-linear models, as shown next.

The following theorem sheds light on the unfairness induced by PATE and the dependency with its parameters.
It provides an upper bound on the expected difference between the 
non-private and private student model parameters. As the paper will 
show in Theorem \ref{thm:3}, this quantity is closely related with the excessive risk. 
Therein, $\optimal{\btheta}$ and $\tilde{\btheta}$ represent the 
parameters of student model $\bar{f}$ which are learned as a result of training, respectively, with a clean or noisy voting scheme. 

\begin{theorem}
\label{thm:1}
Consider a student model $\bar{f}_{\btheta}$ trained with a convex and decomposable loss 
function $\ell(\cdot)$. Then, the expected 
difference between the private and non-private model parameters is 
upper bounded as follows:
\begin{equation}
    \mathbb{E}\left[ \| \optimal{\btheta} - \tilde{\btheta} \| \right] 
    \leq \frac{|c|}{m\lambda} \left[ \sum_{\bm{x} \in \bar{D}} p^{\leftrightarrow}_{\bm{x}} \| g_{\bm{x}}\| \right],
\end{equation}
where $c$ is a real constant and
$g_{\bm{x}}= \max_{\btheta}\| \nabla_{\btheta} h_{\btheta}(\bm{x}) \|$ 
represents the maximum gradient norm distortion introduced by a 
sample $\bm{x}$. Both $c$ and $h$ are defined as in Equation 
\eqref{eq:decomposable}. 
\end{theorem}

\noindent 
The proof relies on $\lambda$-strong convexity of the loss function ${\cL(\cdot)}$ (see Appendix \ref{app:missing_proofs}).
Theorem \ref{thm:1} relates the difference in the expected private  and non-private student  parameters  with  three  key factors: {\bf (1)} the regularization term $\lambda$, {\bf (2)} the flipping probability $\fpx$, and {\bf (3)} the the maximum gradient norm distortion $g_{\bm{x}}$ induced by a sample $\bm{x}$. 
The former two factors are mechanisms-dependent components and the subject of study of this section. As it will be shown next, they are controlled by the size $k$ of the teacher ensemble and the noise parameter $\sigma$. The discussion about data dependent components, including those related with the gradient norms, is delegated to  Section \ref{sec:data_prop}.

Throughout the paper, the quantity $\|\optimal{\btheta} - \tilde{\btheta}\|$ is referred to as \emph{model sensitivity to privacy}, or simply \emph{model sensitivity}, as it captures the effect of the private teacher voting on the student learned model.  

\begin{figure}[!bt]
    \centering
    \includegraphics[width=0.75\linewidth]{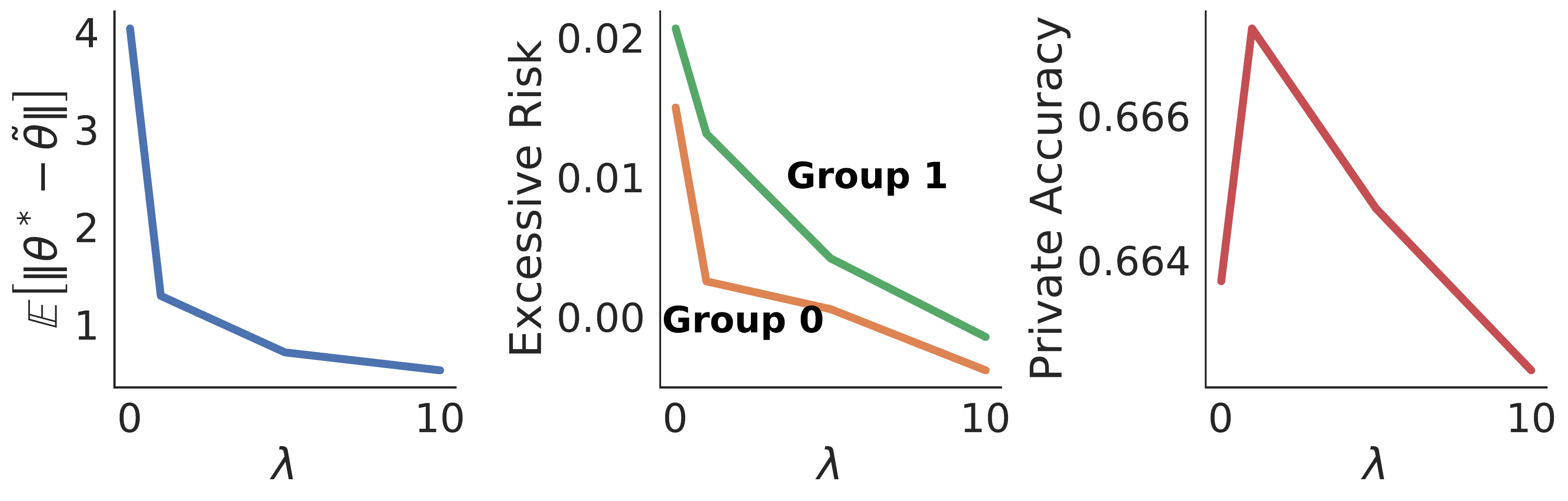}
    \caption{
    Credit-card dataset with $\sigma \!=\! 50, k\!=\!150$. 
    Model sensitivity (left), empirical risk (middle), and model accuracy (right) as a function of the regularization term.}
    \label{fig:lambda_effect}
\end{figure}

\subsection{The impact of the regularization term $\lambda$} 
The first immediate observation of Theorem \ref{thm:1} is that variations of
the regularization term $\lambda$ can reduce or magnify the
difference between the private and non-private student model
parameters. 
Since the model sensitivity $\EE\| \optimal{\btheta} - \tilde{\btheta}\|$ relates directly to the excessive risk (see Theorem \ref{thm:3}), the regularization term affects the  disparate impact of the privacy-preserving student model. 

These effects are further illustrated in Figure \ref{fig:lambda_effect}. The figure shows how increasing $\lambda$ reduces the empirical expected difference between the privacy-preserving and original model parameters $\EE\|\optimal{\btheta} - \tilde{\btheta}\|$ (left),  as well as the excessive risk $R(\bar{D}_{\leftarrow a})$ 
difference between groups $a=0$ and $a=1$ (middle). 
Note, however, that while larger $\lambda$ values may reduce the model unfairness, they can hurt the resulting model accuracy, as shown in the right plot. 
The latter is an intuitive and recognized effect of large regularizers factors. 

\subsection{The impact of the teachers ensemble size $k$} 
The second aspect considered in this section is the relation between the ensemble size $k$ and the resulting private model fairness. The
following result relates the size of the ensemble with its voting confidence.
\begin{theorem}
\label{thm:2}
For a sample $\bm{x} \!\in\! \bar{D}$ assume that the teacher models 
outputs $\hat{y}_i = f^i(\bm{x})$ $(i \in [k])$ are all in agreement. That is, $\hat{y}_i = \hat{y}_j$ for all $i,j \in [k]$. 
Then, the flipping probability $\fpx$ is given by:
\begin{equation}
     \fpx = 1 - \Phi\left(\frac{k}{\sqrt{2} \sigma}\right),
\end{equation}
where $\Phi(\cdot)$ is the CDF of the standard normal distribution 
and $\sigma$ is the standard deviation in the Gaussian mechanism.
\end{theorem}

\noindent 
The proof is based on the properties of independent Gaussian random variables.

The analysis above sheds light on the outcome of the teachers voting 
scheme and its relation with the ensemble size $k$ (as well as the 
privacy parameter $\sigma$). It indicates that larger $k$ values correspond to smaller flipping probability $\fpx$. Combined with Theorem~\ref{thm:1}, the result suggests that the difference between the private and non-private model parameters is inversely proportional to $k$. 

While for simplicity of analysis Theorem \ref{thm:2} requires the decision 
of all teachers to agree on a given sample $\bm{x}$, our empirical 
analysis supports this result for the more general scenario where 
different teachers have different agreements on a sample. 
Figure \ref{fig:flipping_pr_k} (left) illustrates the relation between 
the number $k$ of teachers and the flipping probability $\fpx$ of the ensemble. The plot shows a clear trend indicating that larger ensembles result in smaller flipping probabilities. 
\begin{figure}[!tb]
    \centering
    \includegraphics[width=0.35\linewidth, height=100pt,valign=M]{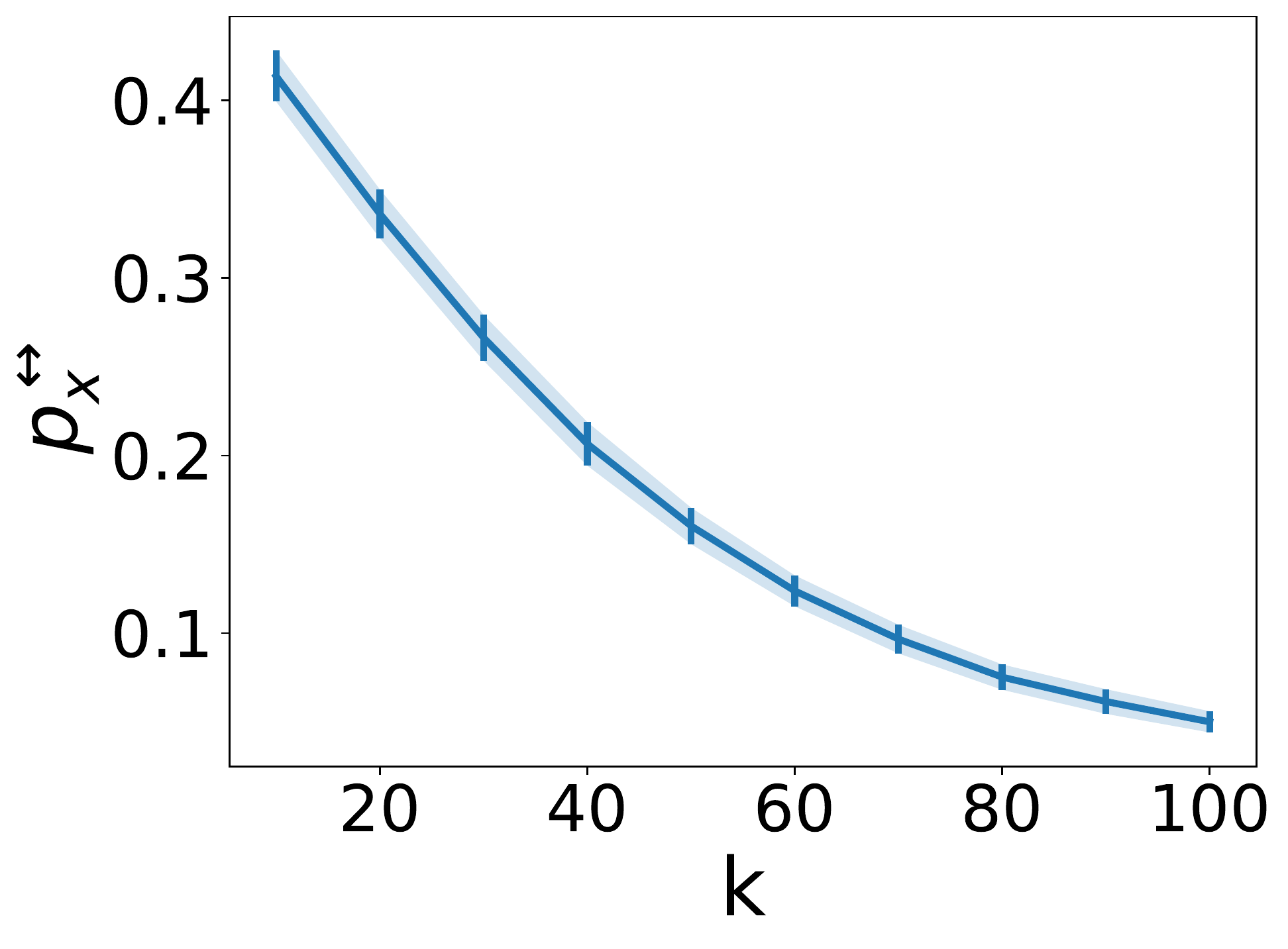}
    \includegraphics[width=0.35\linewidth, height=104pt,valign=M]{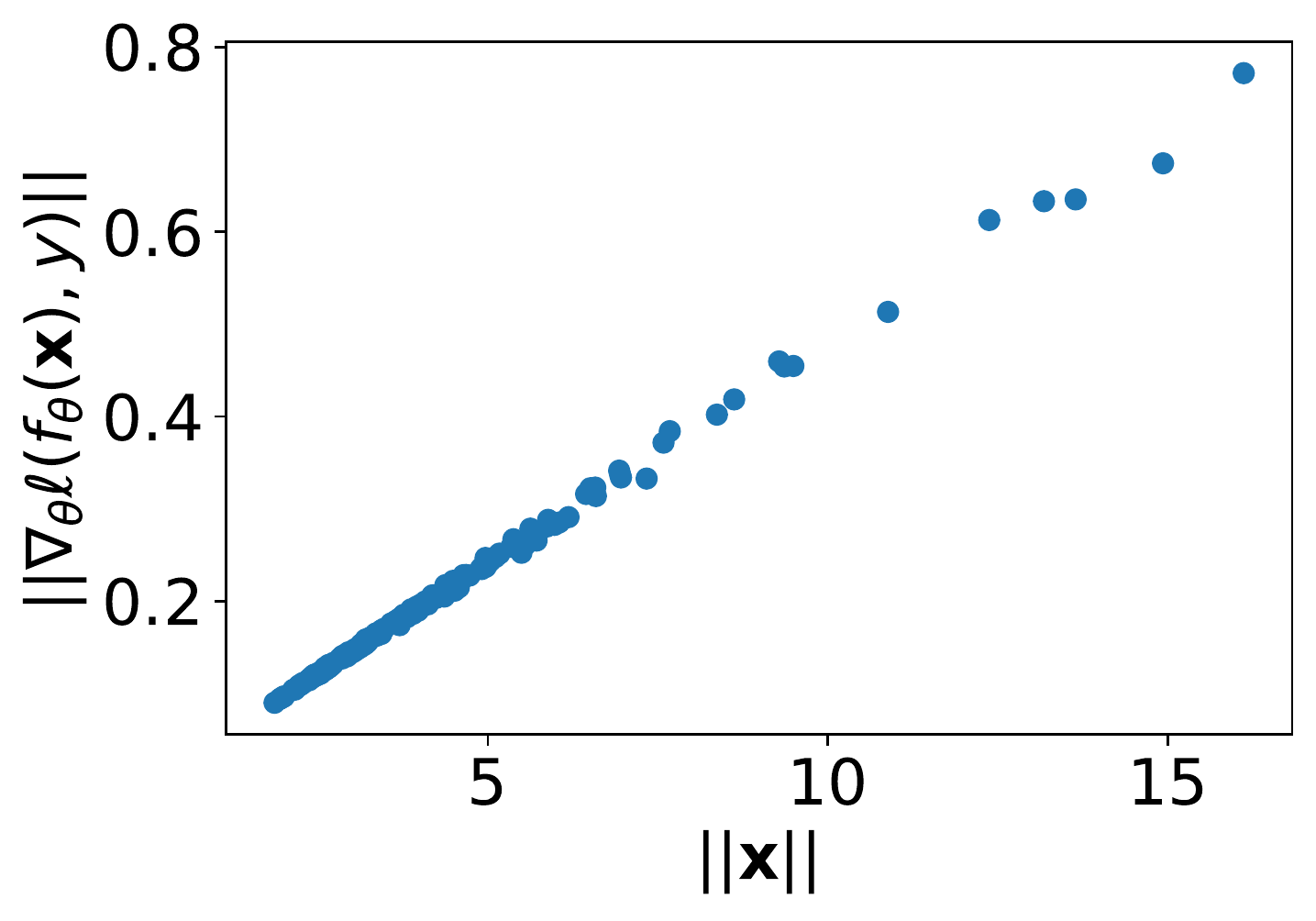}
    \caption{Credit card dataset: Average flipping probability $\fpx$ 
    for samples $\bm{x} \in \bar{D}$ as a function of the ensemble size $k$ (left)
    and 
    relation between gradient and input norms (right).}
    \label{fig:flipping_pr_k}
    \label{fig:grad_inp_corr}
\end{figure}

\begin{figure}[!tb]
    \centering
    \includegraphics[width=0.75\linewidth]{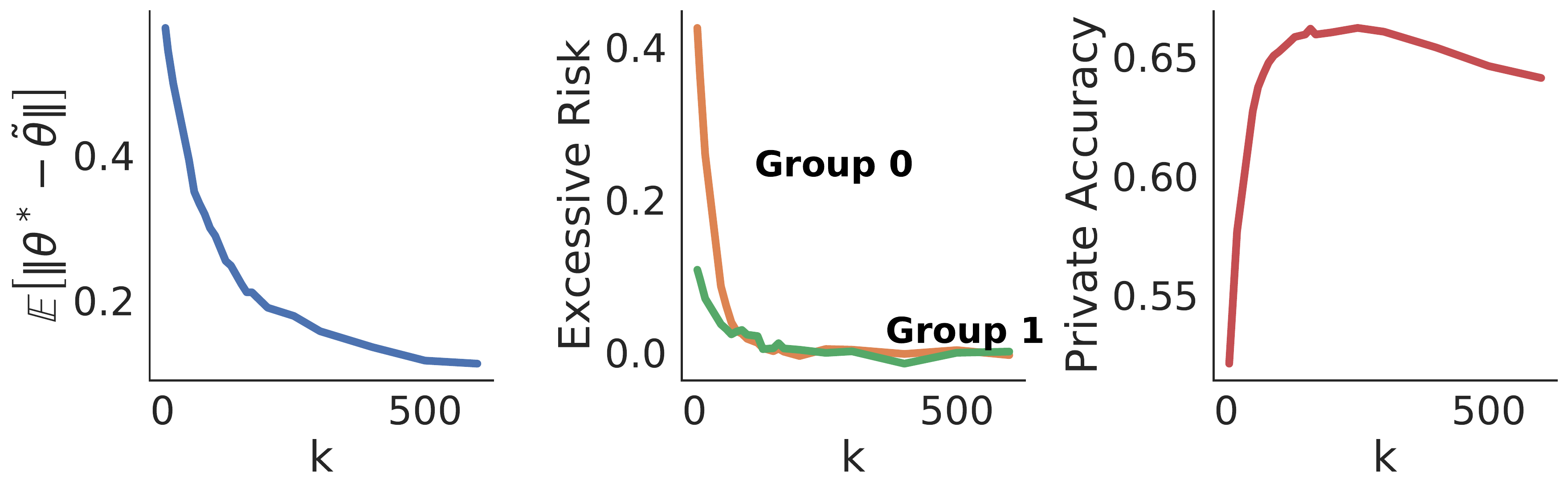}
    \caption{Income dataset with $\sigma\!=\!50, \lambda\!=\!100$. Expected model sensitivity (left), 
    empirical risk (middle), and model accuracy (right) as a function 
    of the ensemble size.}
    \label{fig:effect_k}
\end{figure}

Next, analogously to what is reported in Figure \ref{fig:lambda_effect}, Figure \ref{fig:effect_k} shows that increasing $k$ reduces the difference in the expected private and non-private model parameters (left), reduces the 
group excessive risk difference (middle), and increases the model $\bar{f}$
accuracy (right). 
However, similarly as for the regularization term 
$\lambda$, there is also a downside of using very large ensembles: 
large values $k$ can reduce the accuracy of the (private and non-private)
models. 
While studying these tradeoffs goes beyond the scope of this work, 
we believe this behavior is related with the bias-variance tradeoff imposed on the growing ensemble: The larger the ensemble the less data each teacher is given to train their models, thus affecting their voting accuracy.
We believe this is an interesting and important direction for future work.

\noindent 
This section concludes with a useful corollary of Theorem~\ref{thm:1}. 
\begin{corollary}[Theorem~\ref{thm:1}]
\label{cor:1}
Let $\bar{f}_{\theta}$ be a \emph{logistic regression} classifier. Its 
expected model sensitivity is upper bounded as:
\begin{equation}
\label{eq:8}
 \mathbb{E}\left[ \| \optimal{\btheta} - \tilde{\btheta} \| \right] 
    \leq 
    \frac{1}{m\lambda} \left[ \sum_{\bm{x} \in \bar{D}} \fp{\bm{x}} \| \bm{x} \| \right]. 
 \end{equation}
\end{corollary}

The result above highlights several interesting points. 
First, in logistic regression, samples with large input norms can 
have a non negligible impact on fairness. This place emphasis on an 
nontrivial aspect of the student data properties which may affect 
fairness and is subject of study of the next section. 
Next, notice the similarities between Equation \eqref{eq:8} and Equation \eqref{eq:6}; In the former, gradient norms $\|\bm{x}\|$ multiply the associated flipping probabilities $\fpx$ in place of the gradient norms $\|g_{\bm{x}}\|$. 
Thus the result above indicates the presence of a relation between gradient norms and input norms, which is further highlighted in 
Figure~\ref{fig:grad_inp_corr} (right). The plot illustrates the strong correlation between input norms and their associated gradient norms.

\section{Student's Data Properties}
\label{sec:data_prop}

Having examined the algorithmic properties of PATE affecting fairness, 
this section turns on analyzing a set of properties concerning 
the student data which regulate the disproportionate impacts of the algorithm. 
The subsequent set of results shows that the norms of the student's data samples and their distance to the decision boundary are two key factor tied to the exacerbation of excessive risk in PATE. 

The following is a corollary of Theorem \ref{thm:1} and bounds the 
second order statistics of the model sensitivity to privacy.
\begin{corollary}[Theorem~\ref{thm:1}]
\label{cor:2}
Given the same settings and assumption of Theorem \ref{thm:1}, it follows:
\begin{equation}
\label{eq:9}
    \mathbb{E}\left[ \| \optimal{\btheta} - \tilde{\btheta} \|^2 \right] 
    \leq \frac{|c|^2}{m \lambda^2} \left[ \sum_{\bm{x} \in \bar{D}} p^{\leftrightarrow 2}_{\bm{x}} \| g_{\bm{x}}\|^2 \right].
\end{equation}
\end{corollary}
\noindent
Note that, similarly to as shown by Corollary \ref{cor:1}, when $\bar{f}_\theta$ is a logistic regression model, the gradient norm $ \| g_{\bm{x}}\|$ in Equation \eqref{eq:9}
can be substituted with the input norm $\|\bm{x}\|$. 

The result above is useful to derive an upper bound on the excessive risk, 
as illustrated in the following theorem. 

\begin{theorem}
\label{thm:3}
Let $\ell(\cdot)$ be a $\beta_{\bm{x}}$-smooth loss function. 
The excessive risk $R(\bm{x})$ of a sample $\bm{x}$ is upper 
bounded as:
\begin{equation}
\label{eq:ER_ub}
R(\bm{x}) \leq \| \nabla_{\optimal{\btheta}}  
        \ell(\bar{f}_{\optimal{\btheta}}(\bm{x}),y)\|U_1 
        + \frac{1}{2} \beta_x U_2,
\end{equation} 
where, $U_1 = \mathbb{E}\left[ \| \optimal{\btheta} - \tilde{\btheta} \| \right]$ 
and  $U_2 = \mathbb{E}\left[ \| \optimal{\btheta} - \tilde{\btheta} \|^2 \right]$
capture the first and second order statistics of the model sensitivity. 
\end{theorem}

\noindent The proof of the above theorem relies on Theorem \ref{thm:1} and 
Corollary \ref{cor:2}, which provide bounds for the first and second 
order statistics of the model sensitivity, and on the 
properties of smooth functions.

Theorem \ref{thm:3} provides an upper bound on the
(individual) excessive risk. It shows the presence of three central 
factors controlling this excessive risk: {\bf (1)} \emph{the gradient norm} 
$\| \nabla_{\optimal{\btheta}}  \ell(\bar{f}_{\optimal{\btheta}}(\bm{x}),y)\|$ for a sample $\bm{x}$, 
{\bf (2)} \emph{the smoothness parameter} $\beta_{\bm{x}}$ associated with a sample $\bm{x}$, 
and {\bf (3)} the \emph{model sensitivity} (captured by 
terms $U_1$ and $U_2$).
As the paper shows next, these seemingly unrelated factors are 
controlled \emph{indirectly} by two key data aspects: the samples 
\emph{input norms} and their \emph{distance to the decision boundary}.  

The rest of the section focuses on logistic regression models, however, as our experimental results illustrate, the observations extend to complex nonlinear models as well.

\subsection{The impact of the data input norms}
First notice that the norm $\| \bm{x} \|$ of a sample $\bm{x}$ strongly influences the quantities $U_1$ and $U_2$, as already observed by 
Corollary~\ref{cor:1}.  This aspect is further illustrated in 
Figure \ref{fig:impact_norm_2_exp_diff} (left), which shows a strong 
correlation between the input norms and the expected model sensitivity. 
\begin{figure}[!tb]
    \centering
    \includegraphics[width=0.35\linewidth,height=100pt]{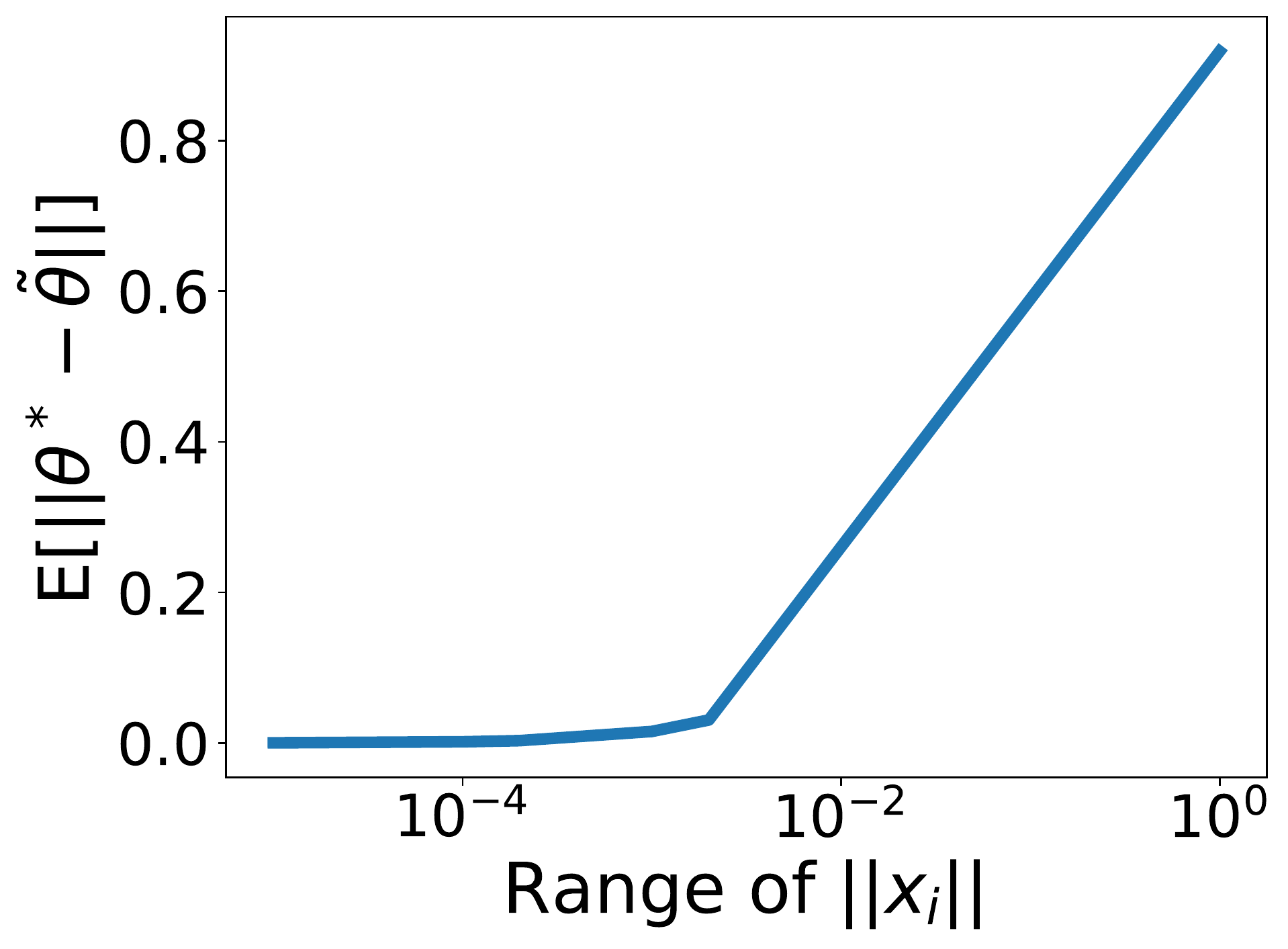}
    \includegraphics[width=0.35\linewidth, height=100pt]{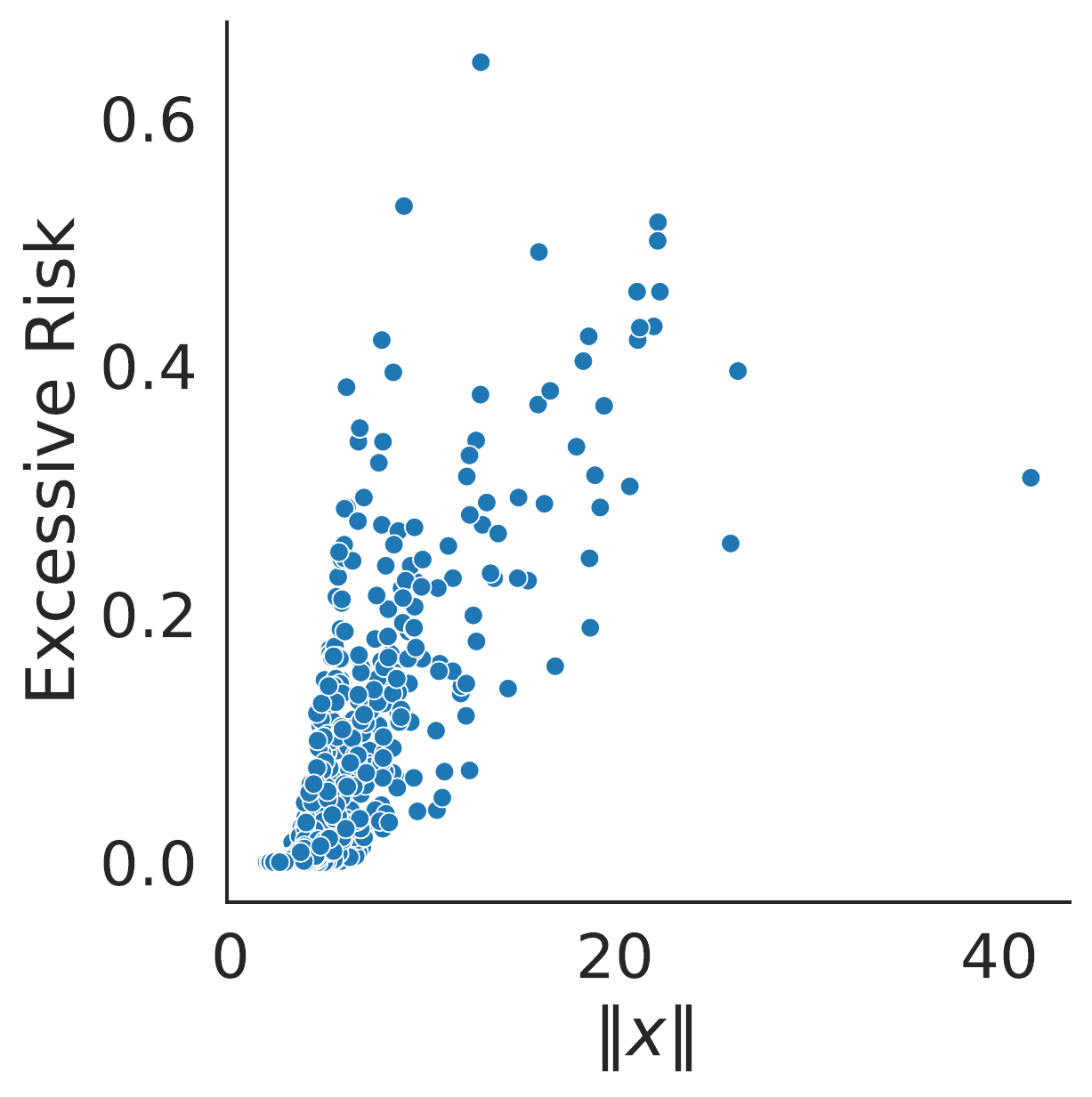}
    \caption{Credit-card data: Relation between input norms and model sensitivity (left)
    and Spearman correlation between input norms and excessive risk (right). }
    \label{fig:impact_norm_2_exp_diff}
\end{figure}
Thus, samples with higher input norms may have a nontrivial impact 
to the model sensitivity and, in turn, to the private model disparate impacts.

\smallskip
Next, the following proposition sheds light on the relation between 
the norm of a sample $\bm{x}$ and its associated gradient norm 
$\| \nabla_{\optimal{\btheta}}\ell(\bar{f}_{\optimal{\btheta}}(\bm{x}),y)\|$.
\begin{proposition}\label{ex:grad_logreg}
Let $\bar{f}_{\btheta}$ be a logistic regression binary classifier 
with cross entropy loss function $\ell( \bar{f}_{\btheta}(\bm{x}, y) ) =- y \log( \bar{f}_{\btheta}(\bm{x}))$. 
For a given sample $(\bm{x}, y) \in \bar{D}$, the gradient
 $\nabla_{\optimal{\btheta}}\ell(\bar{f}_{\optimal{\btheta}}(\bm{x}),y)$ is given by:
\begin{equation}
\label{eq:11}
\nabla_{\optimal{\btheta}}\ell(\bar{f}_{\optimal{\btheta}}(\bm{x}),y) 
= (\bar{f}_{\optimal{\btheta}}(\bm{x}) -y \big) \bm{x}.
\end{equation}
\end{proposition}
\noindent 
Recall that gradient norms have a proportional effect  
on the upper bound of the excessive risk (Equation \eqref{eq:ER_ub}).
Notice further how applying the norm on both side of Equation \eqref{eq:11} illustrates the relation between the gradients and inputs norms.
Thus, the relation above sheds further light on the weight that samples with large norms may have in controlling their associated excessive risk. 
This aspect can be appreciated in Figure \ref{fig:impact_norm_2_exp_diff} (right), which shows a strong correlation between these two quantities. 

The result above can be generalized to multi-class classifiers, as shown in Appendix \ref{app:sec:input_norms}.

Finally, the discussion notes that the smoothness parameter $\beta_{\bm{x}}$ 
captures the local flatness of the loss function at a point $\bm{x}$. 
A derivation of $\beta_{\bm{x}}$ for logistic regression classifier 
is provided below. 
\begin{proposition}\label{ex:hessian_logreg}
Consider again a binary logistic regression as in Proposition 
\ref{ex:grad_logreg}. The smoothness parameter $\beta_{\bm{x}}$ for a 
sample $\bm{x}$ is given by \cite{shi2021aisarah}:
$
    \beta_{\bm{x}} = 0.25 \| \bm{x} \|^2.
$
\end{proposition}
\noindent
The above clearly illustrates the relationship between input norms $\|\bm{x} \|$ and 
the smoothness parameters $\beta_{\bm{x}}$. 

To summarize, propositions \ref{ex:grad_logreg} and \ref{ex:hessian_logreg} 
illustrate that individuals $\bm{x}$ with large (small) input norms 
tends to have large (small) gradient norm and smoothness parameters, 
thus controlling the model sensitivity and, in turn, the excessive risk $R(\bm{x})$. 
An extended analysis of the above claim is provided in Appendix \ref{app:exp_ext}.

\subsection{The impact of the distance to decision boundary}

As mentioned in the previous section, the flipping probability $\fpx$associated with a sample $\bm{x} \in \bar{D}$ directly controls the model sensitivity  $\mathbb{E}[ \| \optimal{\btheta} - \tilde{\btheta} \|]$. Beside the discussed factors, this section further studies which characteristics of sample $\bm{x}$ can causes it to have a high flipping probability. 

Intuitively, samples close to the decision boundary are associated 
to small ensemble voting confidence and vice-versa. To illustrate this 
intuition the paper borrows the concept of \emph{closeness to 
the decision boundary} from \citet{tran2021differentially}.
\begin{definition}[Closeness to decision boundary]
\label{def:dist_boundary}
Let $f_{\btheta}$ be a $C$-classes classifier trained using 
data $\bar{D}$ with its true labels. The closeness to the decision boundary $s(\bm{x})$ is defined as:
$$  s(\bm{x}) \defeq 1- \sum_{c=1}^C f_{\btheta, c} (\bm{x})^2,$$
where  $f_{\btheta,c}$ denotes the softmax probability for class $c$.
\end{definition}
\noindent The above, (together with Theorem 5 of \cite{tran2021differentially}) 
illustrate that large (small) $s(\bm{x})$ values are associated
to close (distant) projections of point $\bm{x}$ to the model decision boundary. 
The concept of closeness to the decision boundary gives a way to indirectly quantify the flipping probability of a sample. Empirically, the correlation between the distance to decision boundary of sample $\bm{x}$ and its flipping probability $\fpx$ 
is illustrated in Figure \ref{fig:corr_a_s} (left). 
The plots are once again generated using a neural network with nonlinear objective and the relation holds for all datasets analyzed. 
Notice the strong positive correlation between these two quantities. The plot indicates that the samples that are close to the decision boundary will have a higher probability of ``flipping'' their label, thus resulting in worse excessive risks. Finally, the proportional effect of the flipping probability on the excessive risks is illustrated in Figure \ref{fig:corr_a_s} (right).
Once again, the plot clearly illustrates that large flipping probabilities $\fpx$ imply large excessive risks.
 
\begin{figure}[tb]
    \centering
    \includegraphics[width=0.35\linewidth, height=100pt]{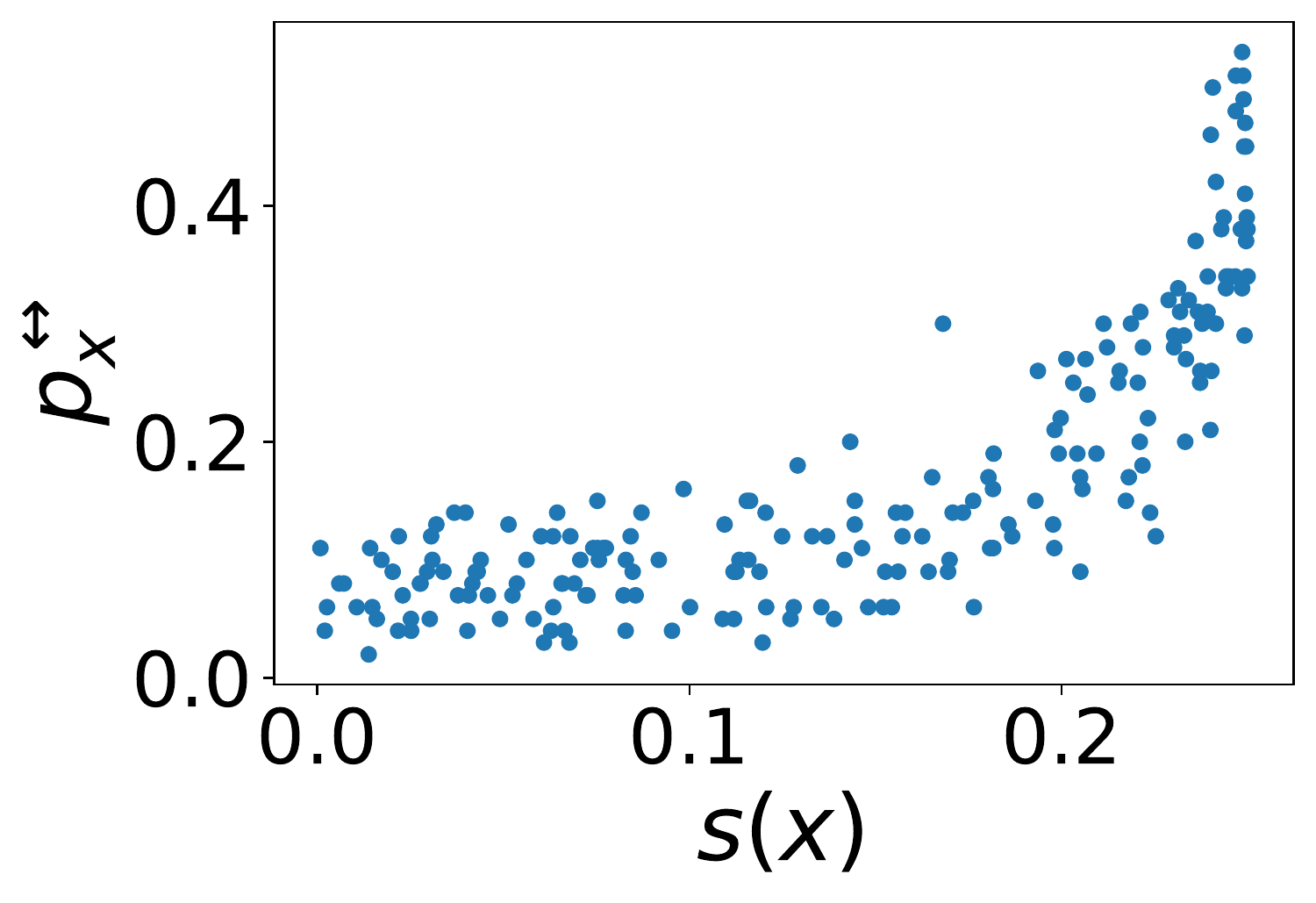}
    \includegraphics[width=0.35\linewidth,height=100pt]{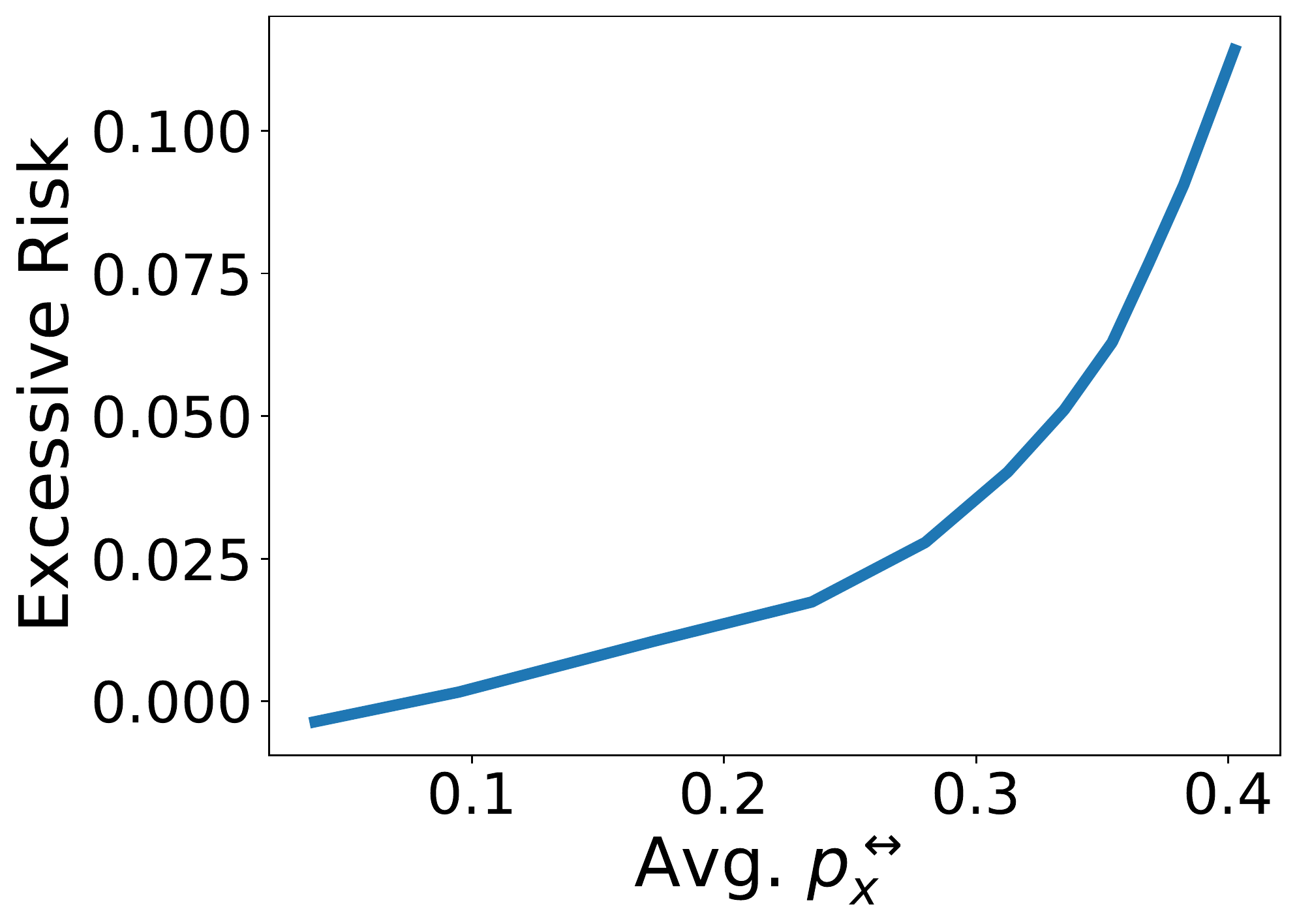}
\caption{Credit-card data: Spearman correlation between the closeness 
to decision boundary $s(\bm{x})$ and the flipping probability $\fpx$ (left)
and relation between input norms and excessive risk (right).}
\label{fig:corr_a_s}
\end{figure}

\section{Mitigation solution}
\label{sec:mitigation}

The previous sections highlighted the presence of several algorithmic 
and data-related factors which affect the disparate impact of the 
student model. 
A common role of these factors was their effects on the model 
sensitivity $\EE\| \optimal{\btheta} - \tilde{\btheta}\|$ 
which, in turn, is related with the excessive risk of different groups, whose difference we would like to minimize. 

Motivated by these observations, this section proposes a mitigating 
strategy that aims at reducing the sensitivity of the private model 
parameters. 
To do so, the paper exploits the idea of \emph{soft labels} (as 
defined below).  
When using the traditional voting process (denoted \emph{hard labels} 
in this section), in low voting confidence regimes small
perturbations (aka additive noise) may significantly affect the result of the voting scheme. 
Consider, for example, the case of a binary classifier where for a sample
$\bm{x}$, $\nicefrac{k}{2}+1$ teachers vote for label $0$ and 
$\nicefrac{k}{2}-1$ for label $1$, for some even ensemble size $k$. 
When perturbations are induced 
to these counts to guarantee privacy, the process can report the incorrect 
label ($\hat{y}=1$) with high probability. 
As a results, the private student model parameters obtained from private training with hard labels can be sensitive to the noisy voting, and may deviate significantly from the non-private one.
This issue can be partially addressed by the introduction of soft labels:
\begin{definition}[Soft label]
\label{def:soft_labels} The \emph{soft label} of a sample $\bm{x}$ is: 
$$ \bm{\alpha}(\bm{x})= \left( \frac{ \#_c(\bm{T}(\bm{x}))}{k}\right)^C_{c=1},$$
\end{definition}
and their privacy-preserving counterparts: 
$$ \tilde{\bm{\alpha}}(\bm{x})= \left( \frac{ \#_c(\bm{T}(\bm{x})) +\mathcal{N}(0, \sigma^2)}{k}\right)^C_{c=1}.$$

\begin{figure}
\centering
\includegraphics[width=0.20\linewidth,height=80pt,valign=M]{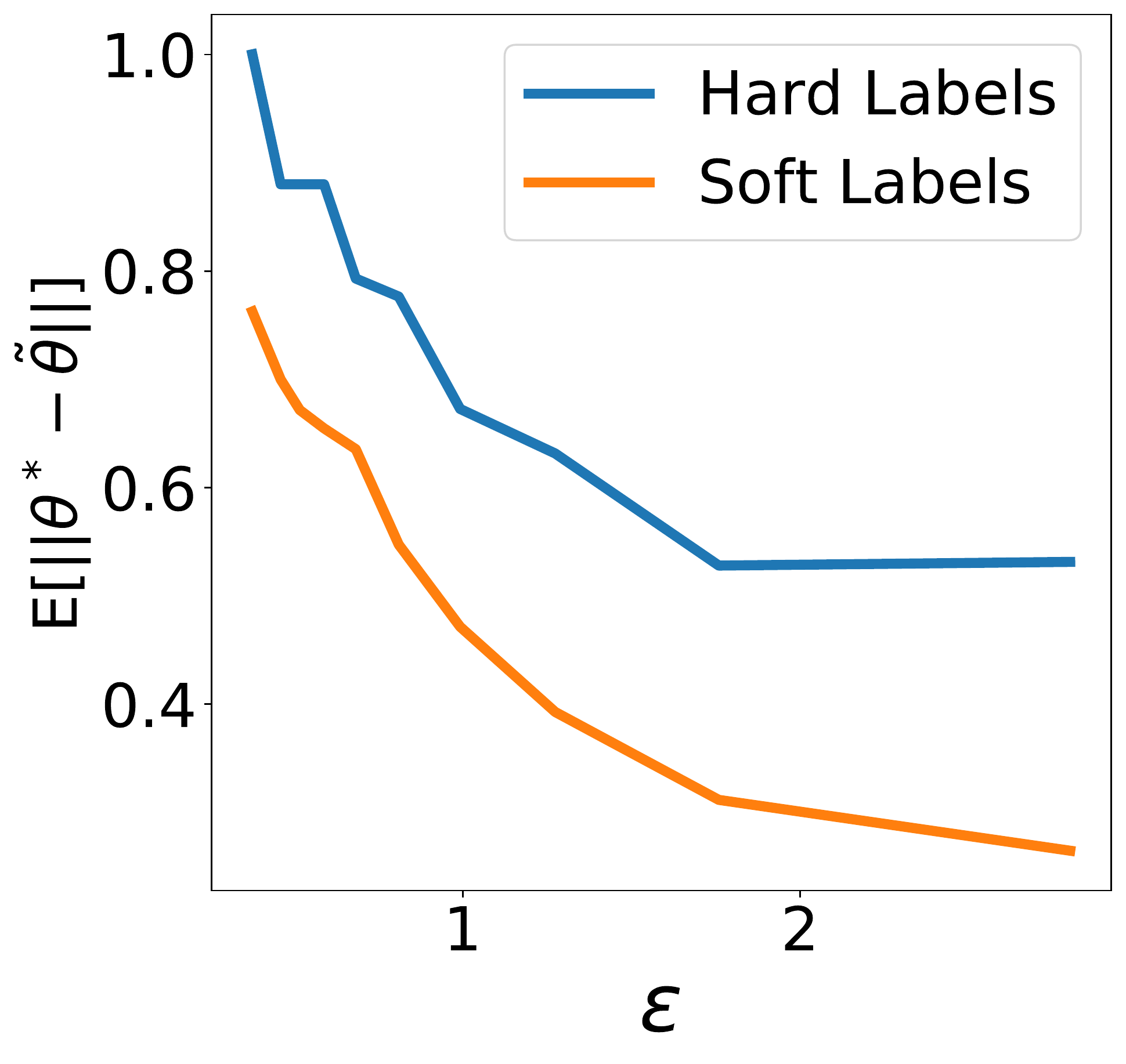}
\includegraphics[width=0.35\textwidth,height=130pt,valign=M]{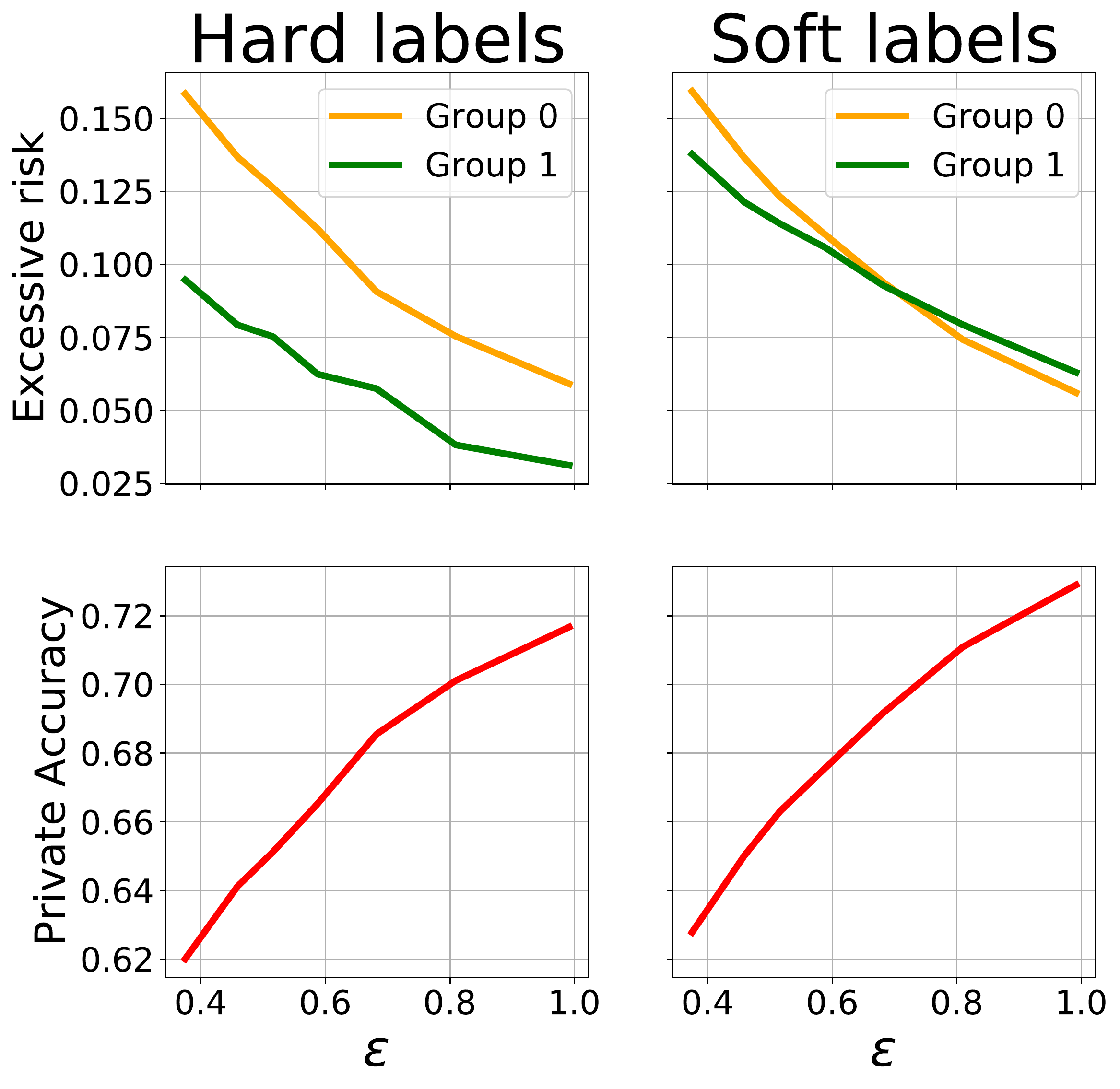}
\includegraphics[width=0.35\linewidth,height=130pt,valign=M]{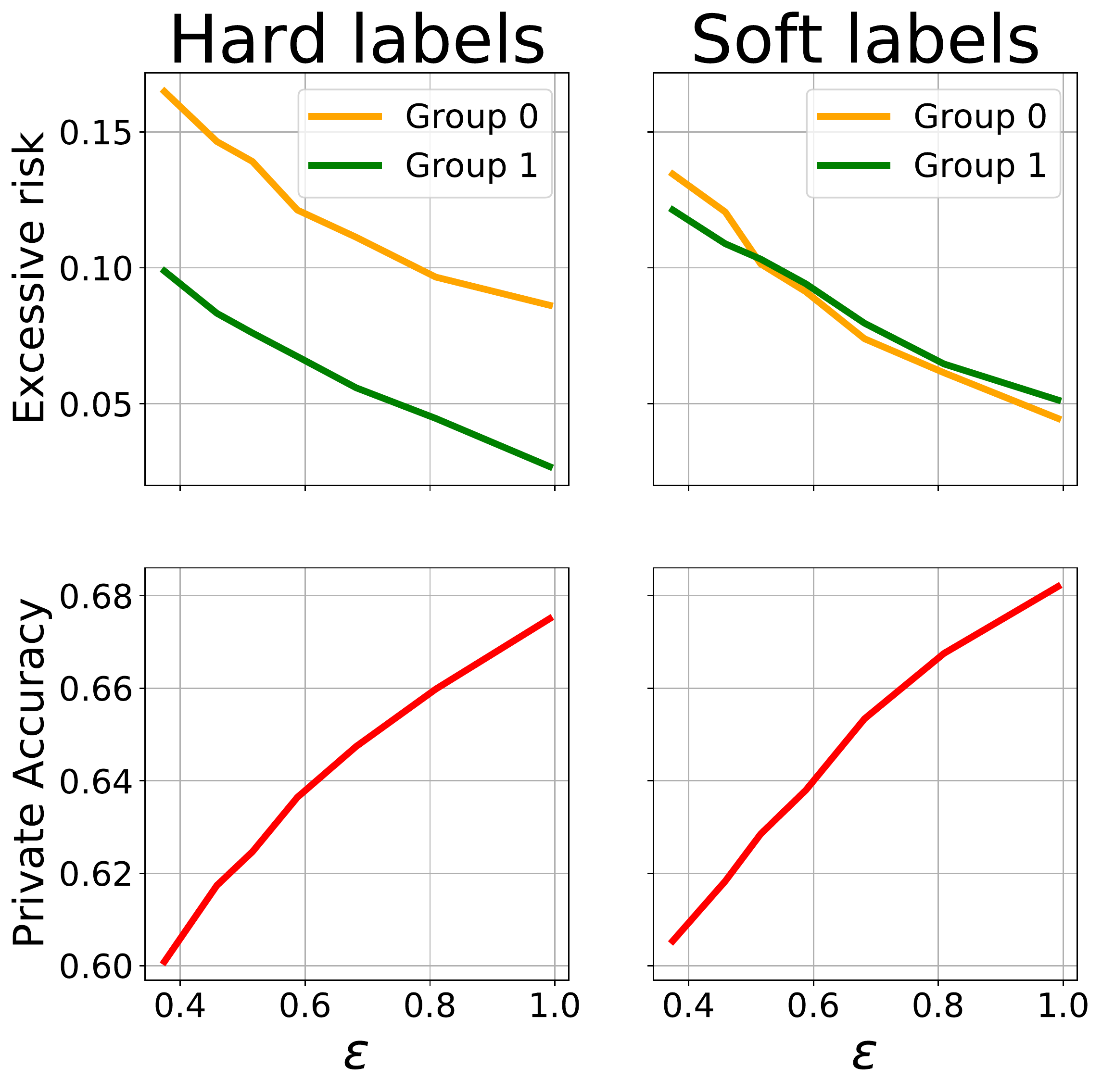}

\caption{Training privately PATE with hard and soft labels: 
Model sensitivity at varying of the privacy loss (left) on Parkinson dataset and excessive risk at varying of the privacy loss for Bank (middle) and Parkinson (right) datasets. 
}
\label{fig:mitigation_solution}
\end{figure}

To exploit soft labels, the training step of the student model is 
altered to use the following loss function:
\begin{equation}
    \ell'(\hat{f}_{\btheta}(\bm{x}), \bm{\tilde{\alpha}}) = \sum_{c=1}^C \tilde{\alpha}_c \ell(f_{\btheta}(\bm{x}), c),
\end{equation}
which can be considered as a weighted version of the original loss function 
$\ell(\hat{f}_{\btheta}(\bm{x}), c)$ on class label $c$, whose weight 
is its confidence $\tilde{\alpha}_c $. 
Note that $\ell'(\hat{f}_{\btheta}(\bm{x}), \bm{\tilde{\alpha}}) 
= \ell(\hat{f}_{\btheta}(\bm{x}))$ 
when all teachers in the ensemble chose the same label.
{The privacy analysis for this model is similar that of classical 
PATE and is reported in Appendix \ref{app:privacy_analysis}}.

The effectiveness of this scheme is demonstrated in 
Figure \ref{fig:mitigation_solution}. 
The experiment settings are reported in details in the {Appendix} and reflect those described at the end of Section \ref{sec:roadmap}.
The left subplot shows the relation between the model sensitivity 
$\mathbb{E}\left[ \| \optimal{\btheta} - \tilde{\btheta} \| \right]$ 
at varying levels of the privacy loss $\epsilon$ (dictated by the noise level 
$\sigma$). Notice how the student models trained using soft labels  
reduce their sensitivity to privacy when compared to the counterparts
that use hard labels. 

The middle and right plots of Figure \ref{fig:mitigation_solution} 
illustrate the effects of the proposed mitigating solution in terms 
of utility/fairness tradeoff on the private student model.
The top subplots illustrate the group excessive risks $R(\bar{D}_{\leftarrow 0})$
and $R(\bar{D}_{\leftarrow 1})$ associated with {minority (0) and majority (1)
groups} while the bottom subplot illustrate the accuracy of the model 
at increasing values of the privacy loss $\epsilon$. 
Notice how soft labels can reduce the disparate impacts in private 
training (top), which consistently reduces the difference in excessive risks 
between two groups, suggesting an improvement in fairness. 
Finally, notice that while fairness is improved there is seemingly no 
cost in accuracy. On the contrary, using soft labels produces comparable 
or better models to the counterparts produced with the hard labels. 

Additional experiments, including illustrating the behavior of the 
mitigating solution at varying of the number $k$ of teachers are 
reported in the appendix and the general message is consistent with 
what described above.
Finally, an important benefit about the proposed solution is that it \emph{does not} require the protected group information ($a \in \cA$) to be part of the training data. Thus, it is applicable in challenging situations when it is not feasible to collect or use protected features (e.g., under the General Data Protection Regulation (GDPR)  \cite{lahoti2020fairness}). 

\emph{These results are significant. They suggest that this mitigating 
solution can be an effective strategy for improving the disparate impact of private model ensembles without sacrificing accuracy.}

\section{Discussion}
We note that the proposed mitigating solution relates to 
concepts explored in robust machine learning. In particular, 
\citet{papernot2016distillation} noted that training a classifiers 
with soft labels can increase its robustness against adversarial samples. 
This connection is not coincidental. Indeed, the model sensitivity
is affected by the voting outcomes of the teacher ensemble (Theorems \ref{thm:1} and \ref{thm:3}). 
Similarly to robust ML models being insensitive to input perturbations, strongly agreeing ensemble will be less sensitive to noise and vice-versa. 

Finally, we notice that the use of more advanced voting schemes, such as the interactive GNMAX \cite{papernot2018scalable}, may produce different fairness results. While this is an interesting avenue for extending our analysis, sophisticated voting schemes may introduce sampling bias (e.g., interactive GNMAX may exclude samples with low ensemble voting agreement). Such bias may trigger some nontrivial unfairness issues on its own. 

\section{Conclusions}
This work was motivated by the recent observations regarding the 
effects of differential privacy to the disparate impacts of machine 
learning models. The paper introduced a notion of fairness that relies 
on the concept of excessive risk and analyzed this notion in the 
Private Aggregation of Teacher Ensembles (PATE) \cite{papernot2018scalable}, 
an important privacy-preserving machine learning framework used in 
semisupervised settings or when one wishes to protect the data
labels. 
This paper isolated key components related with the algorithms parameters
and the public training data characteristics which are responsible 
for exacerbating the disparate impacts, it studied the factors affecting
these components, and introduced a mitigation solution. 

Given the increasing presence of privacy-preserving data-driven algorithms 
in consequential decisions, we believe that this work may represents an 
important and broadly applicable step toward understanding the sources
of disparate impacts observed in differentially private learning systems.

\bibliographystyle{abbrvnat}
\bibliography{lib}

\newpage
\appendix

\setcounter{theorem}{3}
\setcounter{corollary}{2}
\setcounter{lemma}{1}
\setcounter{proposition}{2}

\section{Privacy Analysis}
\label{app:privacy_analysis}
This section provides the privacy analysis for the proposed mitigation solution.  In PATE with the noisy-max scheme presented in Equation \eqref{eq:noisy_max} of the main paper (also called GNMAX), the privacy budget is used for releasing the voting labels $ \tilde{\textsl{v}}(\bm{T}(\bm{x}_i))$ (a.k.a.~hard labels) for each of the $m$ public data samples $\bm{x}_i \in \bar{D}$ according to:

\begin{equation}
    \tilde{\textsl{v}}\left(\bm{T}(\bm{x}_i)\right) 
    = 
    \argmax_c \left\{ \#_c\left(\bm{T}(\bm{x}_i)\right) + 
    \cN\left(0, \sigma^2\right) \right\}.
\end{equation}

The proposed mitigation solutions, instead, releases privately the voting counts $(\#_c(\bm{T}(\bm{x}_i)) \!+\! \cN(0, \sigma^2))_{c=1}^C$ and use these noisy counts to construct the \emph{soft-labels}, see Equation \eqref{def:soft_labels}.  

Using an analogous analysis as that provided in \cite{papernot2018scalable}, adding or removing one individual sample $\bm{x}$ from any disjoint partition $D_i$ of $D$ can change the voting count vector by at most two. This value of the query sensitivity is obtained by GNMAX \cite{papernot2018scalable}. 
Therefore the privacy cost for releasing hard labels or soft-labels is equivalent. 

Next, this section provides the privacy computation $\epsilon$ given by Gaussian mechanism which adds Gaussian noise with standard deviation $\sigma$ to the voting counts.

The privacy analysis of PATE with hard or soft-labels is based on the concept of Renyi differential privacy (RDP) \cite{Mironov_2017}. In either implementations, the process uses the Gaussian mechanism to add independent Gaussian noise to the voting counts. The following Proposition \ref{prop:3} (from \cite{papernot2018scalable}) derives the privacy guarantee for GNMAX.

\begin{proposition} 
\label{prop:3}
The GNMAX aggregator with private Gaussian noise $\mathcal{N}(0, \sigma^2)$ satisfies $(\gamma,\nicefrac{\gamma}{\sigma^2}) $-RDP for all $\gamma \geq 1$.
\end{proposition}

Since the GNMAX mechanism is applied on $m$ public data samples from $\bar{D}$, the total privacy loss spent to provide the private labels is derived by the following composition theorem.

\begin{theorem}[Composition for RDP]
\label{app:rdp_comp}
If a mechanism $\cM$ consists of a sequence of adaptive mechanisms $\cM_1, \cM_2, \ldots, \cM_m$ such that for any $i \in [m]$, $\cM_i$ guarantees $(\gamma, \epsilon_i)$-RDP, then $\cM$ guarantees $(\gamma, \sum_{i=1}^m \epsilon_i)$-RDP.
\end{theorem}

Based on Theorem \ref{app:rdp_comp} and Proposition \ref{prop:3}, PATE satisfies $(\gamma, \nicefrac{m\gamma}{\sigma^2})$-RDP.  PATE also satisfies $(\epsilon, \delta)$-DP by the following theorem.

\begin{theorem}[From RDP to DP]
\label{thm:rdp_2_dp}
If a mechanism $\cM$ guarantees $(\gamma, \epsilon)$-RDP, then $\cM$ guarantees $(\epsilon +  \frac{\log \nicefrac{1}{\delta}}{\gamma-1}, \delta)$-DP for any $\delta \in (0,1)$.
\end{theorem}

Thus, based on Theorem \ref{thm:rdp_2_dp}, PATE (with either hard or soft labels) satisfies $( \nicefrac{m\gamma}{\sigma^2} +\frac{\log \nicefrac{1}{\delta}}{\gamma -1}, \delta)$-DP.

\setcounter{theorem}{0}
\setcounter{corollary}{0}
\setcounter{lemma}{0}
\setcounter{proposition}{0}

\section{Missing Proofs}
\label{app:missing_proofs}
This section contains the missing proofs associated with the theorems
presented in the main paper. The theorems are restated for completeness.

\begin{theorem}
\label{app:thm:1}
Consider a student model $\bar{f}_{\btheta}$ trained with a convex and decomposable loss 
function $\ell(\cdot)$. Then, the expected 
difference between the private and non-private model parameters is 
upper bounded as follows:
\begin{equation}
    \mathbb{E}\left[ \| \optimal{\btheta} - \tilde{\btheta} \| \right] 
    \leq \frac{|c|}{m\lambda} \left[ \sum_{\bm{x} \in \bar{D}} p^{\leftrightarrow}_{\bm{x}} \| g_{\bm{x}}\| \right],
\end{equation}
where $c$ is a real constant and
$g_{\bm{x}}= \max_{\btheta}\| \nabla_{\btheta} h_{\btheta}(\bm{x}) \|$ 
represents the maximum gradient norm distortion introduced by a 
sample $\bm{x}$. Both $c$ and $h$ are defined as in Equation 
\eqref{eq:decomposable}. 
\end{theorem}

Proof of Theorem \ref{app:thm:1} requires the following Lemma \ref{app:lem:1} from \cite{opt_paper} on the property of strongly convex functions.
\begin{lemma}[\citet{opt_paper}]
\label{app:lem:1}
Let $\cL(\btheta)$ be a differentiable function. 
Then $\cL(\btheta)$ is $\lambda$-strongly convex \emph{iff} for all vectors
 $\btheta, \btheta'$: 
\begin{equation}
    \hugP{\nabla_{\btheta}\cL  - \nabla_{\btheta'} \cL}^T 
    \hugP{\btheta - \btheta'} \geq \lambda 
    \left\| \btheta -\btheta' \right\|^2.
\end{equation}
\end{lemma}

\begin{proof}[Proof of Theorem 1]
Denote with $\hat{y}_i =   \textsl{v}(\bm{T}(\bm{x}_i))$ the non-private voting label associated with $\bm{x}_i$ and $\tilde{y}_i =   \tilde{\textsl{v}}(\bm{T}(\bm{x}_i))$ for the private voting label counterpart.
The regularized empirical risk function (Equation \eqref{eq:ERM}) that uses the non-private voting labels can be rewritten as follows:
\begin{align}
\cL & = \frac{1}{m} \sum_{i=1}^m \ell 
\hugP{\bar{f}_{\btheta}(\bm{x}_i),\hat{y}_i}  +\lambda
 \left\| \btheta \right\|\\
& = \frac{1}{m} \sum_{i=1}^m \left[ z(h_{\btheta}(\bm{x}_i)) + 
c \hat{y}_i h_{\btheta}(\bm{x}_i) \right] +\lambda \left\| \btheta \right\|^2,
\label{eq:hat_L}
\end{align}
where the second equality is due to the decomposable loss assumption.
Likewise, define $\tilde{\cL}$ to be the regularized empirical risk function  with  private voting labels $\tilde{y}_i$:
\begin{align}
\label{eq:tilde_L}
\tilde{\cL}  = \frac{1}{m} \sum_{i=1}^m \left[z(h_{\btheta}(\bm{x}_i)) + c \tilde{y}_i h_{\btheta}(\bm{x}_i) \right] +\lambda 
\left\| \btheta \right\|^2,
\end{align}

\noindent Based on Equation \eqref{eq:hat_L}  and Equation\eqref{eq:tilde_L}, it follows that: $\tilde{\cL} = \cL + \Delta_{\cL}$ where 
$$
\Delta_{\cL} = \frac{c}{m} \sum_{i=1}^m(\tilde{y}_i -\hat{y}_i) h_{\btheta}(\bm{x}_i).
$$

\noindent 
Furthermore, since each individual loss function  $\ell(\bar{f}_{\btheta}(\bm{x}_i),\tilde{y}_i)$ and $\ell(\bar{f}_{\btheta}(\bm{x}_i),\hat{y}_i)$ is convex  for all $i \in [m]$, by assumption, then $\tilde{\cL}$ and $\cL$ both are $\lambda$-strongly convex. 

Next, from the definition of $\tilde{\btheta} = \argmin_{\btheta} \tilde{\cL}$,  and  $\optimal{\btheta}\ = \argmin_{\btheta} \cL$ it follows that: 

\begin{equation}
\nabla_{\tilde{\btheta}} \tilde{\cL} =\boldsymbol{0} \ \mbox{and }  \nabla_{\optimal{\btheta}} \cL =\boldsymbol{0}. 
\label{eq:4}
\end{equation}

By Lemma \ref{app:lem:1}, it follows that: 

\begin{equation}
    \left( \nabla_{\tilde{\btheta}}\tilde{\cL} - \nabla_{\optimal{\btheta}} \tilde{\cL} \right)^T 
    \left(\tilde{\btheta} - \optimal{\btheta}\right) \geq 
    \lambda \left\| \tilde{\btheta} - \optimal{\btheta} \right\|^2. 
    \label{eq:5}
\end{equation}

Now since $\nabla_{\tilde{\btheta}}\tilde{\cL} = \bm{0}$ by Equation \eqref{eq:4}, we can rewrite Equation \ref{eq:5} as

\begin{equation}
    \left( - \nabla_{\optimal{\btheta}} \tilde{\cL} \right)^T 
    \left(\tilde{\btheta} - \optimal{\btheta}\right) \geq 
    \lambda \left\| \tilde{\btheta} - \optimal{\btheta} \right\|^2,
    \label{eq:5b}
\end{equation}
since $ \nabla_{\optimal{\btheta}} \tilde{\cL}  = \nabla_{\optimal{\btheta}} \cL + \nabla_{\optimal{\btheta}} \Delta_{\cL}  = \boldsymbol{0} + \nabla_{\optimal{\btheta}} \Delta_{\cL}  =  \nabla_{\optimal{\btheta}} \Delta_{\cL}   $. In addition, by applying the Cauchy-Schwartz inequality to the L.H.S of Equation \eqref{eq:5b} we obtain
\begin{equation}
  \left\|\nabla_{\optimal{\btheta}}\Delta_{\cL} \right\|  
  \left\|(\optimal{\btheta} - \tilde{\btheta})  \right\| \geq  
  -\left(\nabla_{\optimal{\btheta}}\Delta_{\cL}\right)^T 
  \left(\tilde{\btheta} - \optimal{\btheta}\right)  
  \geq 
  \lambda \left\| \tilde{\btheta} - \optimal{\btheta} \right\|^2,
  \label{eq:6}
\end{equation}
and thus,
\begin{equation}
\left\|\nabla_{\optimal{\btheta}}\Delta_{\cL} \right\|  
\geq \lambda 
\left\| \tilde{\btheta} - \optimal{\btheta} \right\|.
\label{eq:7}
\end{equation}

By definition of $\nabla_{\optimal{\btheta}}\Delta_{\cL}$ we can rewrite the above inequality as follows:
\begin{align}
\label{eq:7b}
   \left\|  \nabla_{\optimal{\btheta}}\Delta_{\cL} \right\| 
   &= 
   \left\| \frac{c}{m} \sum_{i=1}^m (\tilde{y}_i -\hat{y}_i) 
    \nabla_{\optimal{\btheta}} h_{\optimal{\btheta}}(\bm{x}_i) \right\| 
    \geq \lambda 
    \left\| \tilde{\btheta} -\optimal{\btheta} \right\|^2.
\end{align}

Next, let $\rho_i = \hat{y}_i - \tilde{y}_i$, applying this substitution to the above and by triangle inequality it follows that
\begin{align}
   \frac{|c|}{m} \sum^m_{i=1}| \rho_i| \|g_i\|   &\geq \frac{|c|}{m} 
   \sum^m_{i=1}| \rho_i| 
   \left\| \nabla_{\tilde{\btheta}} h_{\tilde{\btheta}}(\bm{x}_i) \right\|\\
   &\geq 
   \left\|\frac{c}{m}\sum^m_{i=1}\rho_i \nabla_{\tilde{\btheta}} h_{\tilde{\btheta}}(\bm{x}_i) \right\|
   \geq \lambda 
   \left\| \tilde{\btheta} - \optimal{\btheta} \right\|,
\end{align}
where the first inequality is due to definition of $g_{\bm{x}_i} = \max_{\btheta}\| \nabla_{\btheta} h_{\btheta}(\bm{x}_i) \|$ and the second inequality is due to the general triangle inequality .
Since $| \rho_i |$ is a Bernoulli random variable, in which $|\rho_i| = 1 $ w.p.~$p^{\leftrightarrow}_{\bm{x}_i}$  and $|\rho_i| = 0$ w.p.~$1-p^{\leftrightarrow}_{\bm{x}_i}$. Therefore $\mathbb{E}[ |\rho_i|]=p^{\leftrightarrow}_{\bm{x}_i}$. Thus, it follows that:

\begin{align}
   \mathbb{E}
   \left[ \frac{|c|}{m} \sum^m_{i=1}| \rho_i| \|g_{\bm{x}_i}\| \right]
    = 
    \frac{|c|}{m} \sum^m_{i=1} p^{\leftrightarrow}_{\bm{x}_i} 
    \|g_{\bm{x}_i}\|  
    \geq 
    \lambda 
    \mathbb{E} 
    \left[ \| \tilde{\btheta} - \optimal{\btheta} \| \right],
\end{align}
which concludes the proof.
\end{proof}

\begin{theorem}
\label{app:thm:2}
For a sample $\bm{x} \!\in\! \bar{D}$ assume that the teacher models 
outputs $f^i(\bm{x})$ are in agreement for all $i \in [k]$. 
Then, the flipping probability $\fpx$ is given by:
\begin{equation}
     \fpx = 1 - \Phi\left(\frac{k}{\sqrt{2} \sigma}\right),
\end{equation}
where $\Phi(\cdot)$ is the CDF of the standard normal distribution, 
and $\sigma$ is the standard deviation in the Gaussian mechanism.
\end{theorem}
For simplicity of exposition Theorem 2 considers binary classifiers, i.e., $\mathcal{Y} = \{0,1\}$. The argument, however, can be trivially extended to generic $C$-classifiers. 

\begin{proof}
By assumption, for any given sample $\bm{x}$, all teachers agree in their predictions, so w.l.o.g., assume $k$ teachers output label $0$, while none of them outputs  label $1$.  
Next, let $\psi, \psi' \sim \mathcal{N}(0, \sigma^2)$ be two independent Gaussian random variables which are added to true voting counts, $k$ and $0$, respectively. 
The associated flipping probability is:
\begin{align}
     \fp{\bm{x}} & = \Pr\left( 
    \tilde{\textsl{v}}\left(\bm{T}(\bm{x})\right) 
    \neq \textsl{v}\left(\bm{T}(\bm{x})\right)\right)
    = 
    \Pr\left( k+ \psi \leq 0 + \psi' \right)
    = 
    \Pr\left(\psi' - \psi \geq k\right) \\  
    &= 1 - \Pr\left(\psi - \psi'\leq k\right),
\end{align}
since $\psi, \psi'$ are two independent Gaussian random variable with zero mean and  standard deviation of $\sigma$. 
Therefore, $\psi' - \psi \sim \mathcal{N}(0, 2\sigma^2)$. Thus:

$$\Pr\left(\psi - \psi' \leq k\right) 
= \Pr\left(\mathcal{N}(0, 2\sigma^2) \leq k\right) 
= \Phi\left(\frac{k}{\sqrt{2}\sigma}\right).$$

Hence, the flipping probability will be: $    \fp{\bm{x}}= 1- \Phi(\frac{k}{\sqrt{2} \sigma})$.
\end{proof}

\begin{corollary}[Theorem~\ref{thm:1}]
\label{app:cor:1}
Let $\bar{f}_{\btheta}$ be a \emph{logistic regression} classifier. Its 
expected model sensitivity is upper bounded as:
\begin{equation}
\label{app:eq:8}
 \mathbb{E}\left[ \| \optimal{\btheta} - \tilde{\btheta} \| \right] 
    \leq 
    \frac{1}{m\lambda} \left[ \sum_{\bm{x} \in \bar{D}} \fp{\bm{x}} \| \bm{x} \| \right]. 
 \end{equation}
\end{corollary}

\begin{proof}
The loss function $\ell(\bar{f}_{\btheta}(\bm{x}), y)$ of a logistic regression classifier with binary cross entropy loss can be rewritten as follows:
\begin{align}
    \ell\left(\bar{f}_{\btheta}\left(\bm{x}\right), y\right) & = -y \log\left(\frac{1}{1+\exp\left(-\btheta^T \bm{x}\right)}\right) 
    - \left(1-y\right)   \log\left(\frac{\exp\left(-\btheta^T x\right)}{1+\exp\left(-\btheta^T \bm{x}\right)}\right)\\
    & = y\log\left(\exp\left(-\btheta^T \bm{x}\right)\right) - \log\left(\frac{\exp\left(-\btheta^T x\right)}{1+\exp\left(-\btheta^T \bm{x}\right)}\right)\\
    & = y\left(-\btheta^T \bm{x}\right)   - \log\left(\frac{\exp\left(-\btheta^T \bm{x}\right)}{1+\exp\left(-\btheta^T \bm{x}\right)}\right).
\end{align}
Hence, $\ell(\cdot)$ is decomposable by Definition \ref{def:1} with $h_{\btheta}(\bm{x})=-\btheta^T x$, $c=1$ and $z(h) =-\log(\frac{\exp(h)}{1+\exp(h)}) $. 

Applying Theorem \ref{app:thm:1} with $g_{\bm{x}} = \max_{\btheta} \|\nabla_{\btheta} h_{\btheta}(\bm{x}) \| = \max_{\btheta} \| \nabla_{\btheta}  -\btheta^T \bm{x} \| = \| \bm{x}\|$, and $c=1$, gives the intended result. 

\end{proof}


\begin{corollary}[Theorem~\ref{thm:1}]
\label{app:cor:2}
Given the same settings and assumption of Theorem \ref{thm:1}, it follows:
\begin{equation}
\label{app:eq:9}
    \mathbb{E}\left[ \| \optimal{\btheta} - \tilde{\btheta} \|^2 \right] 
    \leq \frac{|c|^2}{m \lambda^2} \left[ \sum_{\bm{x} \in \bar{D}} p^{\leftrightarrow 2}_{\bm{x}} \| g_{\bm{x}}\|^2 \right].
\end{equation}
\end{corollary}

\begin{proof}
First, by Theorem \ref{app:thm:1} we obtain an upper bound for $ \mathbb{E}\left[ \| \optimal{\btheta} - \tilde{\btheta} \|^2 \right] $ as follows:

\begin{align}
     \mathbb{E}\left[ \| \optimal{\btheta} - \tilde{\btheta} \|^2 \right]  \leq \frac{c^2}{\lambda^2} \left[ \frac{1}{m}\sum_{\bm{x} \in \bar{D}} p^{\leftrightarrow}_{\bm{x}} \| g_{\bm{x}}\| \right]^2.
     \label{eq:second}
\end{align}

Applying the sum of squares inequality on the R.H.S.~of Equation \eqref{eq:second} we obtain:
\begin{align}
     \frac{c^2}{\lambda^2} \left[ \frac{1}{m}\sum_{\bm{x} \in \bar{D}} p^{\leftrightarrow}_{\bm{x}} \| g_{\bm{x}}\| \right]^2 \leq  \frac{c^2}{\lambda^2} \left[ \frac{1}{m}  p^{\leftrightarrow 2}_{\bm{x}} \| g_{\bm{x}}\|^2 \right],
\end{align}
which concludes the proof.

\end{proof}

\begin{theorem}
\label{app:thm:3}
Let $\ell(\cdot)$ be a $\beta_{\bm{x}}$-smooth loss function. 
The excessive risk $R(\bm{x})$ of a sample $\bm{x}$ is upper 
bounded as:
\begin{equation}
\label{app:eq:ER_ub}
R(\bm{x}) \leq \left\| \nabla_{\optimal{\btheta}}  
        \ell\left(\bar{f}_{\optimal{\btheta}}(\bm{x}),y\right) 
        \right\| U_1 + \frac{1}{2} \beta_{\bm{x}} U_2,
\end{equation} 
where, $U_1 = \mathbb{E}\left[ \| \optimal{\btheta} - \tilde{\btheta} \| \right]$ 
and  $U_2 = \mathbb{E}\left[ \| \optimal{\btheta} - \tilde{\btheta} \|^2 \right]$
capture the first and second order statistics of the model sensitivity. 
\end{theorem}

\begin{proof}
By $\beta_{\bm{x}}$ smoothness assumption on the loss function  at a sample $\bm{x}$, it follows that:
\begin{align}
    \ell\left( \bar{f}_{\tilde{\btheta}}(\bm{x}),y\right) 
    & 
    \leq 
    \ell\left(\bar{f}_{\optimal{\btheta}}(\bm{x}),y\right) 
    + 
    \nabla_{\optimal{\btheta}}  
    \ell\left(\bar{f}_{\optimal{\btheta}}(\bm{x}),y\right)^T 
    \left(\tilde{\btheta} - \optimal{\btheta}\right) + 
    \frac{\beta_x}{2}\left\| 
    \tilde{\btheta} - \optimal{\btheta}\right\|^2.
\end{align}
By taking the expectation on both sides of the above equation w.r.t.~the randomness of the noise, we obtain:
\begin{align}
    \mathbb{E}\left[
        \ell\left( \bar{f}_{\tilde{\btheta}}(\bm{x}),y\right)\right] 
        & \leq 
        \ell\left(\bar{f}_{\optimal{\btheta}}(\bm{x}),y\right) + 
        \nabla_{\optimal{\btheta}}  
        \ell\left(\bar{f}_{\optimal{\btheta}}(\bm{x}),y\right)^T 
        \mathbb{E}\left[ (\tilde{\btheta} - \optimal{\btheta})\right]
     + 
     \frac{\beta_x}{2} 
     \mathbb{E}\left[ \|\tilde{\btheta} - \optimal{\btheta}\|^2 \right] \\
    & \leq 
    \ell\left(\bar{f}_{\optimal{\btheta}}(\bm{x}),y\right) 
    + 
    \left\| \nabla_{\optimal{\btheta}}  
    \ell\left(\bar{f}_{\optimal{\btheta}}(\bm{x}),y\right) 
    \right\| 
    \mathbb{E}\left[ \| \optimal{\btheta} - \tilde{\btheta} \| \right]  + \frac{1}{2} \beta_x 
    \mathbb{E}\left[\|\tilde{\btheta} - \optimal{\btheta}\|^2\right],
    \label{eq:final_ineq}
\end{align}
where the last inequality is by Cauchy-Schwarz inequality on vectors. Next, by substituting $R(\bm{x}) =\mathbb{E}[\ell(\bar{f}_{\tilde{\btheta}}(\bm{x}),y)] -\ell(\bar{f}_{\optimal{\btheta}}(\bm{x}),y) $,  $U_1 = \mathbb{E}\left[ \| \optimal{\btheta} - \tilde{\btheta} \| \right]$ 
and  $U_2 = \mathbb{E}\left[ \| \optimal{\btheta} - \tilde{\btheta} \|^2 \right]$, with their definitions into Equation \eqref{eq:final_ineq} we obtain the statement in Theorem \ref{app:thm:3}.

\end{proof}

\section{Extended Experimental Analysis}
\label{app:exp_ext}
This section reports detailed information about the experimental setting
as well as additional results conducted on the Income, Bank, Parkinsons and Credit Card datasets. 

\subsection{Setting and Datasets}

\smallskip\noindent\textbf{Computing Infrastructure} 
All of our experiments are performed on a distributed cluster equipped with Intel(R) Xeon(R) Platinum 8260 CPU @ 2.40GHz and 8GB of RAM.

\smallskip\noindent\textbf{Software and Libraries}  
All models and experiments were written in Python 3.7. All neural network classifier models in our paper were implemented in Pytorch 1.5.0. 

The Tensorflow Privacy package was also employed for computing the privacy loss. 


\smallskip\noindent\textbf{Datasets}  
This paper evaluates the fairness analysis of PATE on the following four UCI datasets: \emph{Bank}, \emph{Income}, \emph{Parkinsons} and \emph{Credit card} dataset. 
A descriptions of each dataset is reported as follows:

\begin{enumerate}
    \item \textbf{Income} (Adult) dataset, where the task is to predict if an 
    individual has low or high income, and the group labels are defined by race: 
    \emph{White} vs \emph{Non-White} \cite{UCIdatasets}.

    \item \textbf{Bank} dataset, where the task is to predict if a user subscribes 
    a term deposit or not and the group labels are defined by age: 
    \emph{people whose age is less than vs greater than 60 years old} \cite{UCIdatasets}. 
    
    \item \textbf{Parkinsons} dataset, where the task is to predict if 
    a patient has total UPDRS score that exceeds the median value, 
    and the group labels are defined by gender: \emph{female vs male} \cite{article}.  
       
    \item \textbf{Credit Card} dataset, where the task is to predict if 
    a customer defaults a loan or not. The group labels are defined by gender:
    \emph{female vs male} \cite{creditdataset}. 
\end{enumerate}

Each dataset has been standardized to render its features having zero mean and unit standard deviation. Each dataset was partitioned into three disjoint subsets: private set, public train,  and test set, as follows. 75\% of the dataset was used as private data and the rest for public data. For the public data, 200 samples were randomly selected to train the student model and the rest of the data was used as a test set to evaluate that model.

\smallskip\noindent\textbf{Models' Setting} 

To illustrate the tightness of the upper bound provided in Corollary \ref{app:cor:1}, the paper uses a logistic regression model executed over 1000 runs to estimate the expected model sensitivity 
$\mathbb{E}\left[ \| \optimal{\btheta} - \tilde{\btheta} \| \right]$. 
For all other experiments, the paper uses a neural network with two hidden layers and nonlinear ReLU activations for both the ensemble and student models. 
All reported metrics are an average of 100 repetitions, used to compute the empirical expectations. The batch size for stochastic gradient descent was fixed to 32 and the learning rate to $\eta = 1e-4$.

\subsection{The impact of regularization parameter} 
This section provides further empirical supports regarding impact of the regularization parameter $\lambda$ to the accuracy and fairness trade-off. As shown in Theorem \ref{thm:1}, increasing $\lambda$ reduces the model sensitivity $\mathbb{E}\left[ \| \optimal{\btheta} - \tilde{\btheta} \| \right]$, 
which in turns decreases the  group excessive risk $R(\bar{D}_{\leftarrow a})$ (for any group $a \in \cA$) by Theorem \ref{thm:3}. 
On the other hand, large regularization can negatively impact the model accuracy. Figure \ref{app:fig:lambda_effect} illustrates this discussion. It shows how model sensitivity (left), excessive risk difference between two groups (middle), and utility (right) vary according to $\lambda$. 

\begin{figure}
\centering
\begin{subfigure}[b]{0.485\textwidth}
\includegraphics[width = 1.0\linewidth]{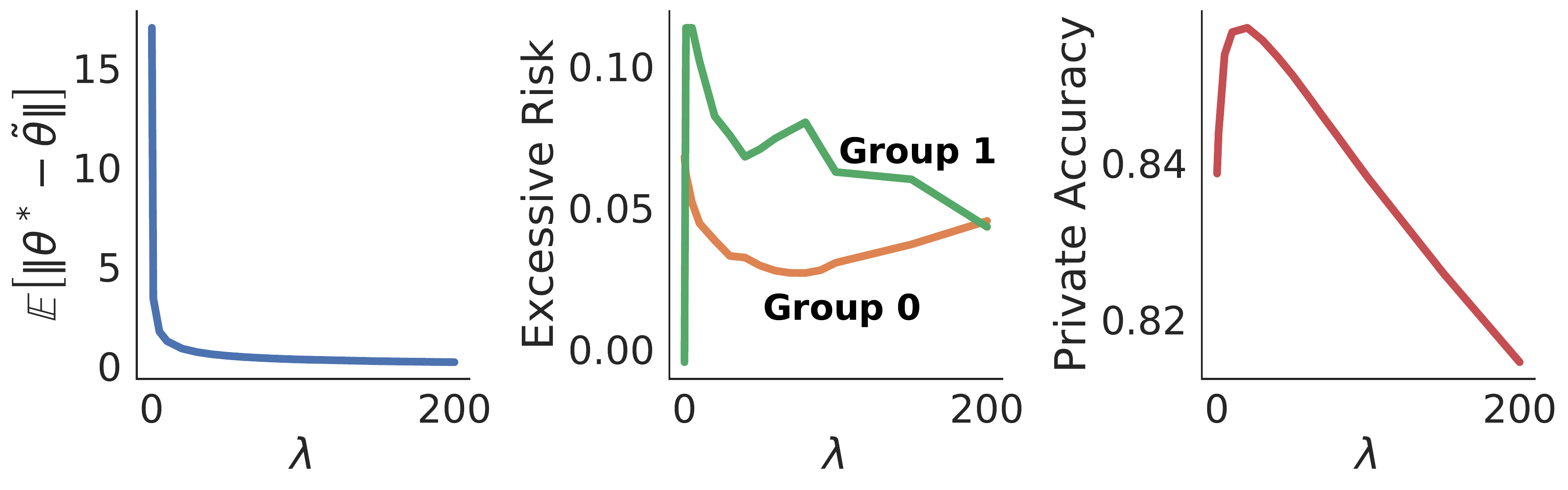}
\caption{Bank dataset}
\end{subfigure}
\begin{subfigure}[b]{0.485\textwidth}
\includegraphics[width = 1.0\linewidth]{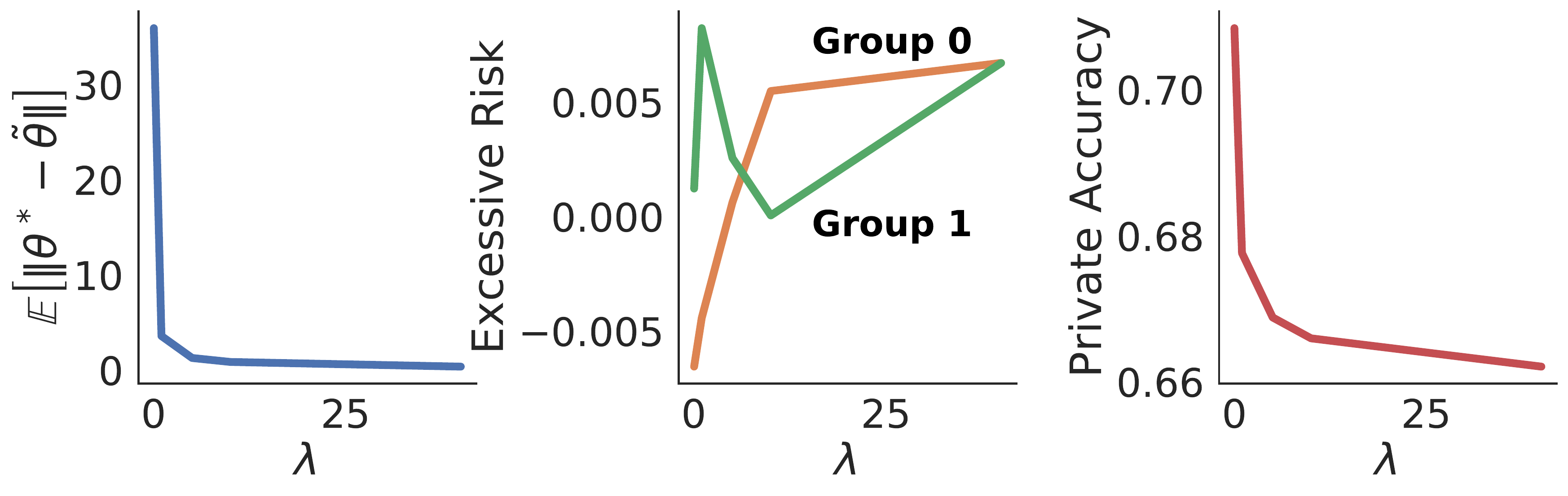}
\caption{Income dataset}
\end{subfigure}
\begin{subfigure}[b]{0.485\textwidth}
\includegraphics[width = 1.0\linewidth]{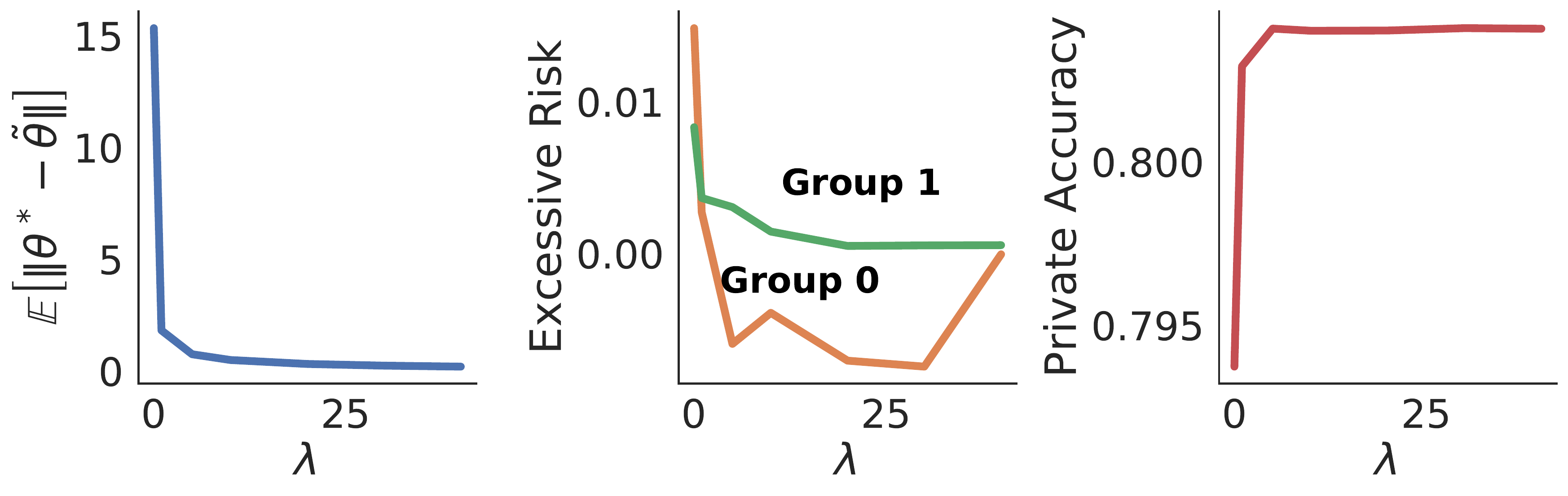}
\caption{Parkinsons dataset}
\end{subfigure}
\caption{Expected model sensitivity (left), empirical risk (middle), and model accuracy (right) as a function of the  regularization. Here for each dataset, number of teacher k =150, $\sigma = 50$.}.
\label{app:fig:lambda_effect}
\end{figure}

\subsection{The impact of teachers ensemble size k} 
This section illustrates the effects of varying the teacher ensemble sizes $k$ with respect to two factors: (1) the flipping probability $\fpx$, and (2) the trade-offs among the model sensitivity $\mathbb{E}\left[ \| \optimal{\btheta} - \tilde{\btheta} \| \right]$ and the model fairness and utilities. 

Recall that Theorem \ref{thm:2} shows that larger $k$ values correspond to smaller flipping probability $\fpx$. This dependency is reported in Figure \ref{fig:flip_prob_all}. Notice how increasing the number of teachers $k$ reduces the flipping probability $\fpx$ on all samples $\bm{x}$. 

\begin{figure}
\centering
\begin{subfigure}[b]{0.3\textwidth}
\includegraphics[width = 1.0\linewidth]{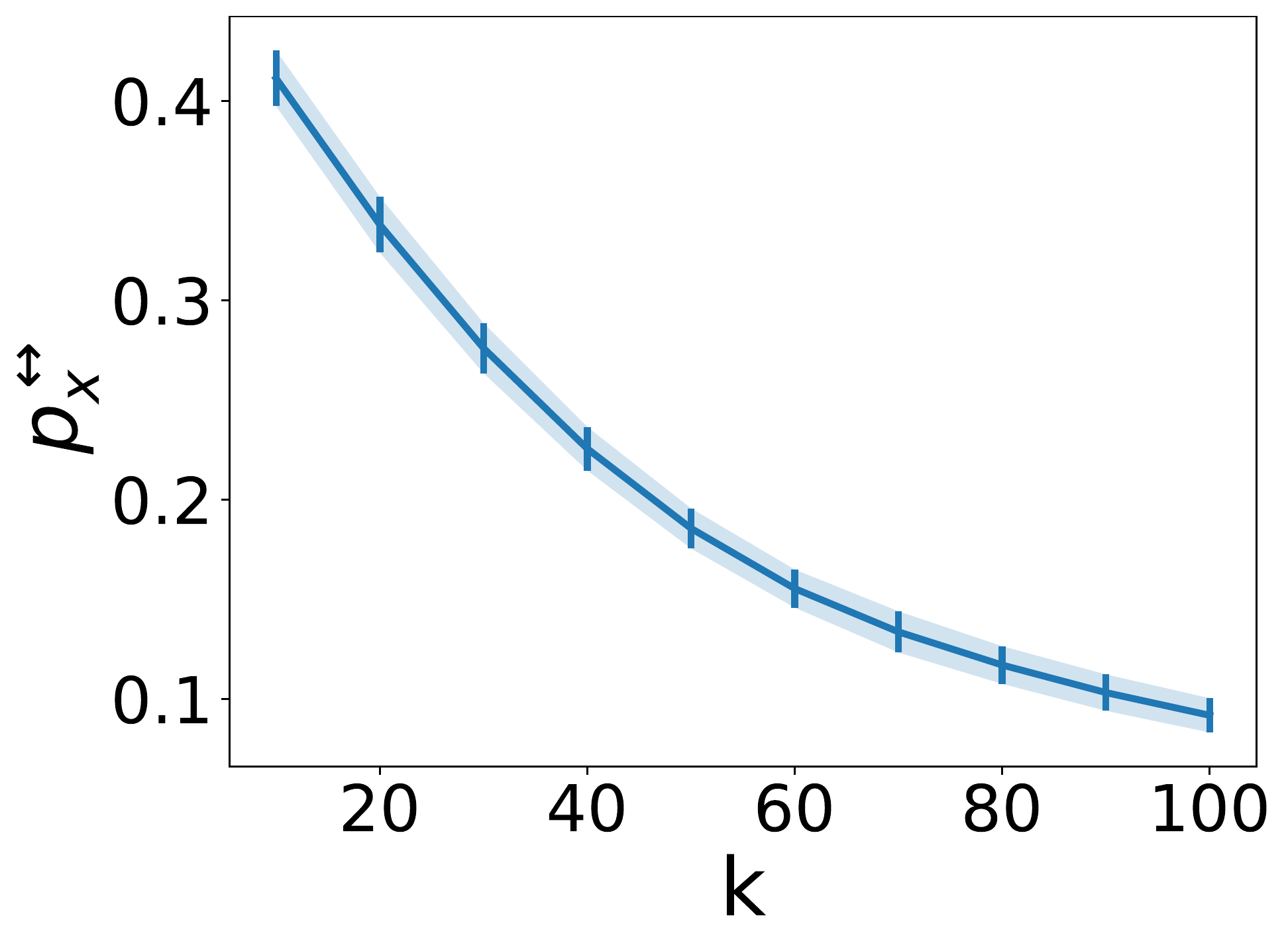}
\caption{Bank dataset}
\end{subfigure}
\begin{subfigure}[b]{0.3\textwidth}
\includegraphics[width = 1.0\linewidth]{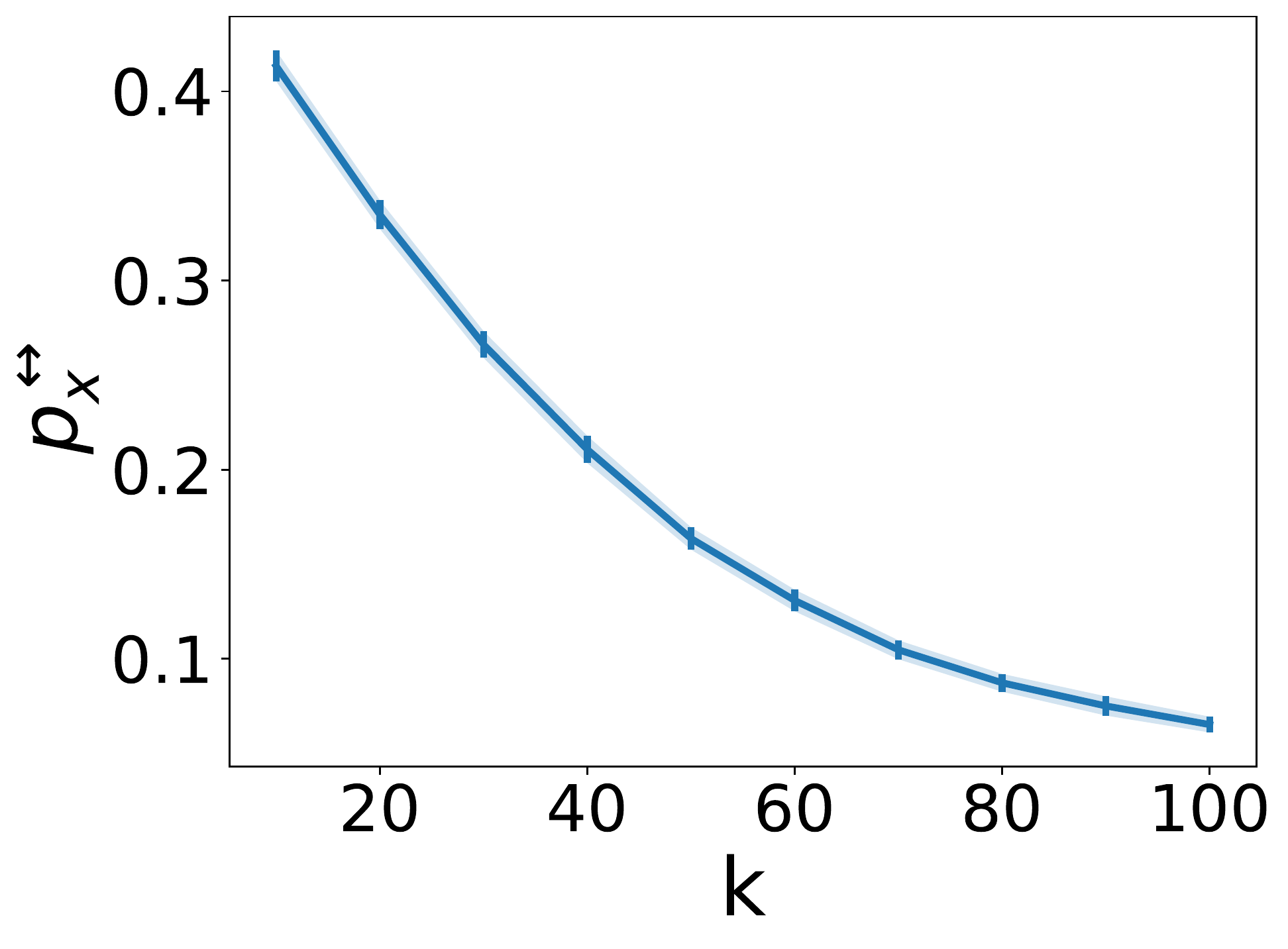}
\caption{Income dataset}
\end{subfigure}
\begin{subfigure}[b]{0.3\textwidth}
\includegraphics[width = 1.0\linewidth]{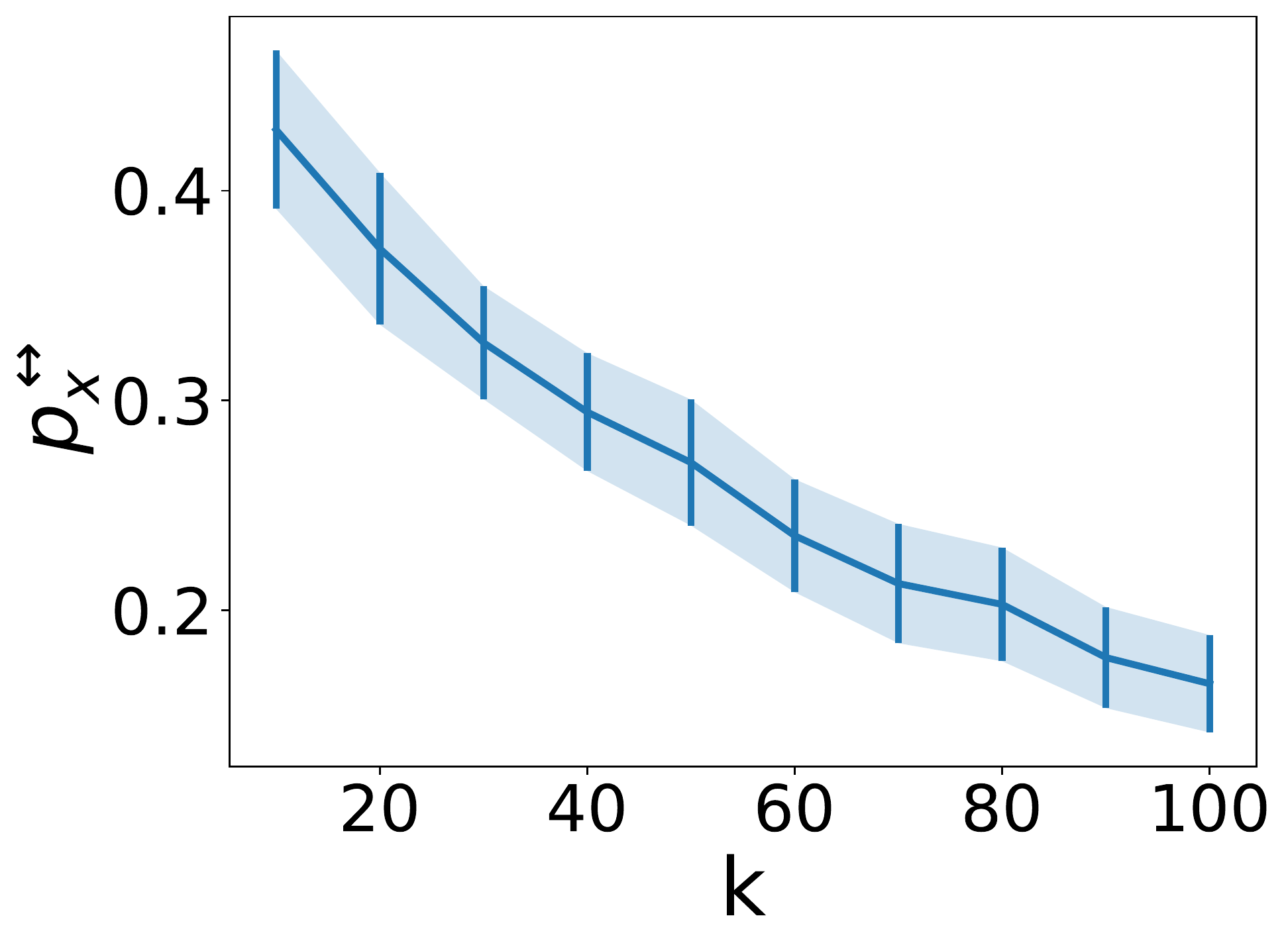}
\caption{Parkinsons dataset}
\end{subfigure}
\caption{ Average flipping probability $\fpx$  for samples $\bm{x} \in \bar{D}$ as a function of the ensemble size $k$. }
\label{fig:flip_prob_all}
\end{figure}

Next, concerning the fairness analysis, we provide additional empirical support on the effects of $k$ to the model sensitivity, the difference between the group excessive risk, and the utility of the PATE models. 
These metrics are summarized in Figure \ref{fig:k_effect}. A similar trend with what observed for the regularization parameter $\lambda$  can be observed here, when varying the ensemble size $k$. 
Additionally, when the values $k$ grow large they produce models with small model sensitivity as well as low accuracy, but the unfairness, measured by the excessive risk difference between two groups, reduces. 
This observation can be explained by looking at Figure \ref{fig:flip_prob_all} and by Theorem \ref{thm:1}: Large $k$ values imply smaller flipping probability, which, in turn, reduce the model sensitivity. 
Notice also that Theorem \ref{thm:3} shows that small model sensitivities can reduce the level of unfairness.

\begin{figure}
\centering
\begin{subfigure}[b]{0.485\textwidth}
\includegraphics[width = 1.0\linewidth]{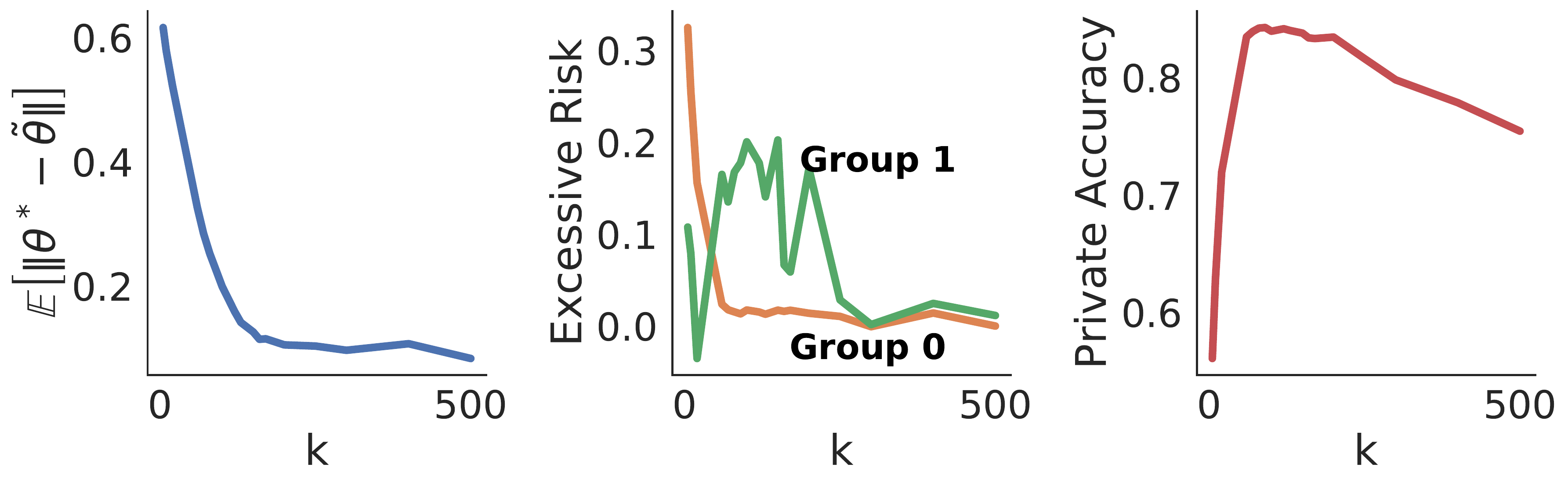}
\caption{Bank dataset}
\end{subfigure}
\begin{subfigure}[b]{0.485\textwidth}
\includegraphics[width = 1.0\linewidth]{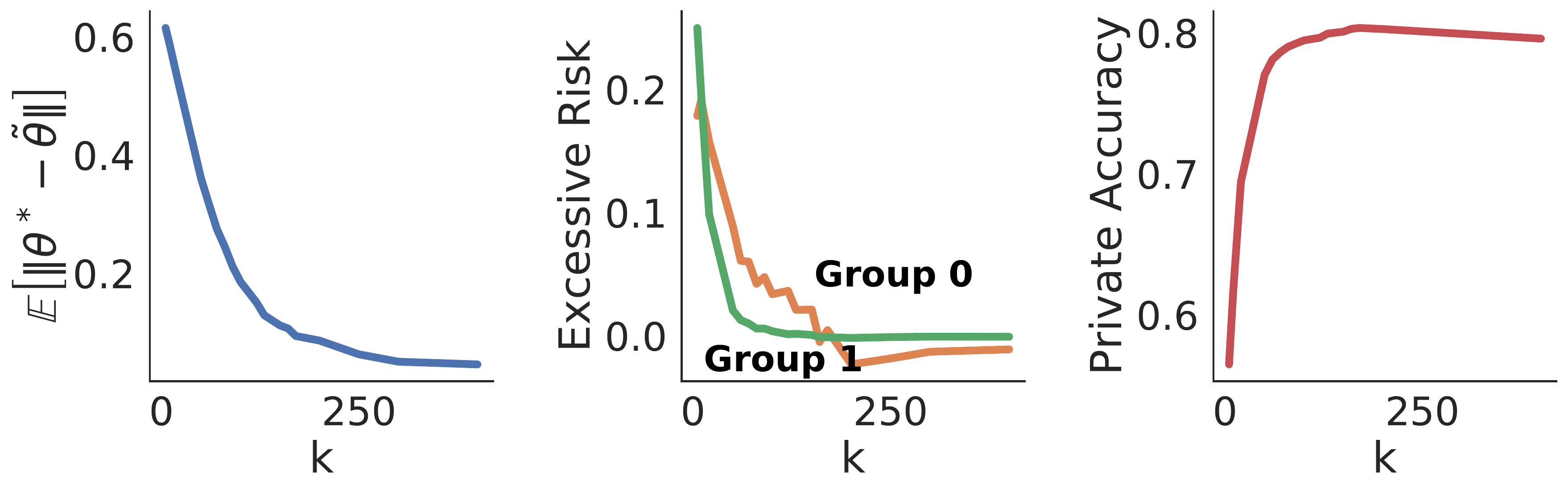}
\caption{Income dataset}
\end{subfigure}
\begin{subfigure}[b]{0.485\textwidth}
\includegraphics[width = 1.0\linewidth]{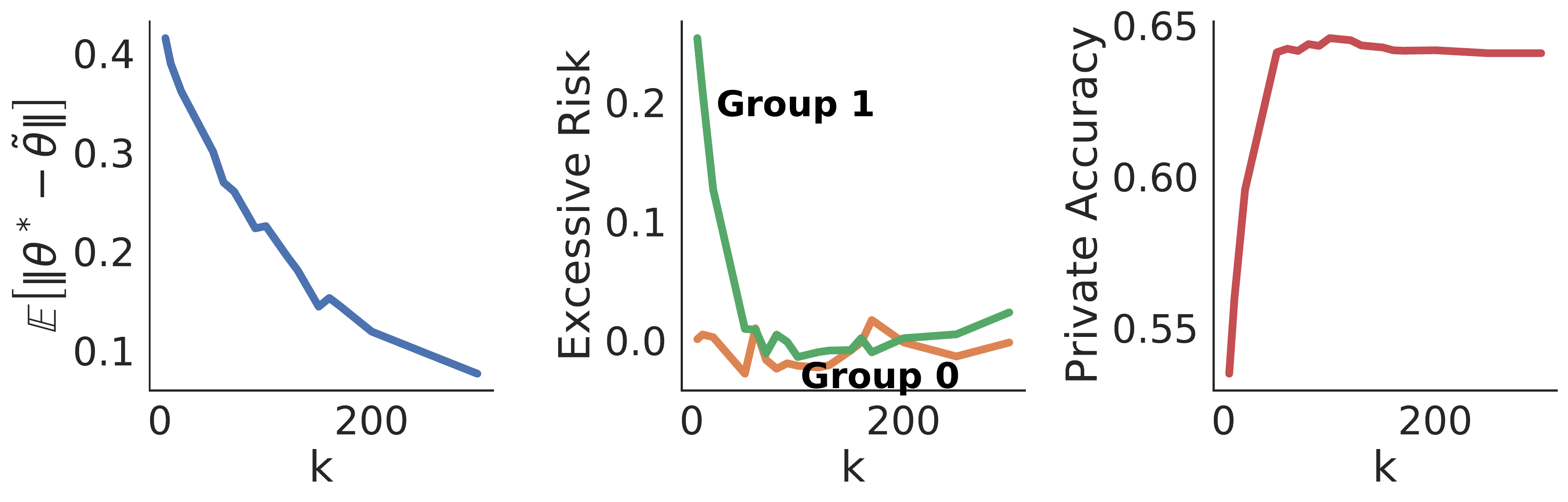}
\caption{Parkinsons dataset}
\end{subfigure}
\caption{Expected model sensitivity (left), empirical risk (middle), 
and model accuracy (right) as a function of the ensemble size. Here  $\lambda = 100$, $\sigma = 50$.}.
\label{fig:k_effect}
\end{figure}

\subsection{The impact of the data input norm}
This sections provides further experimental results regarding the relation between the input norms with (1) the private model sensitivity and (2) the model excessive risk.

Regarding the first relation, Corollary  \ref{cor:1} shows that the larger the input norm $\|\bm{x}\|$ the larger the model sensitivity. To illustrate this claim, for each dataset, the experiments vary the range of the input norm $\| \bm{x} \| $ and report the associated values of the expected model sensitivity. \ref{fig:inputnorm_vs_expected_diff} clearly illustrates a strong, non-decreasing, relation between input norms and the model sensitivity.

On the other hand, large input norms can affect the excessive risk because they directly control the gradient norms, by the analysis performed in Subsection \ref{app:sec:input_norms}. 
By Theorem \ref{thm:3}, the individuals generating large gradient norms can suffer from large excessive risk. Similarly, the individuals associated with large input data norms---which are often observed at the tail of data distribution---are more impacted in terms of accuracy drop, when compared to individuals with smaller input norms. 
These claims are illustrated in Figure \ref{fig:corr_norm_all}, which shows the Spearman correlation between input norms and the associated individual excessive risk of the model. On all datasets, observe the positive relation between the data input norm and the excessive risk.

\begin{figure}
\centering
\begin{subfigure}[b]{0.3\textwidth}
\includegraphics[width = 1.0\linewidth]{impact_input_norm_income.pdf}
\caption{Income dataset}
\end{subfigure} 
\begin{subfigure}[b]{0.3\textwidth}
\includegraphics[width = 1.0\linewidth]{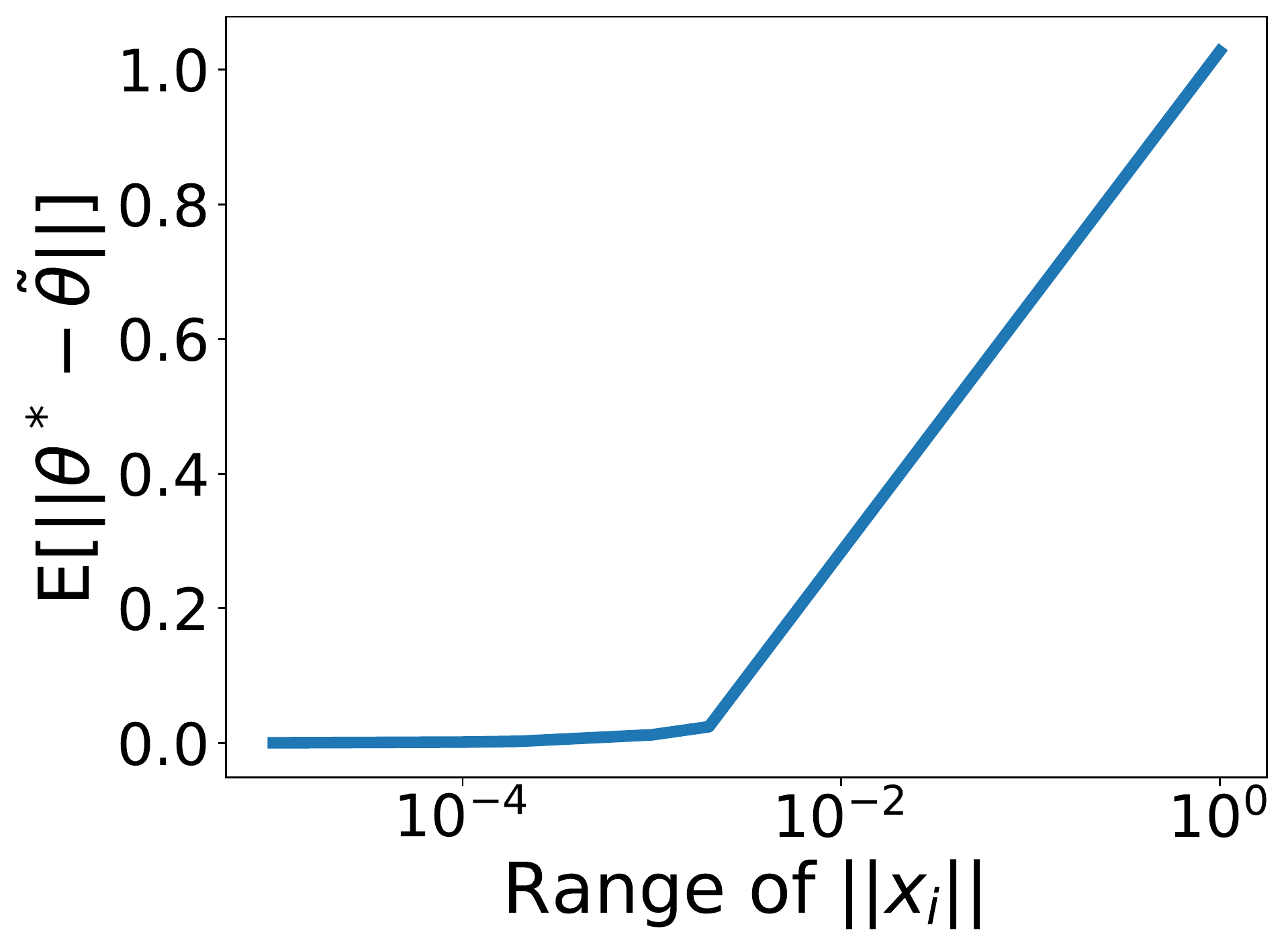}
\caption{Bank dataset}
\end{subfigure}
\begin{subfigure}[b]{0.3\textwidth}
\includegraphics[width = 1.0\linewidth]{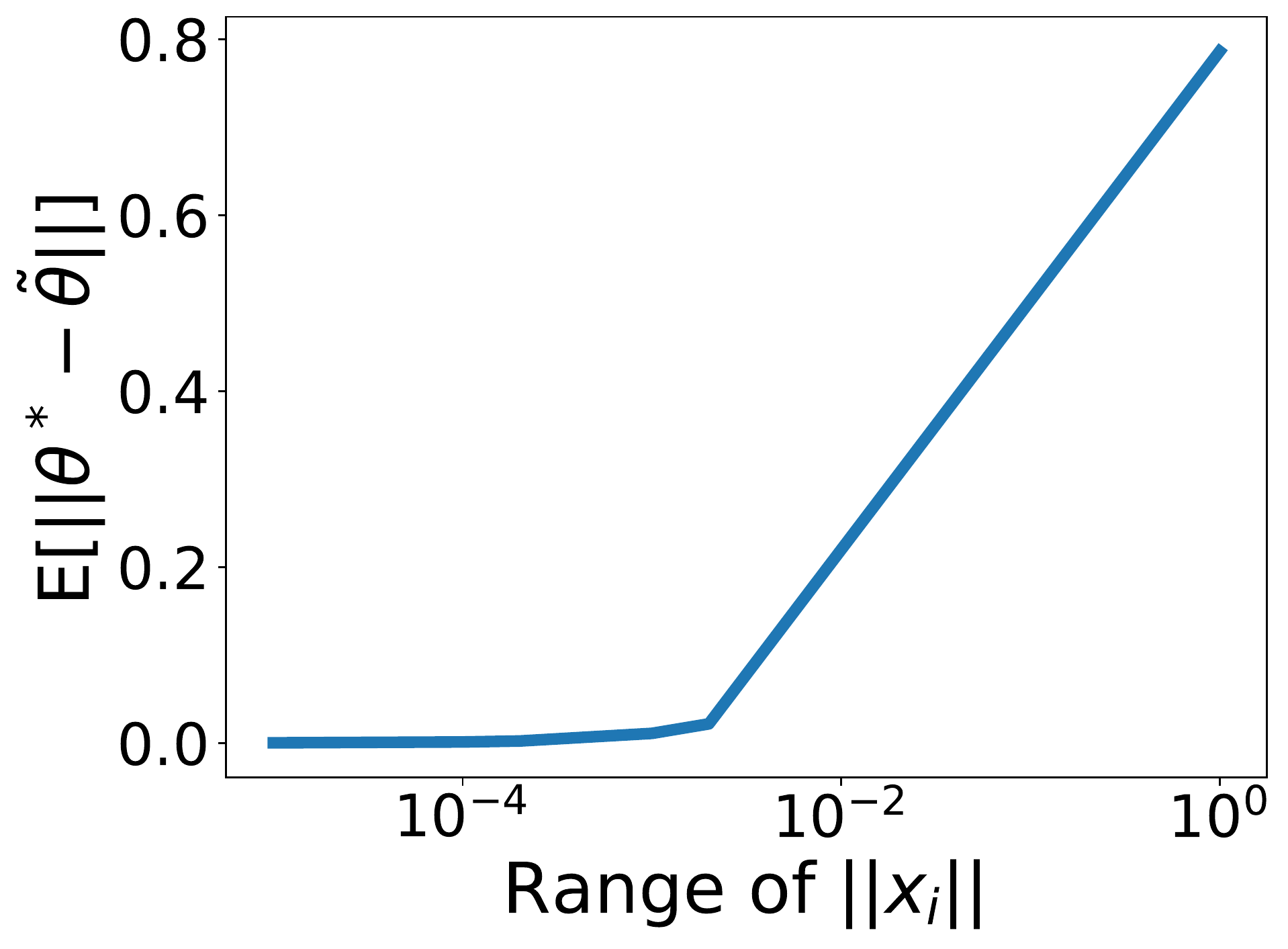}
\caption{Parkinsons dataset}
\end{subfigure}
\caption{Relation between input norm and model sensitivity.}
\label{fig:inputnorm_vs_expected_diff}
\end{figure}

\begin{figure}
\centering
\begin{subfigure}[b]{0.49\textwidth}
\includegraphics[width = 1.0\linewidth]{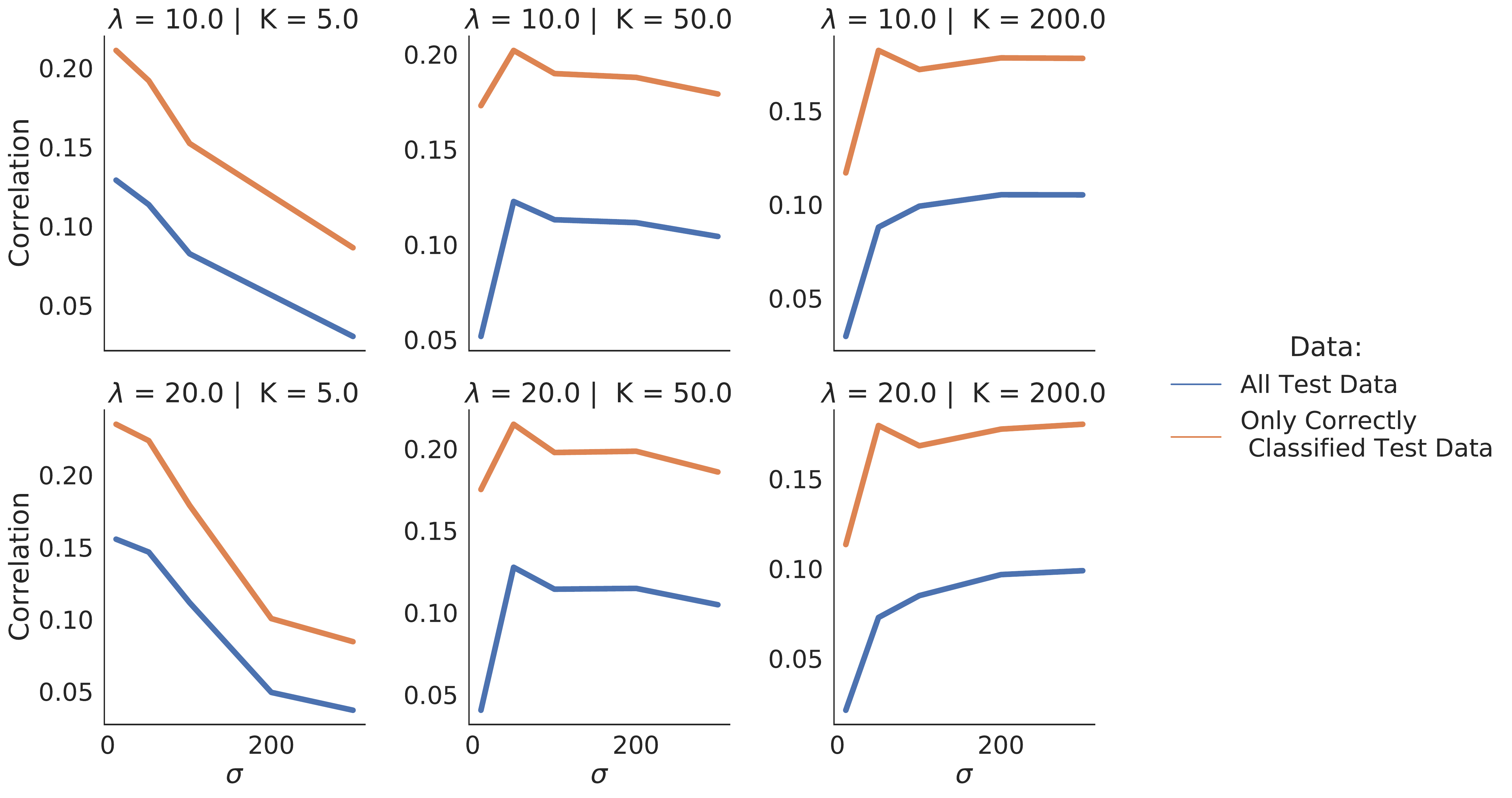}
\caption{Bank dataset}
\end{subfigure} 
\begin{subfigure}[b]{0.49\textwidth}
\includegraphics[width = 1.0\linewidth]{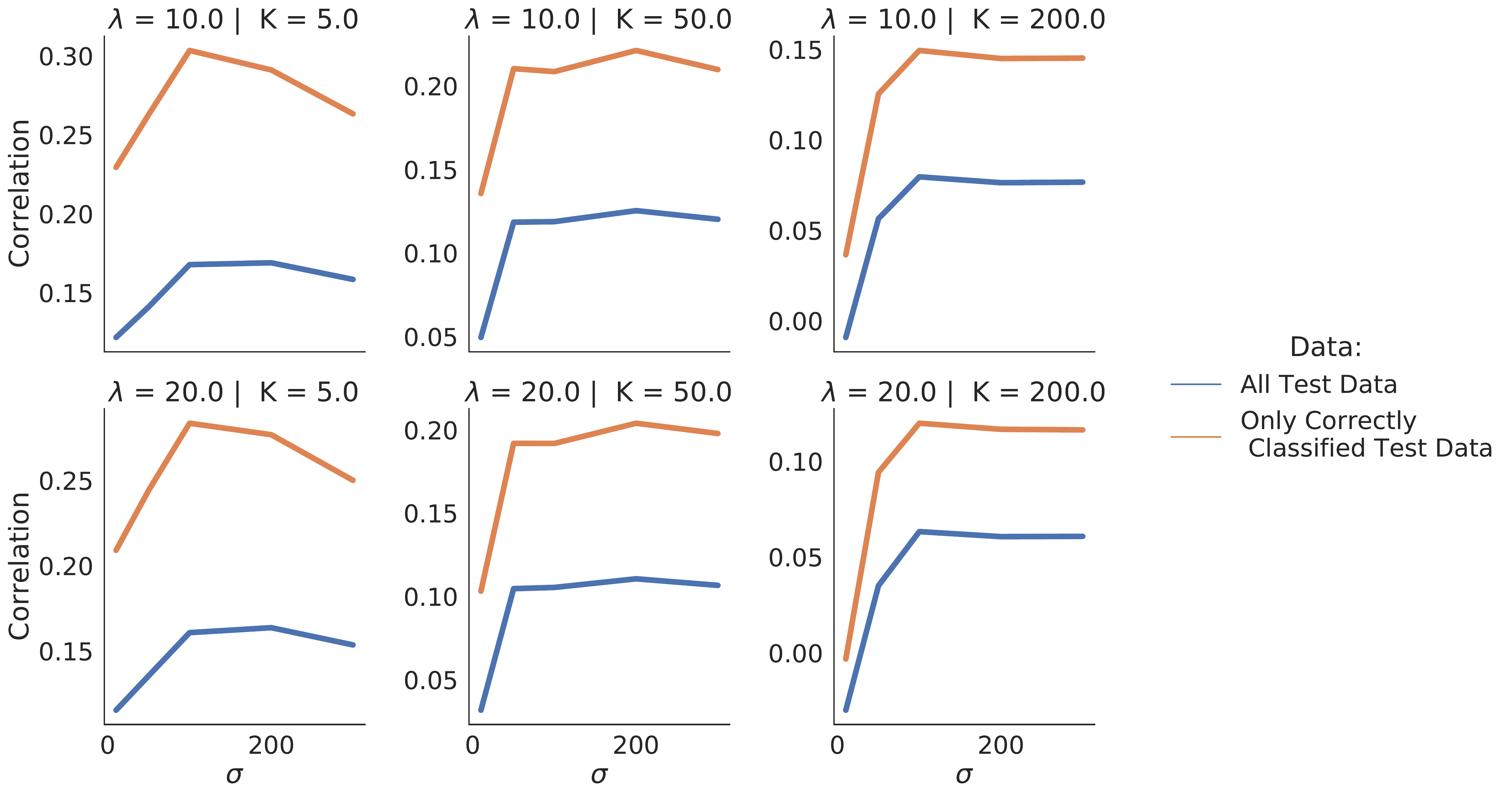}
\caption{Income dataset}
\end{subfigure}
\begin{subfigure}[b]{0.49\textwidth}
\includegraphics[width = 1.0\linewidth]{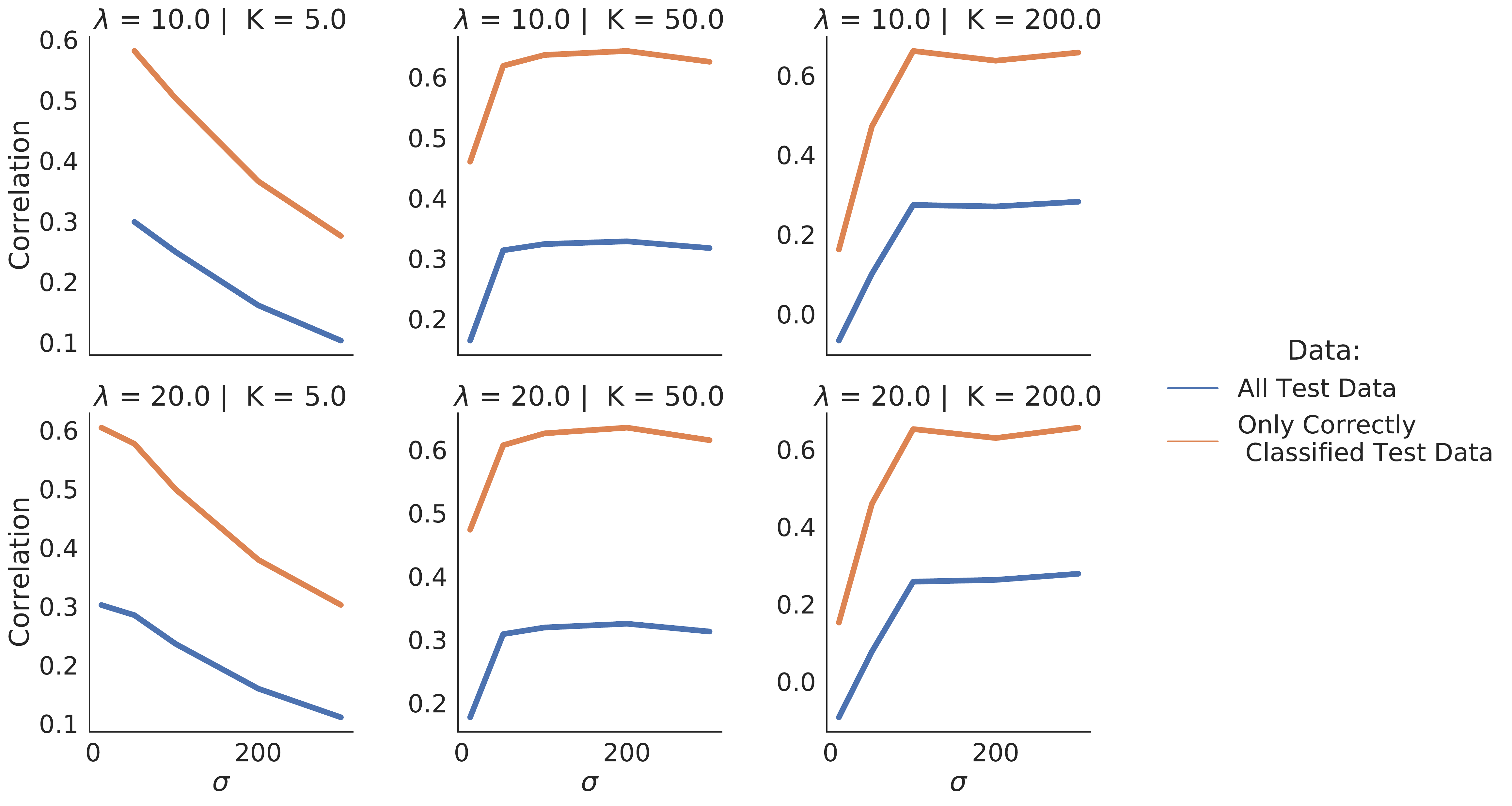}
\caption{Credit card dataset}
\end{subfigure}
\begin{subfigure}[b]{0.49\textwidth}
\includegraphics[width = 1.0\linewidth]{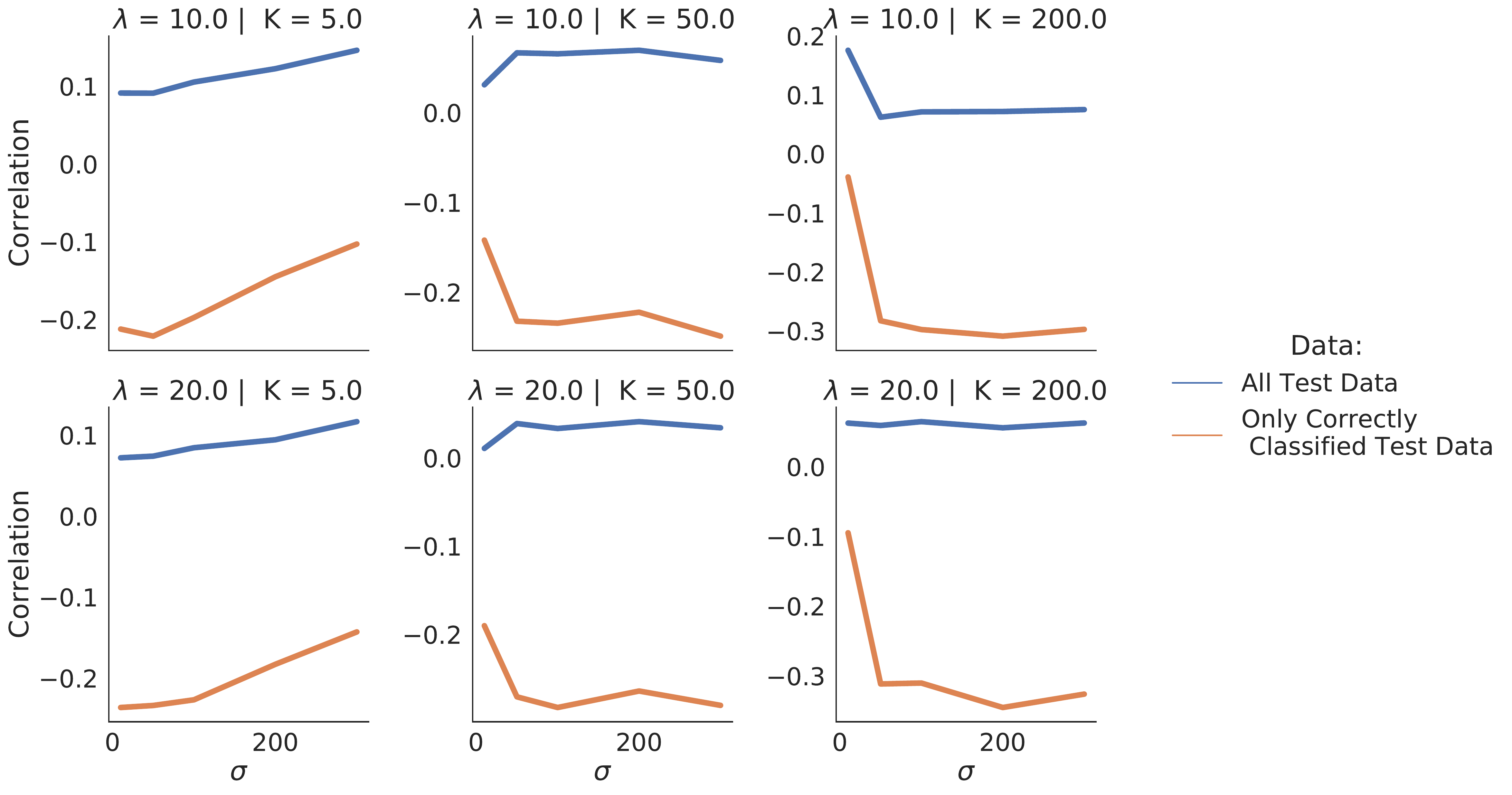}
\caption{Parkinsons dataset}
\end{subfigure}
\caption{Correlation between the excessive risk and input norm on 4 datasets. Here for each dataset, number of teacher $\lambda = 100$, $\sigma = 50, k = 150$.}
\label{fig:corr_norm_all}
\end{figure}

\subsection{Connection between input norm and gradient norm}
\label{app:sec:input_norms}
\begin{figure*}
    \centering
    \includegraphics[width=0.75\linewidth]{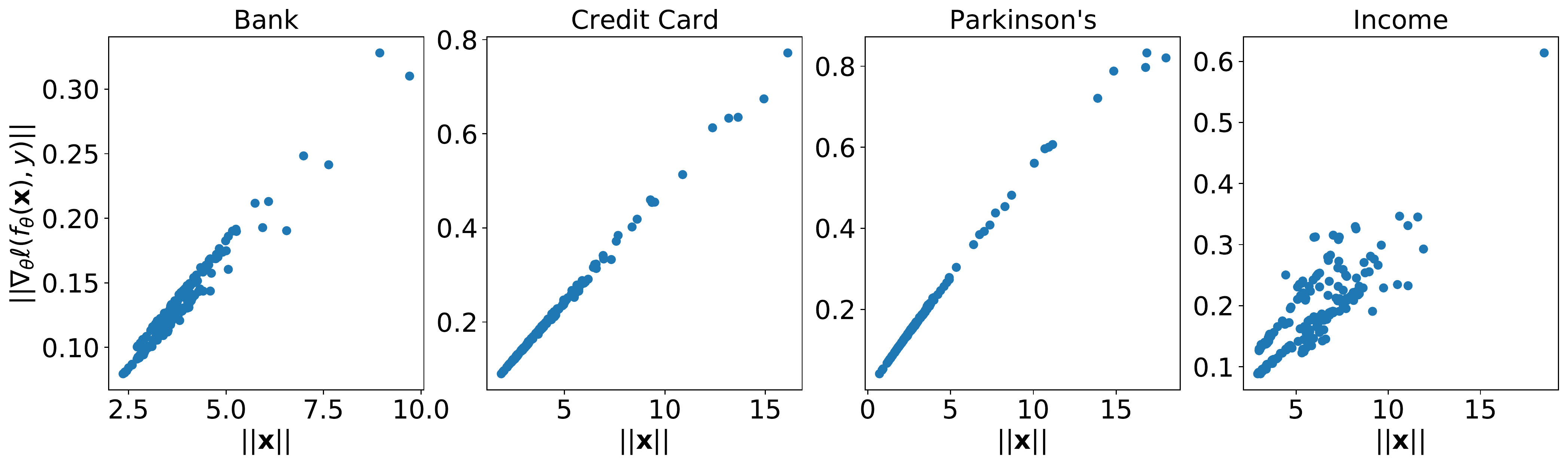}
    \caption{Relation Between Gradient Norm and Input Norm on all datasets.}
    \label{fig:grad_inp_corr_all}
\end{figure*}

Propositions \ref{ex:grad_logreg} and \ref{ex:hessian_logreg} showed the presence of a strong relation between the individual input norms $\| \bm{x}\|$ and their associated gradient norms at the optimal model parameter $\optimal{\btheta}$, $  \| \nabla_{\optimal{\btheta}}  \ell(\bar{f}_{\optimal{\btheta}}(\bm{x}),y)\|$, 
for logistic regression classifiers. 
This subsection extends the analysis to non-linear models. 

In particular, it will show a similar connection between the gradient norms and the input norms for a neural network with a single hidden layer. We start by considering the following settings:
\paragraph{Settings} 

Consider a neural network model $\bar{f}_{\optimal{\btheta}}
 (\bm{x}) \defeq \textsl{softmax} \left(\optimal{\btheta}_1^T \tau(\optimal{\btheta}^T_2 \bm{x}) \right)$ 
 where $\bm{x} = (\bm{x}^i)_{i=1}^d $ is a $d$ dimensional input vector,  $\tau(\cdot)$ is an activation function, 
 the parameters 
 $\optimal{\btheta}_{\!1} \in \RR^{H \times C}, 
 \optimal{\btheta}_{\!2} \in \RR^{d \times H}$, 
 and the cross entropy loss $\ell(\bar{f}_{\optimal{\btheta}}(\bm{x}),y) = -\sum_{c=1}^C y_c \log 
 \bar{f}_{\optimal{\btheta}, c}(\bm{x})$. \\
 Let $\bm{O} = \tau(\optimal{\btheta}^T_2 \bm{x}) \in \RR^H$
 be the vector $(O_1, \ldots, O_H)$ of $H$ hidden nodes of the network. 
 Denote the variables $h_j = \sum_{i=1}^d \optimal{\btheta}_{2,j,i} \bm{x}^i$ as the $j$-th hidden unit 
 before the activation function. 
 Next, denote $\optimal{\btheta}_{1,j,k} \in \RR$ as the weight parameter that connects 
 the $j$-th hidden unit $h_j$ with the $c$-th output unit $\bar{f}_c$ and  
 $ \  \optimal{\btheta}_{2,i,j} \in \RR$ as the weight parameter that connects the $i$-th 
 input unit $\bm{x}^i$ with the $j$-th hidden unit $h_j$.

Given the settings above, we now show the dependency between gradient norms and input norms. First notice that we can decompose the gradients norm of this neural network into two layers as follows:
\begin{equation}
    \| \nabla_{\optimal{\btheta}}\ell(\bar{f}_{\optimal{\btheta}}(\bm{x}),y) \|^2 =  \| \nabla_{\optimal{\btheta}_1}\ell(\bar{f}_{\optimal{\btheta}}(\bm{x}),y) \|^2 + \| \nabla_{\optimal{\btheta}_2}\ell(\bar{f}_{\optimal{\btheta}}(\bm{x}),y) \|^2.
\end{equation}
We will show that $\nabla_{\optimal{\btheta}_2}\ell(\bar{f}_{\optimal{\btheta}}(\bm{x}),y) \| \propto  \|\bm{x}\|.$

Notice that: 
$$
\|\nabla_{\optimal{\btheta}_2}\ell(\bar{f}_{\optimal{\btheta}}(\bm{x}),y) \|^2 = \sum_{i,j} \| \nabla_{\optimal{\btheta}_{2,i,j}}\ell(\bar{f}_{\optimal{\btheta}}(\bm{x}),y) \|^2.
$$

Applying, Equation (14) from \citet{gradient_formula}, it follows that: 
\begin{equation}
   \nabla_{\optimal{\btheta}_{2,i,j}}\ell(\bar{f}_{\optimal{\btheta}}(\bm{x}),y) = 
   \sum_{c=1}^C \left(y_c - \bar{f}_{\optimal{\btheta},c}(\bm{x}) \right) \, 
   \optimal{\btheta}_{1,j,c}\left(O_j(1 - O_j) \right) \bm{x}^i,
\end{equation}
which highlights the dependency of the gradient norm 
$\| \nabla_{\optimal{\btheta}_2}\ell(\bar{f}_{\btheta}(\bm{x}),y) \|$ and the input norm $\| \bm{x}\|$. Figure \ref{fig:grad_inp_corr_all} provides empirical evidence supporting this dependency. It shows a strong positive correlation between input norms and the gradient norms at individual levels on all datasets analyzed.

\subsection{Upper bound of the expected model sensitivity} 

The following provides empirical results for Corollary \ref{cor:1} on four benchmark datasets. As indicated in this corollary, the expected model sensitivity is bounded by $ \frac{1}{m\lambda} \left[ \sum_{\bm{x} \in \bar{D}} \fp{\bm{x}} \| \bm{x} \| \right]$. 
Figure \ref{fig:cor1_bound} illustrates the tightness of this bound by plotting the RHS and the LHS values of Equation \eqref{eq:8} on different datasets. The plots in Figure 
\ref{fig:cor1_bound} use $20$ teachers and regularization parameter $\lambda = 20$ (top)  and $\lambda = 100$ (bottom). 

\begin{figure}[t]
\centering
\begin{subfigure}[b]{0.24\textwidth}
\includegraphics[width = 1.0\linewidth]{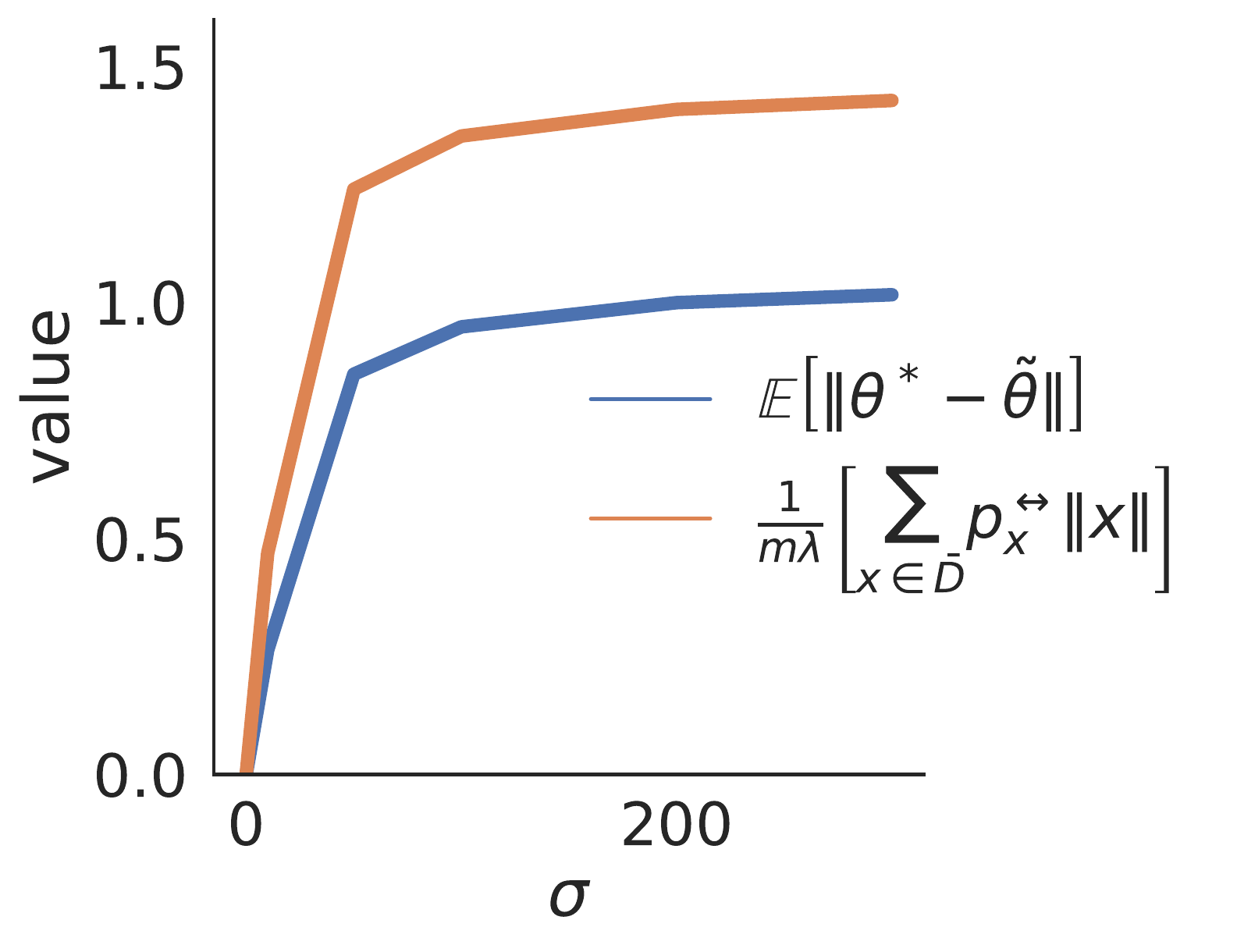}\\
\includegraphics[width = 1.0\linewidth]{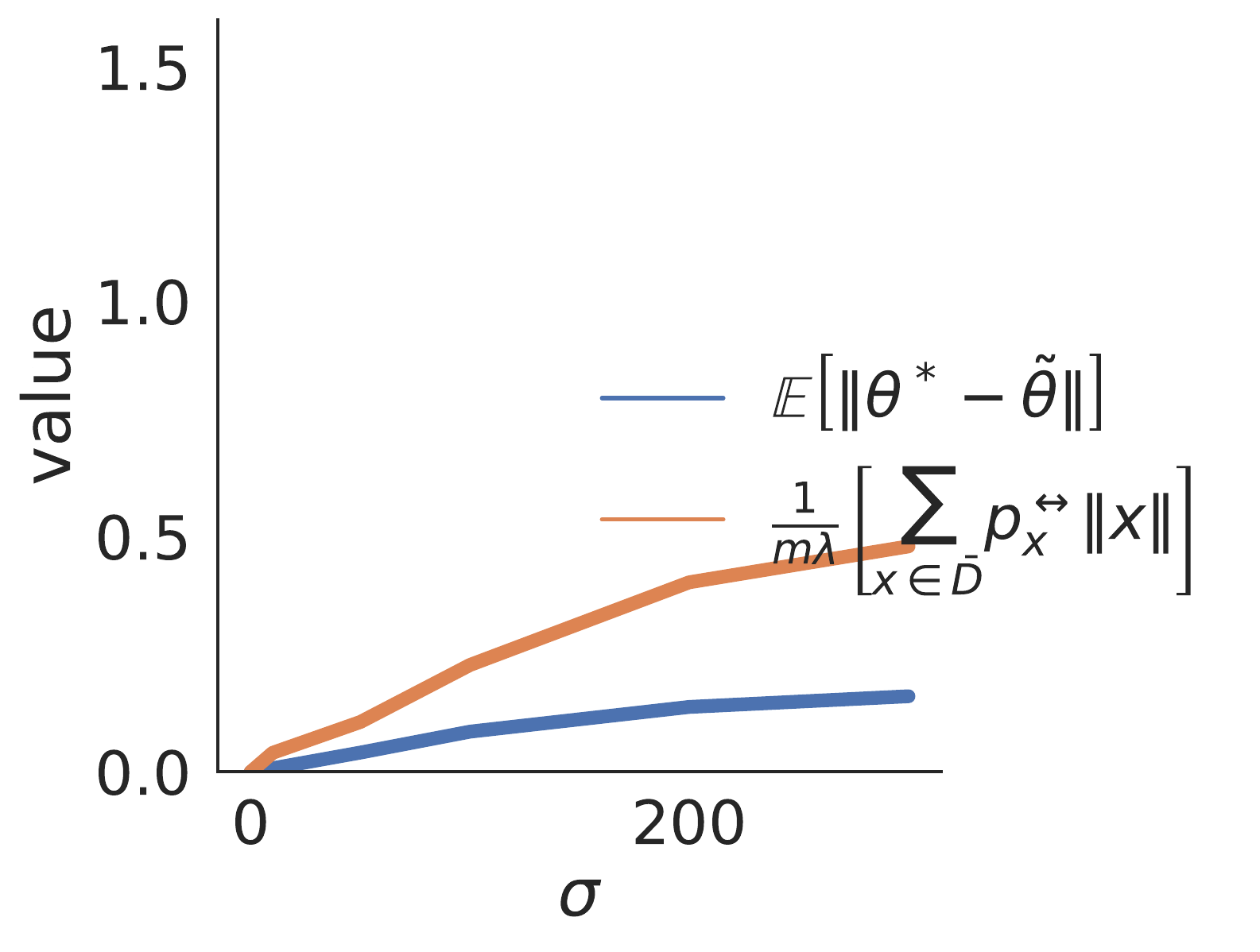}
\caption{Bank dataset}
\end{subfigure}
\begin{subfigure}[b]{0.24\textwidth}
\includegraphics[width = 1.0\linewidth]{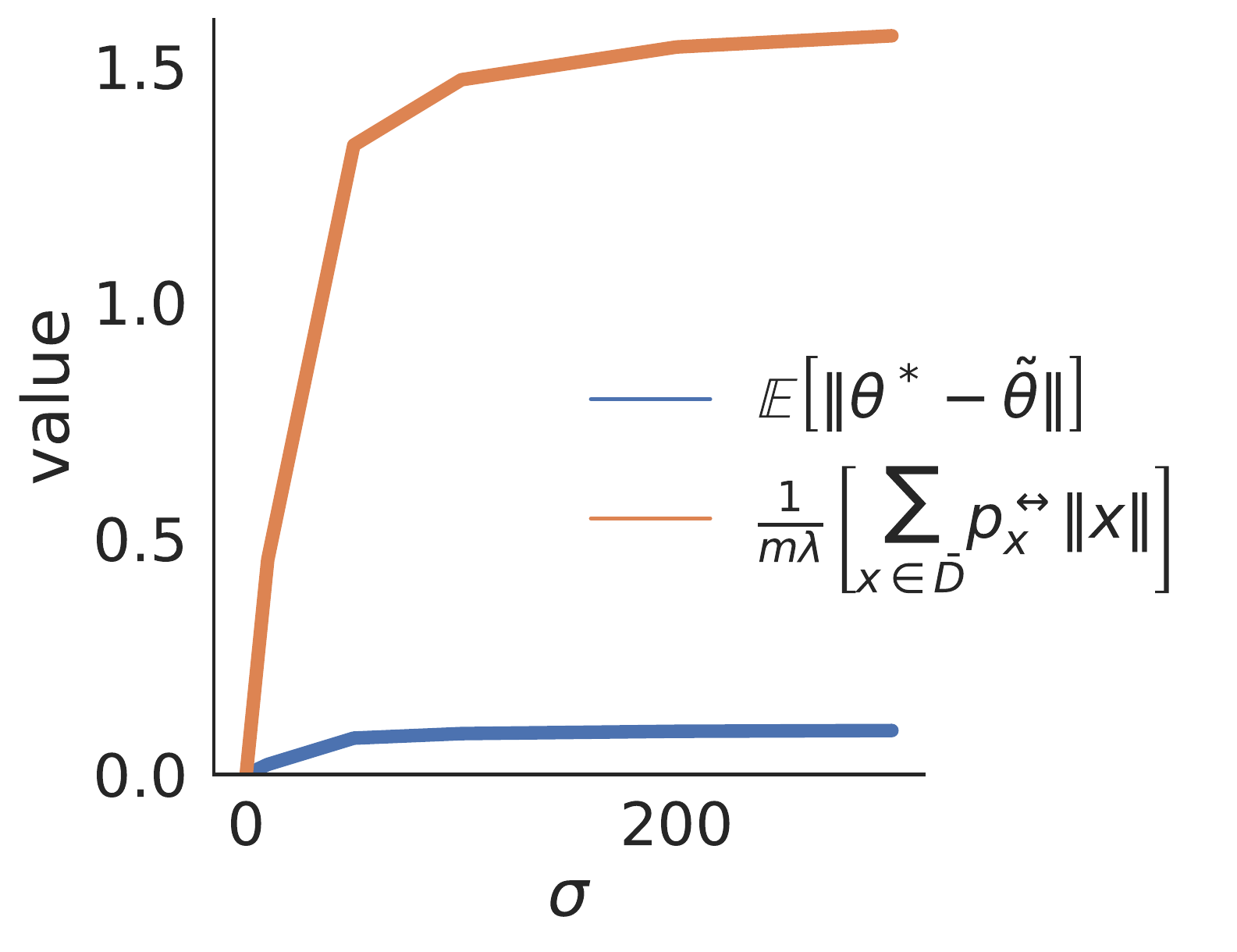}\\
\includegraphics[width = 1.0\linewidth]{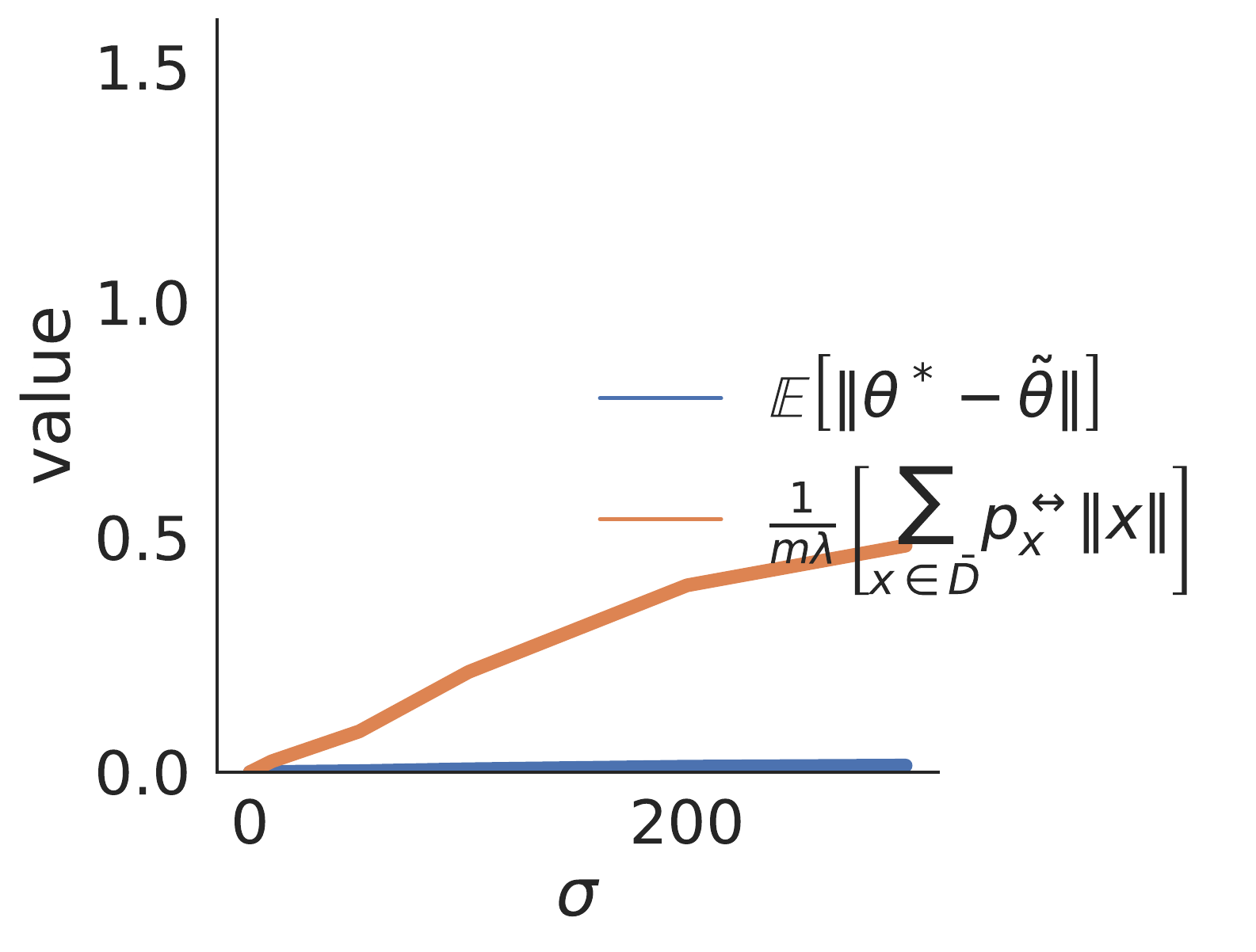}
\caption{Credit card dataset}
\end{subfigure}
\begin{subfigure}[b]{0.24\textwidth}
\includegraphics[width = 1.0\linewidth]{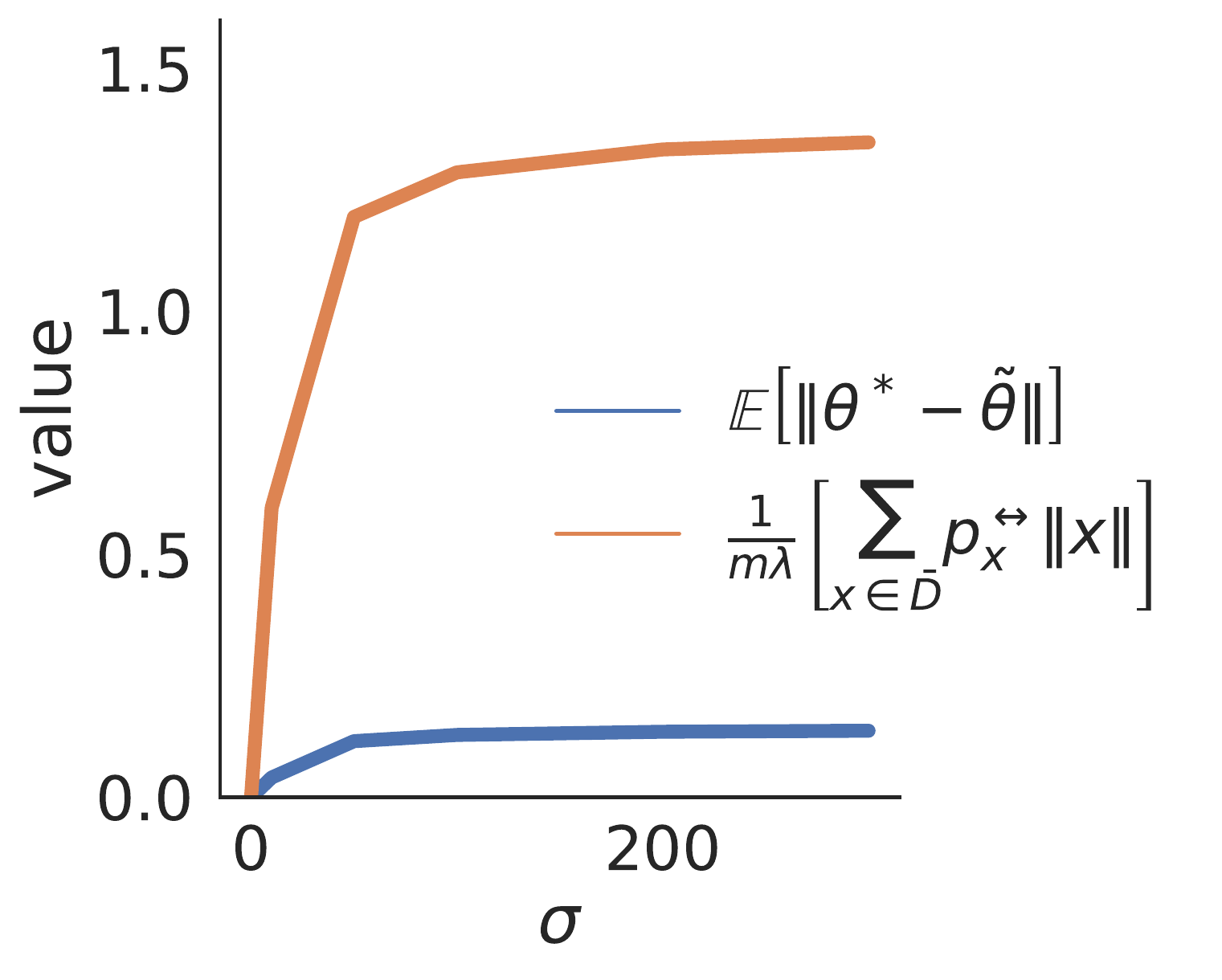}\\
\includegraphics[width = 1.0\linewidth]{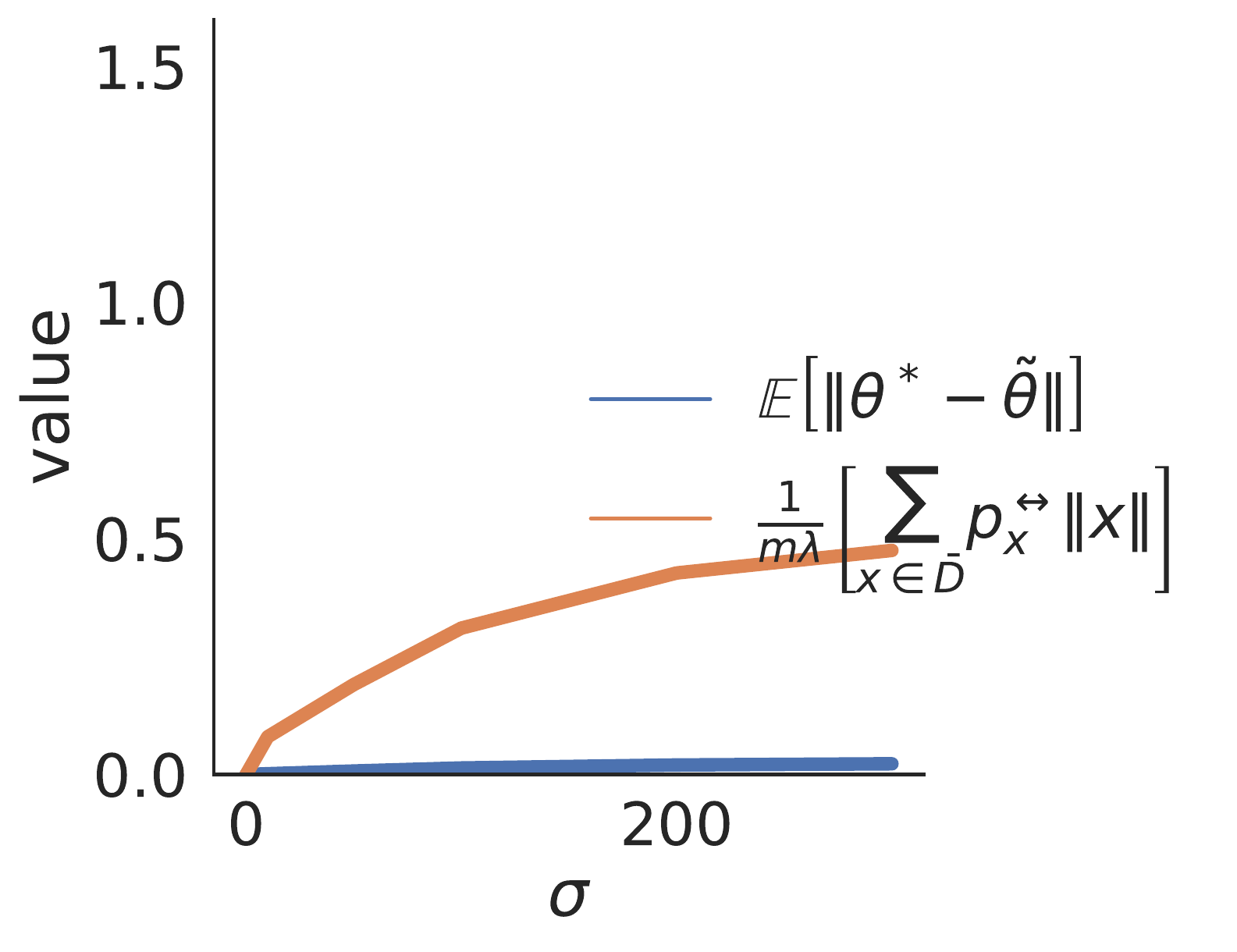}
\caption{Income dataset}
\end{subfigure}
\begin{subfigure}[b]{0.24\textwidth}
\includegraphics[width = 1.0\linewidth]{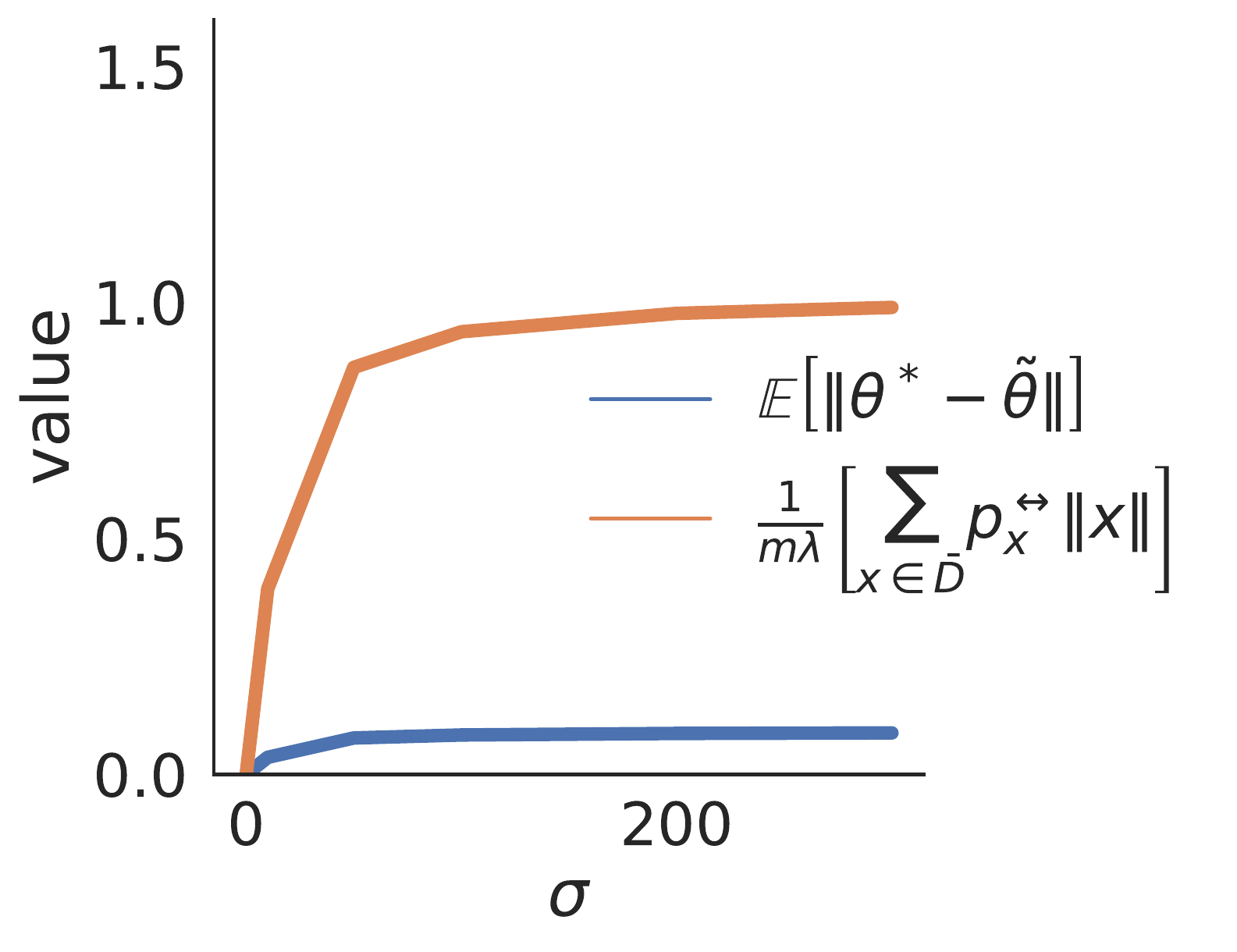}\\
\includegraphics[width = 1.0\linewidth]{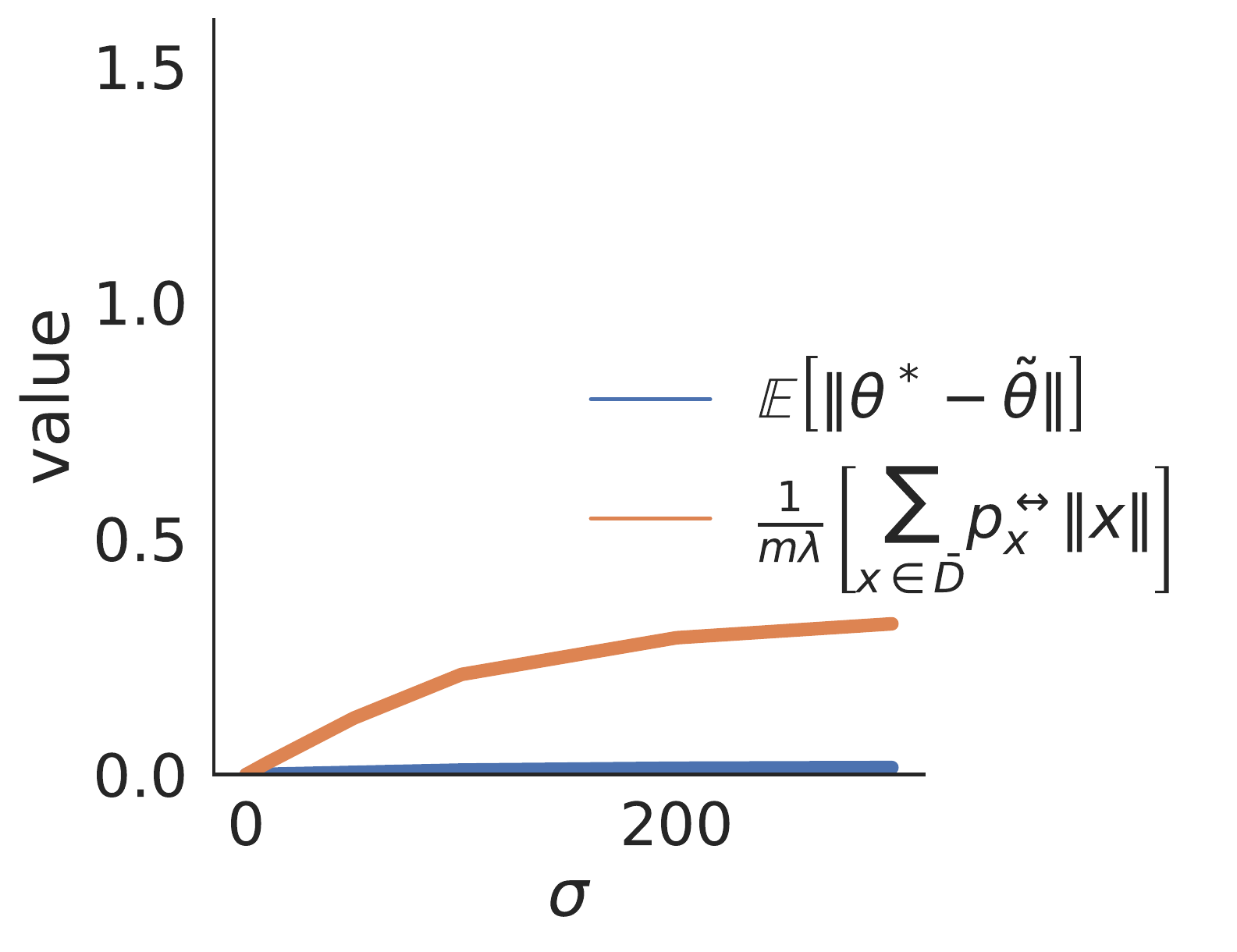}
\caption{Parkinsons dataset}
\end{subfigure}
\caption{Upper bound of the expected model sensitivity on 4 datasets with $k = 20$ and $\lambda = 20$ (top) and $\lambda = 100$ (bottom).}
\label{fig:cor1_bound}
\end{figure}

\subsection{Effectiveness of mitigation solution}

This subsection provides extended empirical results regarding the effectiveness of the proposed mitigation solution, presented in Section \ref{sec:mitigation}. 

It reports a comparison between training PATE with hard and soft labels when $k=20$  (Figure \ref{fig:mitigation_solution_K_20}) and when $k=150$ (Figure \ref{fig:mitigation_solution_K_150}). 
The analysis compares the models learned with hard and soft labels in terms of utility and fairness. 
In each figure, the top subplots show the group excessive risks $R(\bar{D}_{\leftarrow 0})$ and $R(\bar{D}_{\leftarrow 1})$ associated with minority (0) and majority (1) groups while the bottom subplot illustrate the accuracy of the model, at increasing of the privacy loss $\epsilon$. 
The figures clearly show how the models trained using soft labels achieve improved fairness (it reduces the excessive risk differences between the groups) without sacrificing accuracy.

Finally, recall that the mitigation solution does not require the availability of group labels during training. This challenging settings are of importance under the scenario when it is not feasible to collect or use protected features (e.g., under GDPR).

\begin{figure*}
\centering
\begin{subfigure}[b]{0.45\textwidth}
\includegraphics[width=\textwidth, height=120pt]{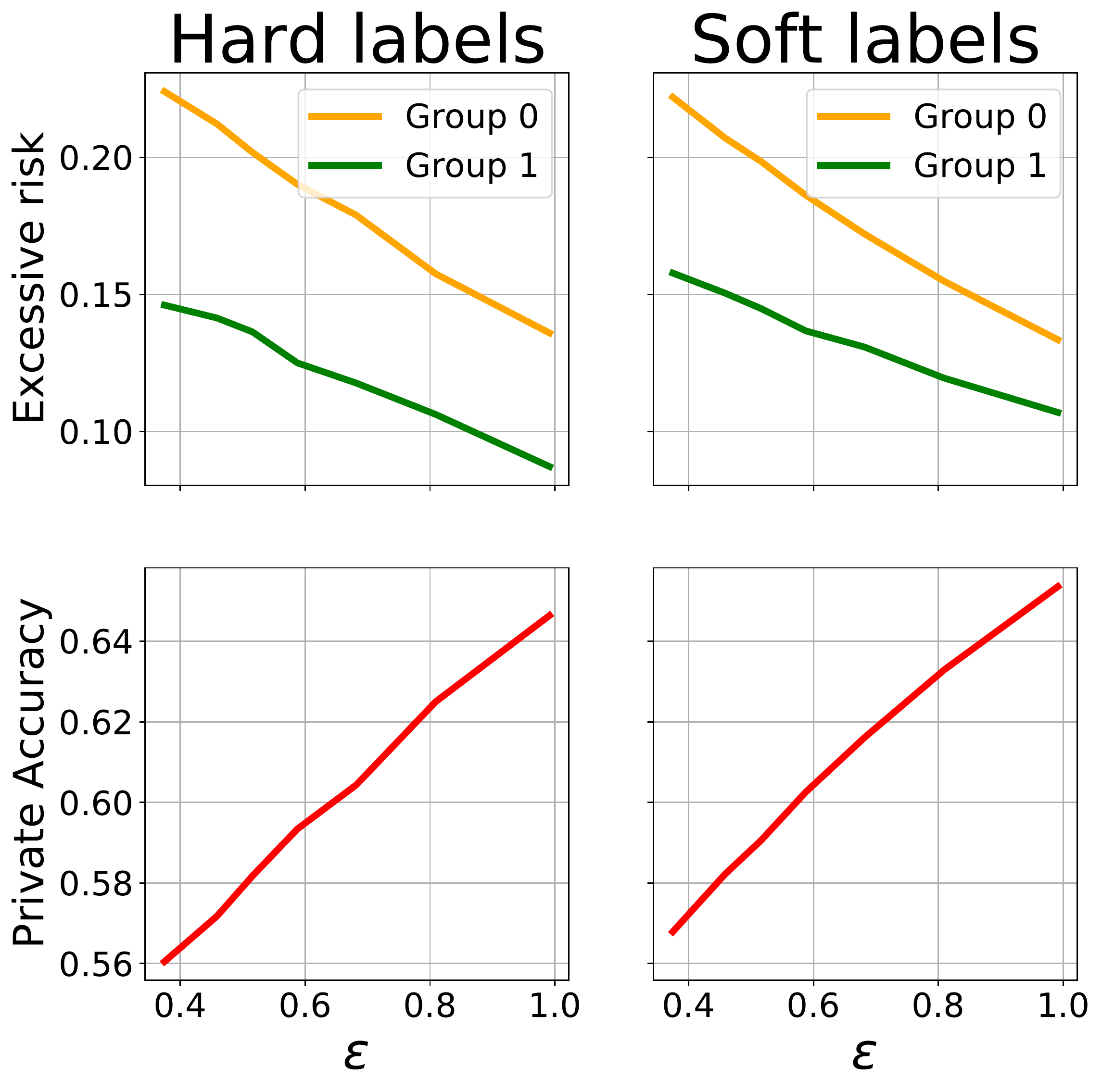}
\caption{}
\end{subfigure}
\begin{subfigure}[b]{0.45\textwidth}
\includegraphics[width=\linewidth, height=120pt]{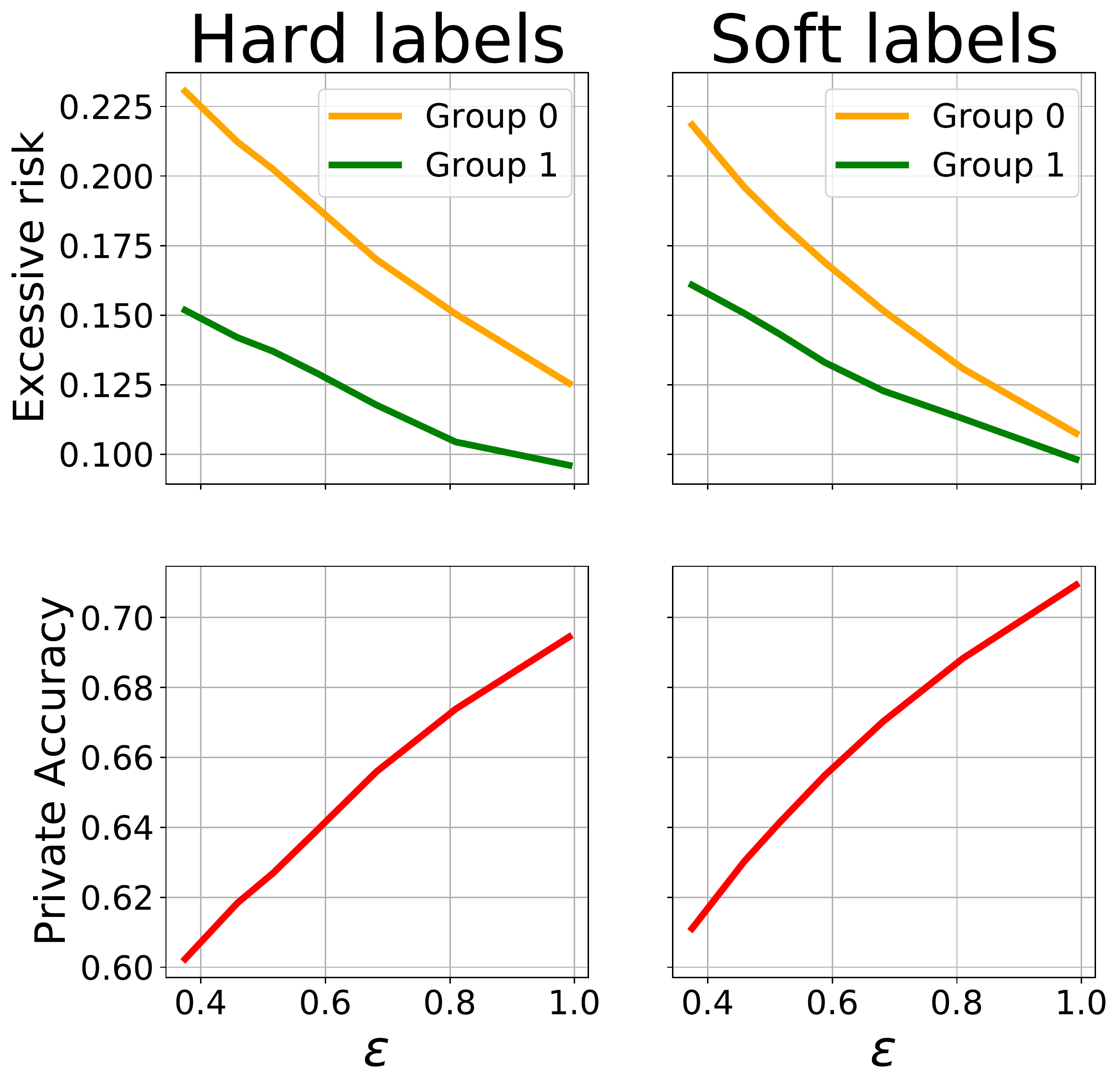}
\caption{}
\end{subfigure}
\begin{subfigure}[b]{0.45\textwidth}
\includegraphics[width=\linewidth, height=120pt]{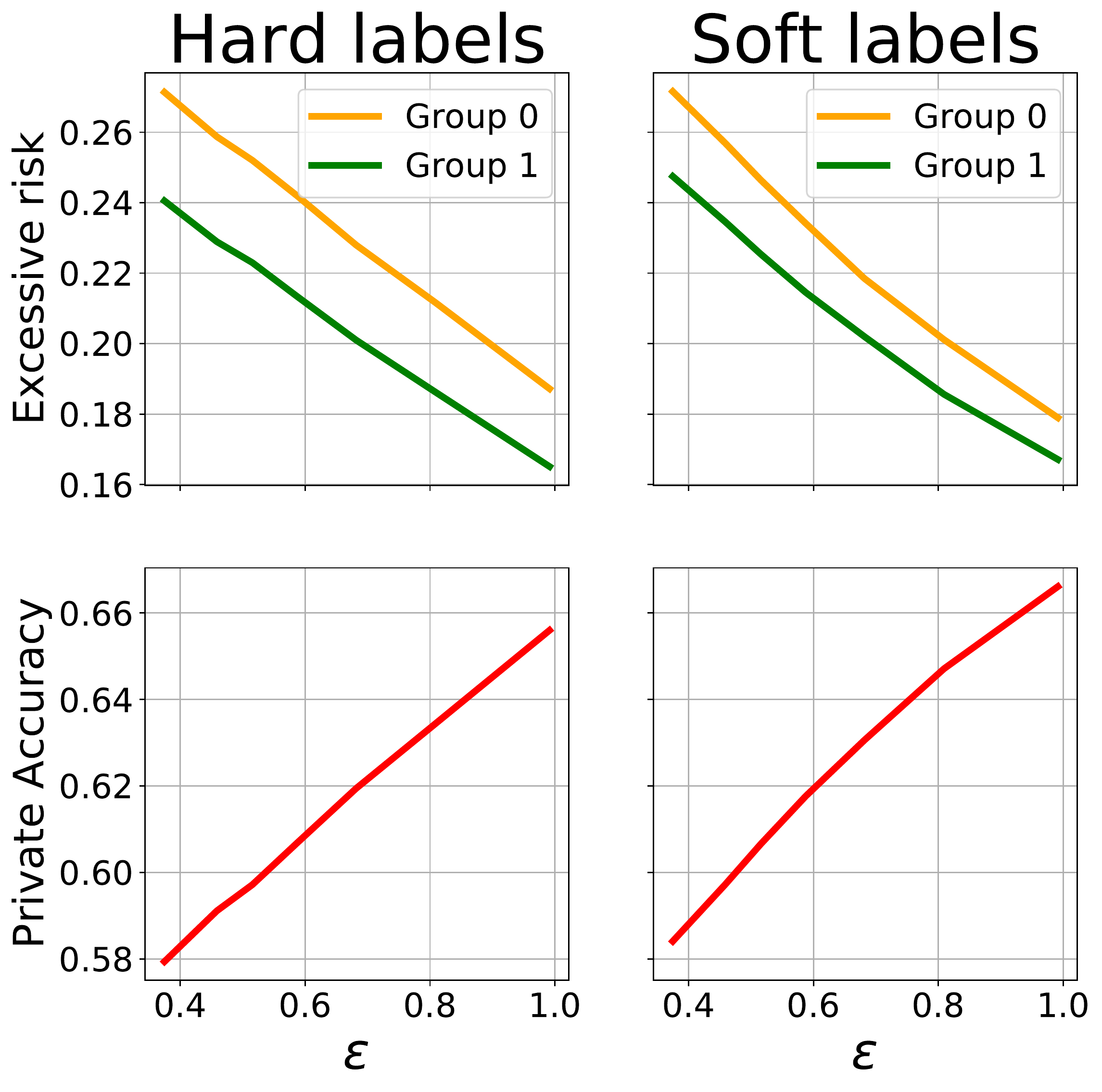}
\caption{}
\end{subfigure}
\begin{subfigure}[b]{0.45\textwidth}
\includegraphics[width=\linewidth, height=120pt]{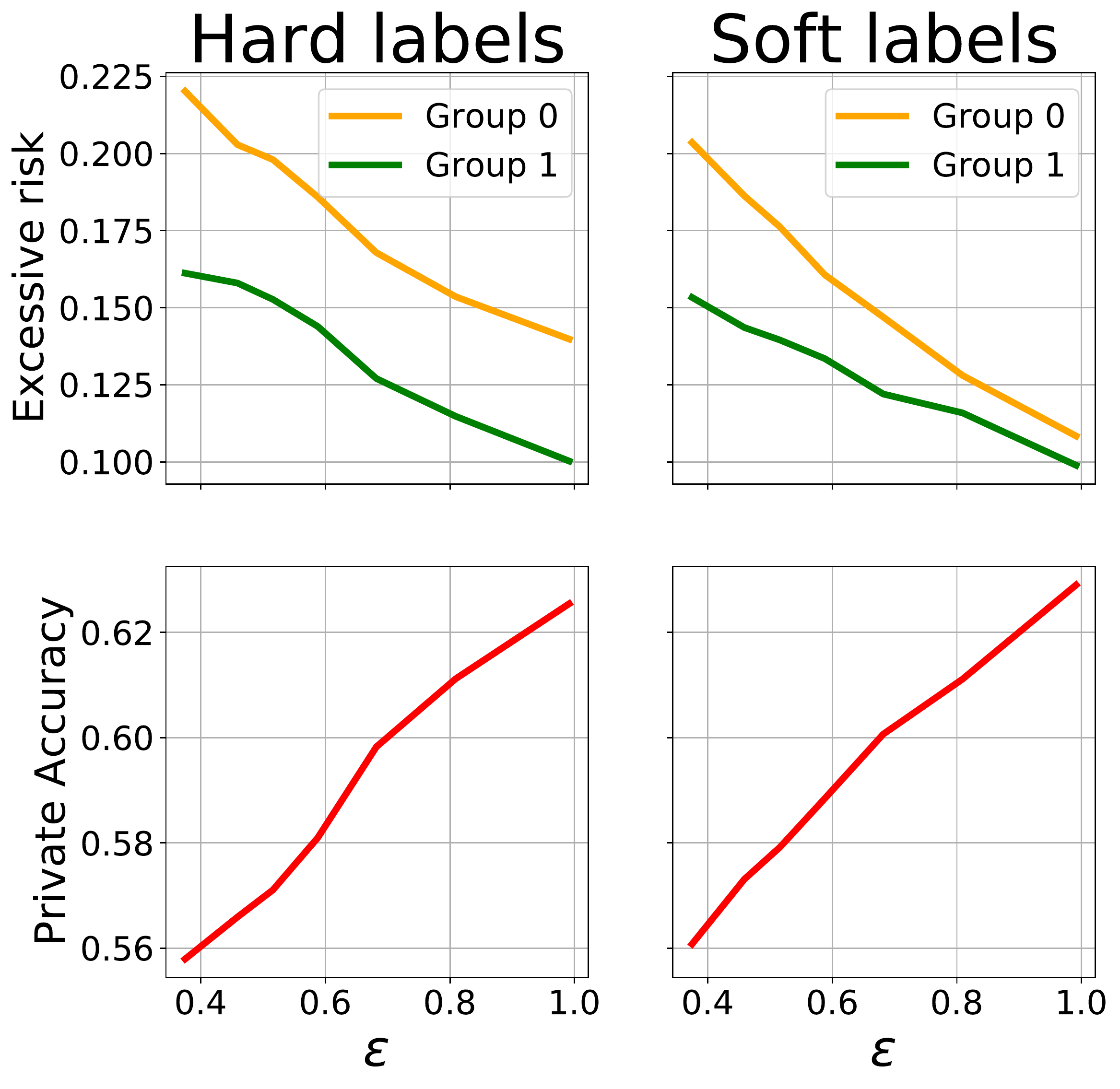}
\caption{}
\end{subfigure}
\caption{Comparison between training privately PATE with hard labels and soft labels in term of fairness (top subfigures) and utility(bottom subfigures) on (a) Bank, (b) Credit card, (c) Income and (d) Parkinsons dataset. Here for each dataset, the number of teachers $k=20$. }
\label{fig:mitigation_solution_K_20}
\end{figure*}

\begin{figure*}
\centering
\begin{subfigure}[b]{0.4\textwidth}
\includegraphics[width=\textwidth, height=4cm]{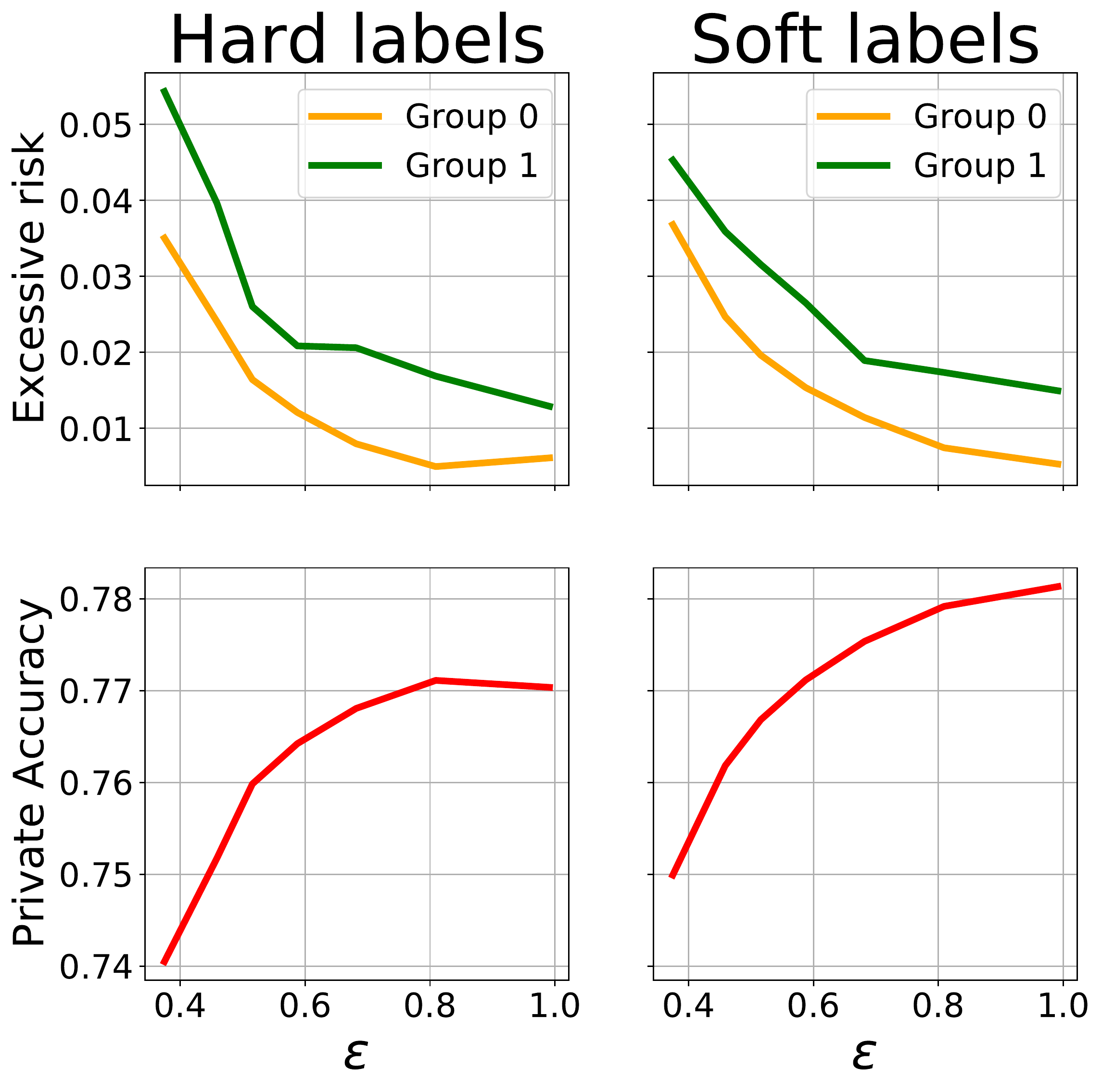}
\caption{}
\end{subfigure}
\begin{subfigure}[b]{0.4\textwidth}
\includegraphics[width=\linewidth, height=4cm]{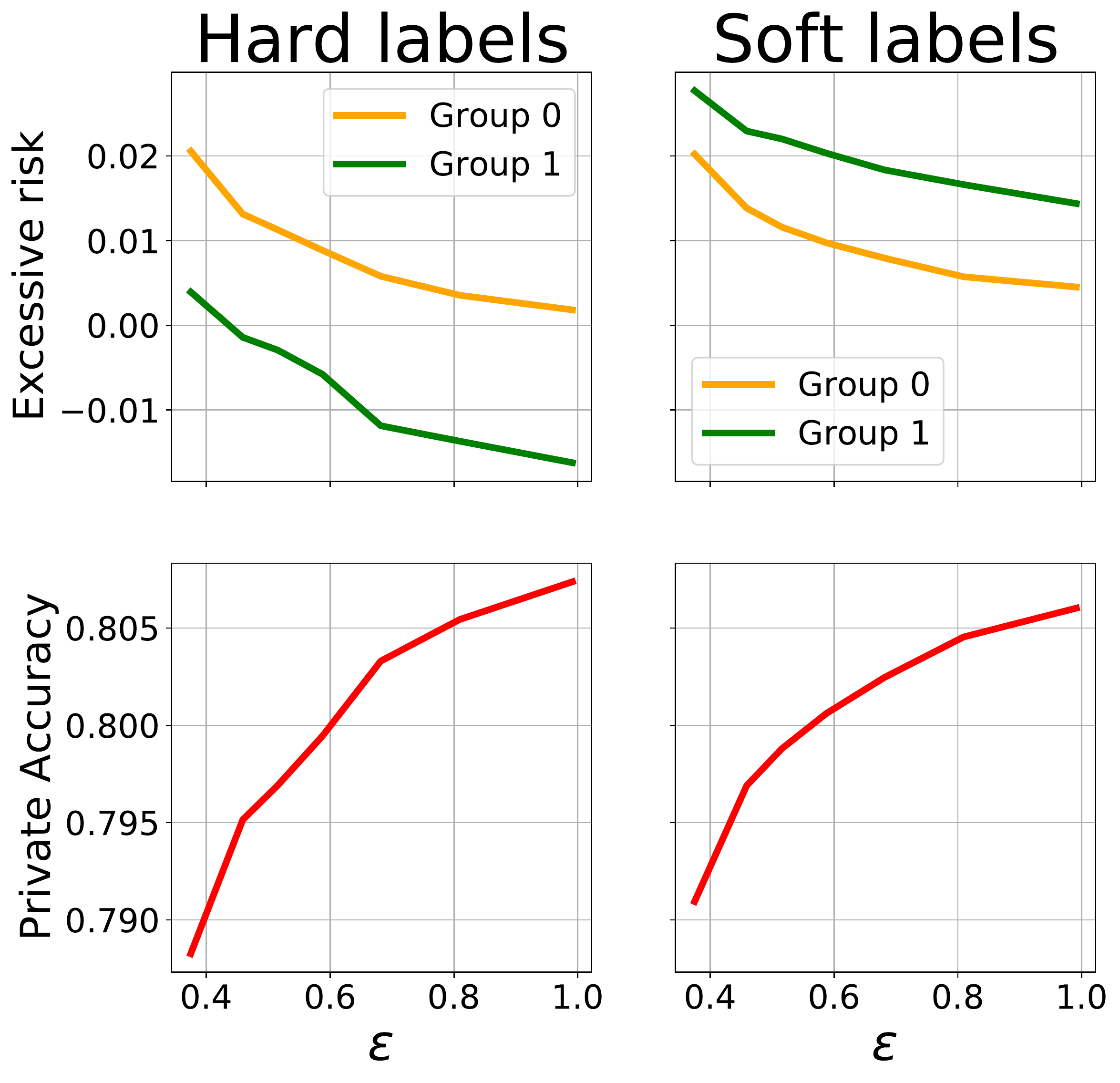}
\caption{}
\end{subfigure}
\begin{subfigure}[b]{0.4\textwidth}
\includegraphics[width=\linewidth, height=4cm]{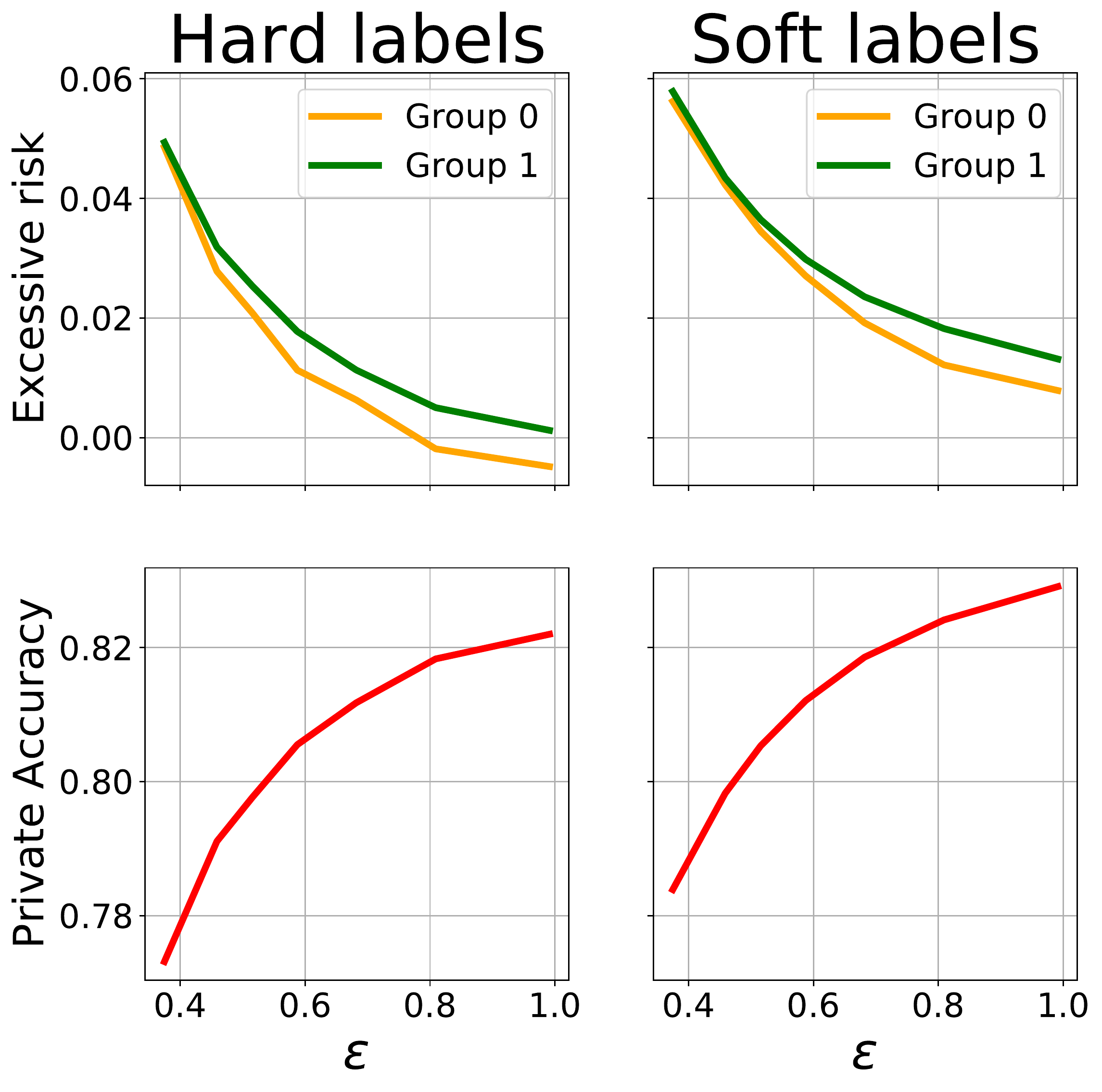}
\caption{}
\end{subfigure}
\begin{subfigure}[b]{0.4\textwidth}
\includegraphics[width=\linewidth, height=4cm]{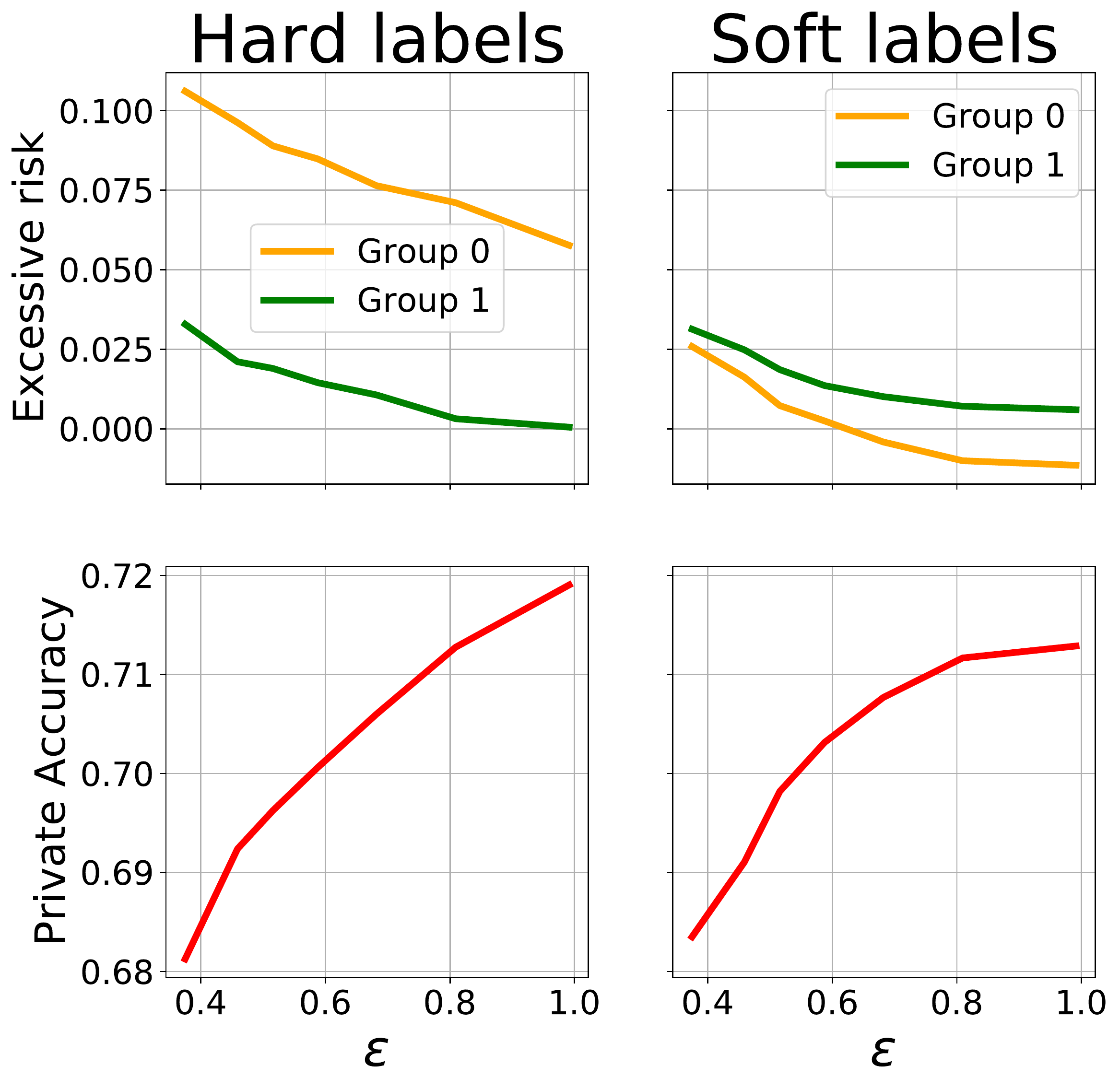}
\caption{}
\end{subfigure}
\caption{Comparison between training privately PATE with hard labels and soft labels in term of fairness (top subfigures) and utility(bottom subfigures) on (a) Bank, (b) Credit card, (c) Income and (d) Parkinsons dataset. Here for each dataset, the number of teachers $k=150$. }
\label{fig:mitigation_solution_K_150}
\end{figure*}

\end{document}